\documentclass[10pt,a4paper]{article}
\usepackage[english]{babel}
\usepackage[utf8]{inputenc}
\usepackage[T1]{fontenc}
\usepackage[a4paper,margin=1in,marginparwidth=1in]{geometry}
\usepackage{setspace}
\usepackage{amsmath,amssymb,amsthm,mathtools}
\usepackage{graphicx,xcolor,array,booktabs,multirow,tabularx}
\usepackage[authoryear]{natbib}
\usepackage{caption,subcaption}
\usepackage{algorithm}
\usepackage{algpseudocode}
\usepackage{longtable,float,placeins,enumitem,needspace,etoolbox}
\usepackage{tikz}
\usetikzlibrary{arrows.meta,patterns,positioning,fit,backgrounds,calc}
\usepackage{pgfplots}
\pgfplotsset{compat=1.16}
\usepgfplotslibrary{fillbetween}
\usepackage[colorlinks=true,allcolors=blue]{hyperref}
\usepackage{xurl}
\hypersetup{
  pdftitle={When Should AI Agents Ask? Clarification Requests Under Limited Human Attention},
  pdfauthor={Hanzhao Wang; Shengjie Liu; Zhengchao Wang},
  pdfsubject={AI agents, clarification requests, and shared human attention}
}
\theoremstyle{plain}
\newtheorem{theorem}{Theorem}
\newtheorem{proposition}{Proposition}
\newtheorem{lemma}{Lemma}
\newtheorem{corollary}{Corollary}
\newtheorem{definition}{Definition}
\newtheorem{example}{Example}
\theoremstyle{remark}

\newcommand{\E}{\mathbb E}

\newcommand{\ind}[1]{\mathbf{1}\{#1\}}
\definecolor{pcpink}{HTML}{23262B}
\definecolor{pcpmuted}{HTML}{5A5E66}
\definecolor{pcpgrid}{HTML}{8A8E96}
\definecolor{pcpblue}{HTML}{2A78D6}
\definecolor{pcporange}{HTML}{EB6834}

\algrenewcommand\algorithmicrequire{\textbf{Input:}}
\algrenewcommand\algorithmicensure{\textbf{Output:}}
\setlist{nosep,leftmargin=1.6em,topsep=4pt}
\pretocmd{\section}{\Needspace{9\baselineskip}}{}{}
\title{When Should AI Agents Ask?\\Clarification Requests Under Limited Human Attention}
\author{Hanzhao Wang$^{1}$\thanks{Corresponding author: \href{mailto:hanzhao.wang@sydney.edu.au}{hanzhao.wang@sydney.edu.au}.}
\and Shengjie Liu$^{2}$
\and Zhengchao Wang$^{1}$}
\date{\small
$^{1}$The University of Sydney Business School, The University of Sydney\\
$^{2}$Amazon\\
\href{mailto:hanzhao.wang@sydney.edu.au}{hanzhao.wang@sydney.edu.au}\quad
\href{mailto:zycjlsj@amazon.com}{zycjlsj@amazon.com}\\
\href{mailto:zhengchao.wang@sydney.edu.au}{zhengchao.wang@sydney.edu.au}}
\begin{document}
\maketitle
\onehalfspacing
\begin{abstract}
AI agents increasingly carry out delegated tasks in parallel, with one person overseeing several of them at once. In practice, an agent often lacks information its task requires, such as an unstated preference, constraint, or specification. It can ask the human to clarify, or it can search for the information itself, for example in past records, which takes time and effort and may fail. In this paper, we study how to coordinate these decisions across several agents that share one human with limited attention. We build a stylized model in which each agent can search, ask, or continue the parts of its task that do not depend on the missing information, and we formulate the problem as a dynamic program. We first show that whether an agent should ask depends on more than its own uncertainty: an agent with other work to do can be better off asking earlier, and which question the human should answer first depends on all waiting questions, so no ranking computed agent by agent is optimal. We then propose the Pairwise Coordination Policy (PCP), which uses exact one- and two-task problems to estimate how pairs of tasks delay each other when they share the human. Theoretically, PCP is exactly optimal when at most two requirements remain unresolved and costs at most twice the optimum when every unresolved requirement must be resolved by the human. We evaluate PCP on 864 synthetic cases and in a case study with LLM agents. On the synthetic benchmark, PCP's mean gap to optimal cost is 0.07\%, and it reduces cost by an average of 18--31\% relative to three standard rules for when to ask. In the LLM case study, agents completed every task correctly, and PCP achieved the lowest mean cost in all tested settings.
\end{abstract}
\section{Introduction}\label{sec:introduction}

Organizations are increasingly delegating multistep work to AI agents. In McKinsey's 2026 survey, 40\% of respondents from organizations reported scaling agents in at least one business function, up from 27\% in the previous year \citep{mckinsey2026}. Delegation is also changing how individuals organize their work. Usage evidence from Codex shows users assigning tasks to several agents running concurrently, tracking their progress in separate threads, and stepping in when clarification, review, or correction is needed \citep{johnston2026}. A single person may thus oversee several streams of agentic work at once and serve as the shared source of human input for all of them.

This arrangement raises an operational question: when should an agent continue working autonomously, and when should it seek human input? In practice, this choice already depends on how much human capacity is available. OpenTable's customer-service agents, for example, draw on knowledge resources and follow escalation rules whose thresholds can be adjusted when staffing is tight \citep{salesforce2025}. Similarly, an internal study at Anthropic describes employees managing competing requests from agents and developing tools to triage them \citep{wilmers2026}. These examples motivate studying how agents' autonomous effort and their requests for human intervention should be coordinated across tasks.

This question becomes concrete when an agent lacks information it needs to complete a delegated task. Instructions may leave preferences, constraints, or other details unstated. Some gaps require the person who delegated the task to clarify what is needed, and others can be filled from material already available to the agent. Consider a procurement manager whose agents each handle a different purchase. A missing specification may be recoverable from records of previous orders, or it may require the manager's clarification. A developer coordinating several coding agents may face questions about intended behavior whose answers are already in the tests or issue history. Likewise, someone who has delegated travel research, a purchase comparison, and event planning may be asked about preferences recorded in their emails, calendar, or notes. Such a gap need not bring the task to a halt. The agent may still be able to organize records, write code unaffected by the open question, or compare general alternatives while the missing information remains unresolved.

Figure~\ref{fig:motivation} shows a simple instance in Codex \citep{johnston2026}. Two coding agents working for the same person need input at the same time. One is building a daily sales report and asks which time zone defines a business day. The other is implementing a checkout and asks whether a coupon and a loyalty discount can be combined. Both agents have paused, and their questions wait for the same person.

How an agent should handle such a gap is not obvious. Looking for the answer takes time and effort that could otherwise advance the task, and the search may come up empty. When it succeeds, however, the human is spared an interruption and has more attention left for other agents. Asking early has its own appeal: the human can answer while the agent continues with other parts of the task, reducing the time lost waiting. Yet every question competes for the same person's attention, and a question the agent could have answered itself may delay one that requires human input. Asking also need not end the agent's search: while waiting for a response, it can keep looking and withdraw the question if it finds the answer before the human takes it up. The right choice for one agent therefore depends on the human's availability and on what the other agents need. We study how to make these decisions together: how agents should balance finding missing information, continuing their work, and seeking clarification, and which waiting question the human should answer whenever available.

\begin{figure}[t]
\centering
\begin{tikzpicture}[x=1cm,y=1cm,
  shot/.style={inner sep=0pt,draw=pcpgrid!80,line width=0.4pt},
  tag/.style={inner sep=0pt,font=\scriptsize\bfseries},
  ask/.style={-{Stealth[length=5pt,width=4pt]},line width=0.8pt,#1}]
\node[shot,anchor=north west] (S1) at (0,0) {\includegraphics[height=4.8cm]{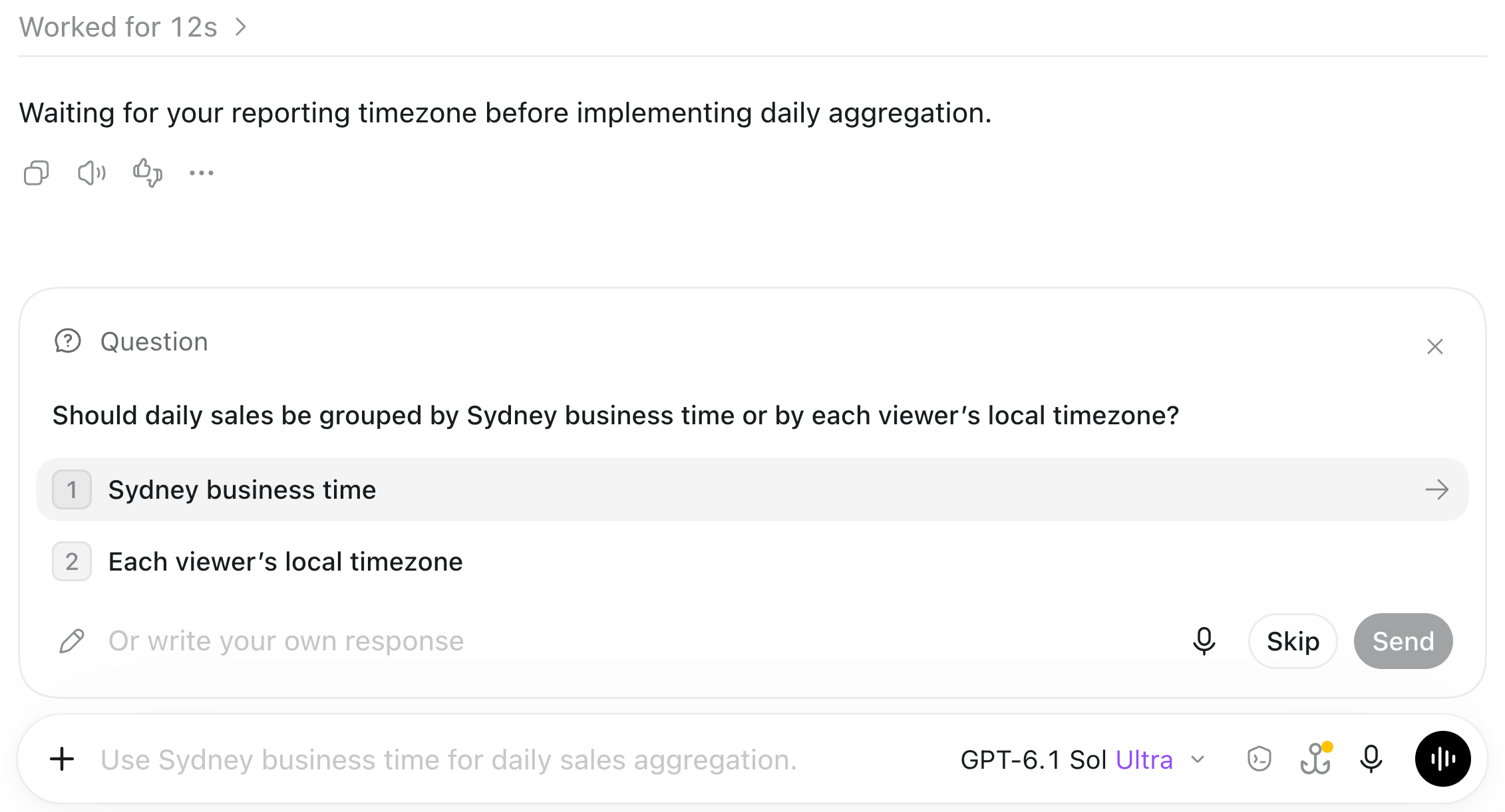}};
\node[shot,anchor=north east] (S2) at (\textwidth,0) {\includegraphics[height=4.8cm]{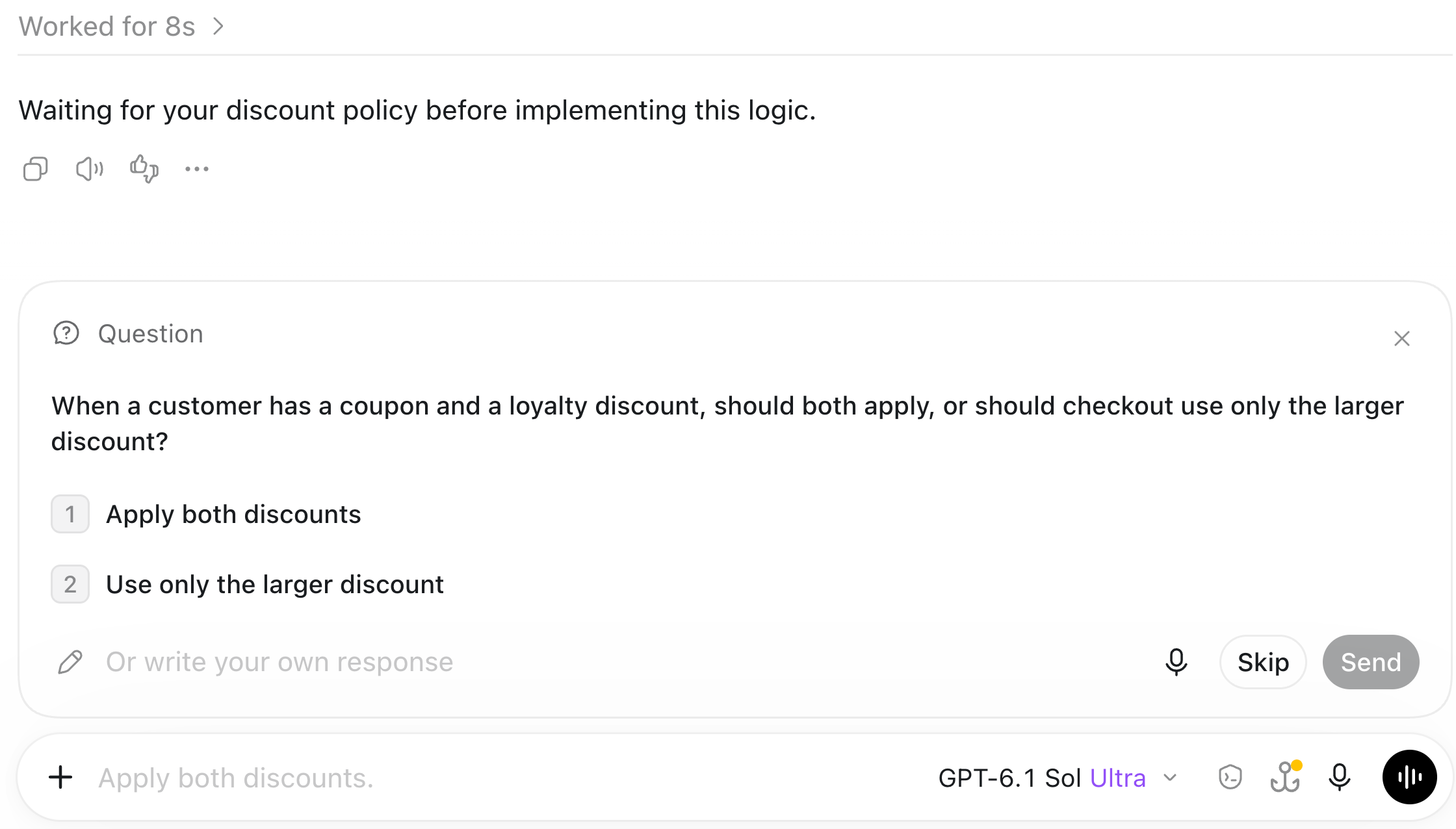}};
\node[tag,text=pcpblue,above=3pt of S1.north west,anchor=south west] {AGENT 1};
\node[tag,text=pcporange,above=3pt of S2.north west,anchor=south west] {AGENT 2};
\coordinate (H) at ($(S1.south east)!0.5!(S2.south west)+(0,-0.95)$);
\fill[pcpink] ($(H)+(0,0.3)$) circle (0.17);
\fill[pcpink] ($(H)+(-0.3,-0.22)$) arc[start angle=180,end angle=0,x radius=0.3,y radius=0.34] -- cycle;
\node[below=0.28cm of H,inner sep=0pt,font=\small\bfseries] {Human};
\draw[ask=pcpblue] (S1.south) to[out=-90,in=180] ($(H)+(-0.5,0)$);
\draw[ask=pcporange] (S2.south) to[out=-90,in=0] ($(H)+(0.5,0)$);
\end{tikzpicture}
\caption{Two Codex coding agents that need input from the same person at the same time.}
\label{fig:motivation}
\end{figure}
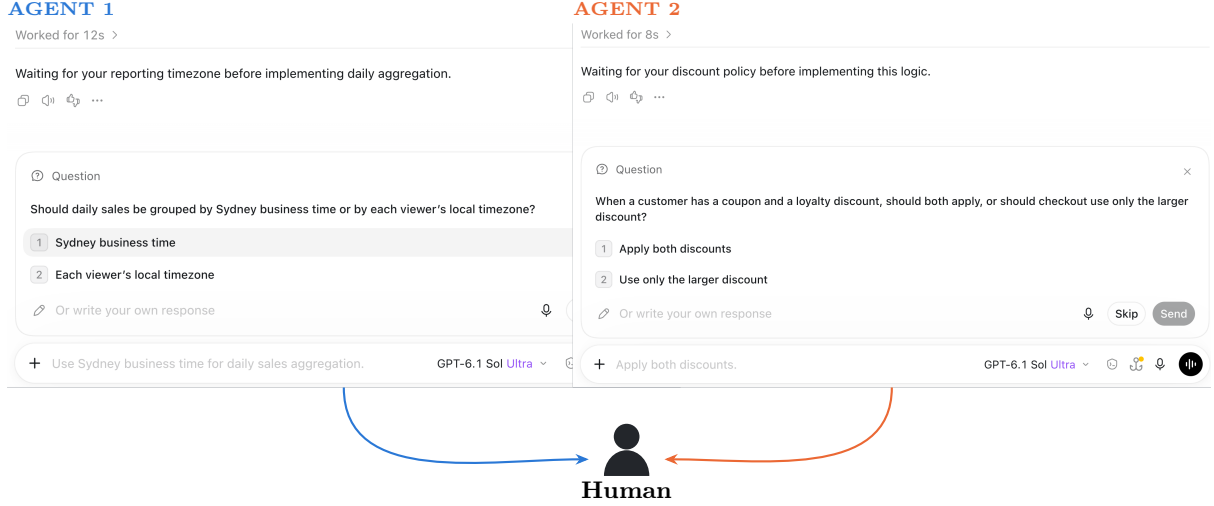

We formulate this problem as a dynamic program. Each task has its own agent, which in every period can search existing records for the missing requirement, advance work that does not depend on it, or submit a question to the shared human. The human is available only intermittently and answers at most one question per period. A question stays pending until it is answered, and the agent can keep searching and withdraw the question if a search succeeds. We show that several actions can be removed without loss of optimality, which makes the problem a finite dynamic program, although its size grows exponentially with the number of tasks.

We first study the single-agent case and examine how different settings affect the optimal strategy. We show that an agent whose independent work is complete and whose requirement remains unresolved has an optimal policy characterized by two thresholds on the number of failed searches, which determine whether it should search again or ask the human. We also show that increasing the amount of independent work (which requires no additional information from either the human or a search) can encourage the agent to ask sooner and increase demand on the human. Finally, we study how allowing an agent to withdraw a pending question affects demand on the human, highlighting withdrawal as an important decision option in this setting.

We then consider portfolios with multiple tasks. We first establish a negative result: no policy based solely on each task's individual state is optimal. We then propose the Pairwise Coordination Policy (PCP), which estimates future costs using exactly solvable one- and two-task problems. Each pairwise term measures how much two tasks delay one another by sharing the human and we call this quantity the pairwise price of shared attention. PCP is exactly optimal when at most two requirements remain unresolved and costs no more than twice the optimal cost once no search can succeed.

Numerically, we first conduct a synthetic experiment in which PCP's cost is, on average, within 0.07\% of the optimal cost across 864 portfolios of three or four tasks. PCP also reduces costs by 18\% to 31\% on average compared with common rules for deciding when to ask. A case study in which LLM agents carry out procurement tasks using Codex shows a similar advantage for PCP.

The rest of the paper is organized as follows. Section~\ref{sec:literature-review} reviews related work. Section~\ref{sec:setup} presents the model, Section~\ref{sec:structure} analyzes a single task, Section~\ref{sec:algorithm} develops PCP and its guarantees, Section~\ref{sec:numerics} reports the numerical study and the case study, and Section~\ref{sec:conclusion} concludes.


\section{Literature Review}
\label{sec:literature-review}

This study brings together four streams: human--agent delegation, clarification and sequential information acquisition, allocation of shared human attention, and structured dynamic optimization. Specifically, we examine a setting in which several agents can search for missing requirements or advance independent work, while their clarification requests compete for a single human's attention.

\subsection{Delegation and adjustable autonomy}
\label{subsec:literature-delegation}

Research on agentic information systems treats delegation as an evolving distribution of responsibility between humans and agents. \citet{bairdMaruping2021} develop a conceptual framework for delegation to and from agentic systems, and \citet{zhangSun2026} set out a decision-science agenda centered on the interaction structures through which humans and agents make decisions over time. Behavioral evidence complements these frameworks: \citet{fugenerEtAl2022} show that the performance of human--AI collaboration depends on the direction of delegation and on humans' ability to assess their own capabilities. Earlier work on mixed-initiative interaction and adjustable autonomy treats asking for human input as a sequential decision. \citet{horvitz1999} advocates decision-theoretic reasoning about when to act and when to engage a user, while \citet{scerriEtAl2002} model revisable transfers of control that account for uncertain human response times and their effects on team performance. We take a normative, operational approach. Assuming fixed search and human-answer reliability, we model each question as pending until it is answered or withdrawn after a successful search. While the question is pending, the agent can continue working. We account separately for notification, answering, and waiting costs, allowing each to shape a distinct decision.

\subsection{Clarification and sequential information acquisition}
\label{subsec:literature-clarification}

Choosing between further search and asking relates to the economics of sequential search, in which the value of further information governs both search order and stopping \citep{weitzman1979}. In our setting, repeated attempts target a single requirement, failed searches weakly lower the probability that the next attempt succeeds, and search consumes agent time that could advance independent work. Recent clarification research treats questions as costly information acquisition. \citet{dongEtAl2026} weigh the expected utility gain from asking against the cognitive cost to the user, and \citet{suriEtAl2026} apply the expected value of perfect information to structured tool parameters, distinguishing uncertainty about the user's specification from uncertainty in model predictions. Computational studies further show that how and when clarification enters execution matters: \citet{edwardsSchuster2026} separate detection of underspecified instructions from code execution in a multi-agent scaffold, and \citet{gulatiEtAl2026} find that the value of a clarification depends on the missing information and on the stage of execution at which it arrives. Unlike these studies, we consider multiple distinct tasks that share a human. Searching for missing information on one task preserves the human's capacity to assist with others, while an agent can continue working as it awaits an answer. This setting lets us determine when the presence of independent work makes immediate clarification more attractive and evaluate clarification requests by their impact on portfolio completion costs.

\subsection{Allocation of shared human attention}
\label{subsec:literature-human-attention}

A related stream of research studies human supervision of multiple robots. \citet{dahiyaEtAl2022} formulate operator allocation among semiautonomous robots as a restless-bandit problem and derive conditions for a Whittle-index policy, building on \citet{whittle1988}. \citet{caiEtAl2022} schedule robots that can complete tasks autonomously or faster with human assistance using a mixed-integer formulation with known processing times. \citet{costenEtAl2024} design a bidding mechanism for shared assistance that balances its direct benefit against learning about latent operator performance. In our model, demand for assistance is endogenous: an agent can remove its own need for clarification by finding the information, even after asking. Search competes with work for the agent's time, whereas human service runs in parallel with work. These features change how priorities should be set. Our pairwise policy prices the cost of sharing the human using exactly solved two-task problems. Our setting also connects to scheduling. When every remaining requirement must be supplied by the human, choosing which question to answer is a single-server scheduling problem with unit processing times. In the absence of independent work, serving questions in descending order of weight is optimal, as in Smith's rule \citep{smith1956}.

\subsection{Structured dynamic optimization}
\label{subsec:literature-dynamic-optimization}

Our computational approach builds on approximate dynamic programming with factored value functions. For multiagent planning in factored Markov decision processes, \citet{guestrinEtAl2002,guestrinEtAl2003} use local value functions and coordination graphs. \citet{deFariasVanRoy2003} relate value approximation to control performance in linear-programming approaches, while \citet{boehmerEtAl2020} approximate joint action values using individual utilities and pairwise payoffs. The Pairwise Coordination Policy constructs its components from exactly solved reduced versions of our model. For each pair of tasks, it subtracts the two singleton values from the joint value when both tasks share a human, isolating the cost of sharing that capacity. Each component therefore has an operational interpretation.

\section{Model}\label{sec:setup}

We model a finite batch of delegated tasks that share one human. Throughout, we use the procurement setting of Section~\ref{sec:introduction} as a running example to illustrate the setup: each agent handles one purchase for a manager, and the missing requirement is a specification, such as the maximum acceptable lead time, that may or may not be recorded in earlier orders.

\subsection{Tasks, agents, and the human}

In this study, we assume that each task is delegated to its own agent, as when a person runs several agent threads at once \citep{johnston2026}. Let $n$ be the number of tasks. A task is complete only when both conditions are met. First, the agent must finish the task's \emph{independent work}: steps that do not depend on the missing requirement, such as extracting supplier prices and computing landed costs. Second, the missing requirement must be \emph{resolved}, either by the agent searching existing records or by the human answering a question. Time is divided into periods, each long enough for one search attempt or one unit of independent work. An agent can perform at most one of these activities in a period, so searching and working compete for the agent's time.

To ask, the agent submits a question, which remains \emph{pending} until the human answers it. The human is shared by all tasks and only intermittently available. While its question is pending, an agent can continue its independent work, including in the period in which the human answers. It can also keep searching, and if a search succeeds first, the question is \emph{withdrawn}. A central controller makes all decisions: which questions to submit, which pending question the human answers, and whether each agent searches, works, or waits.

\subsection{Search and information}

We assume the agent does not know whether the needed requirement is in the records. Task $i$'s requirement is recoverable with probability $p_i\in[0,1)$. We exclude $p_i=1$ so that repeated failures eventually make further search unattractive. If the requirement is recoverable, each search attempt finds it independently with probability $\alpha_i\in(0,1]$; if it is not, every search fails. Recoverability and search outcomes are independent across tasks, and a successful search resolves the requirement correctly and observably. In the procurement example, $p_i$ is the probability that the specification appears in earlier orders, and $\alpha_i$ is the probability that one pass through those records finds it, given that it is there. We assume that the agent system knows (or has estimates of) these parameters; addressing cases in which they are unknown is left for future research.

By Bayes' rule, a failed search weakly increases the probability that the requirement is not recoverable. After $k$ failed searches, the probability that the next search succeeds is
\begin{equation}\label{eq:belief}
 q_i(k)=\frac{p_i\alpha_i(1-\alpha_i)^k}{1-p_i+p_i(1-\alpha_i)^k},
\end{equation}
which is weakly decreasing in $k$. Thus, the number of failures summarizes all the information the agent has learned from searching.

\subsection{State and timing}

At each period $t$, let $s_i=(x_i,k_i,z_i)$ denote the observed state of task $i$, suppressing the time index for notational convenience. Here, $x_i$ is the number of independent work units remaining, $k_i$ is the number of failed searches, and $z_i\in\{U,P,R\}$ indicates whether the requirement is unresolved with no question submitted ($U$), unresolved with a pending question ($P$), or resolved ($R$). Task $i$ starts in state $s_i=(b_i,0,U)$, where $b_i\ge0$ denotes its amount of independent work. We also allow a question to be pending initially. The task is complete when $x_i=0$ and $z_i=R$. Figure~\ref{fig:lifecycle} shows how $z_i$ evolves.

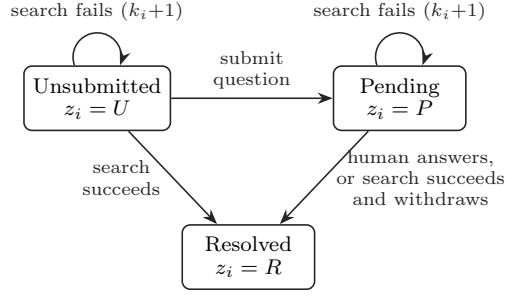
\begin{figure}[t]
\centering
\begin{tikzpicture}[>=Stealth,
  st/.style={draw=pcpink, line width=0.6pt, rounded corners=3pt, minimum width=1.75cm,
             minimum height=0.85cm, align=center, font=\footnotesize, fill=white},
  lab/.style={font=\scriptsize, align=center, text=pcpink},
  arr/.style={->, line width=0.6pt, draw=pcpink}]
 \node[st] (U) at (0,0) {Unsubmitted\\[-1pt]$z_i=U$};
 \node[st] (P) at (4.0,0) {Pending\\[-1pt]$z_i=P$};
 \node[st] (R) at (2.0,-2.1) {Resolved\\[-1pt]$z_i=R$};
 \draw[arr] (U) -- node[lab, above] {submit\\question} (P);
 \draw[arr] (U) -- node[lab, left, xshift=-1pt] {search\\succeeds} (R);
 \draw[arr] (P) -- node[lab, right, xshift=1pt] {human answers,\\or search succeeds\\and withdraws} (R);
 \draw[arr] (U.north) ++(-0.3,0) arc[start angle=200, end angle=-20, radius=0.32]
   node[lab, pos=0.5, above, yshift=1pt] {search fails ($k_i{+}1$)};
 \draw[arr] (P.north) ++(-0.3,0) arc[start angle=200, end angle=-20, radius=0.32]
   node[lab, pos=0.5, above, yshift=1pt] {search fails ($k_i{+}1$)};
\end{tikzpicture}
\caption{Status $z_i$ of task $i$'s requirement. Each search succeeds with probability $q_i(k_i)$. }
\label{fig:lifecycle}
\end{figure}

In each period, the human is available with probability $a\in(0,1]$, independently across periods and of all search outcomes. This is a reduced-form description of a person who checks in on agents intermittently while attending to other duties. Events within each period occur in the following order:
\begin{enumerate}
 \item Before learning whether the human is available, the controller submits questions. Each task's question can be submitted at most once.
 \item Availability is revealed. If the human is available, the controller selects at most one unanswered submitted question for service; this may include a question submitted in the current period. Service need not follow a first-come, first-served order.
 \item Each agent searches, works, or waits, knowing whether the human is available and which question, if any, is being served.
 \item Outcomes take effect at the end of the period. Work units are completed, and a served question resolves its requirement because answers are correct. A successful search also resolves its requirement and withdraws any pending question at no cost. Completed tasks leave the system.
\end{enumerate}
We assume that answers and successful searches are correct, separating timing and capacity from the reliability of retrieval; this assumption is consistent with behavioral research \citep{fugenerEtAl2022}.

\subsection{Costs and objective}

We assume there are four types of cost. Each period in which task $i$ is incomplete incurs a cost of $w_i>0$, which we call its \emph{delay weight}. A higher weight marks a more urgent task, such as a supplier decision that holds up a purchase. Delay costs are charged at the beginning of each period. Each search on task $i$ incurs a cost of $c_i>0$, reflecting compute and tool usage. Each submitted question incurs a cost of $\kappa_i\ge0$, representing the interruption caused by notifying the human, and each answer incurs a cost of $\eta\ge0$, representing the human time required to respond. Because submission and answering are charged separately, a question withdrawn before service incurs only the submission cost $\kappa_i$.

Let $C_i$ be the number of periods until task $i$ completes, $N_i^S$ its number of searches, $N_i^Q$ its number of submitted questions, and $N^H$ the total number of human answers. From the joint state $s=(s_1,\ldots,s_n)$, the controller seeks
\begin{equation}\label{eq:objective}
 V^*(s)=\inf_\pi\E_s^\pi\!\left[
 \sum_i w_i C_i+\sum_i c_i N_i^S
 +\sum_i\kappa_i N_i^Q+\eta N^H\right],
\end{equation}
where the infimum is over nonanticipative policies $\pi$, which may depend on the observed history. For a continuation state, $C_i$ counts the periods remaining until task $i$ completes, and the objective excludes fees already incurred.

\subsection{Structural properties}

We first identify restrictions on the controller's decisions that preserve optimality.

\begin{lemma}[Structural properties]\label{lem:reform}
There exists an optimal policy with the following three properties simultaneously.
\begin{enumerate}[label=(\roman*)]
 \item An agent never searches while its question is being answered, performs its remaining work without delay after resolution, and, before resolution, works whenever it does not search and work remains.
 \item At most one new question is submitted per period, and a newly submitted question is served if the human is available in that period.
 \item Task $i$ is never searched after $M_i$ failures, where $M_i=\min\{k\ge0:q_i(k)B_{\max}\le c_i\}$ is finite and $B_{\max}=\sum_iw_ib_i+\frac na\sum_iw_i+\sum_i\kappa_i+n\eta$ bounds the optimal remaining cost from any reachable state.
\end{enumerate}
With these reductions, the optimal value solves a finite dynamic program in which every state-changing transition advances work, failures, submission, or resolution.
\end{lemma}

Lemma~\ref{lem:reform} shows that the controller can restrict attention to a much smaller set of decisions without losing optimality: agents never sit idle when they have work, at most one new question is submitted per period, and each task has a finite search budget. Section~\ref{sec:algorithm} uses these properties to define the candidate plans of our policy. Appendix~\ref{app:dp} shows that these caps are sufficient and states the resulting dynamic program, and Appendix~\ref{app:proof-reform} proves these properties.

Even with these reductions, solving for the optimal policy exactly remains computationally challenging. A single task $i$ has at most $3(b_i+1)(M_i+1)$ states, so the joint state space can contain as many as $\prod_{i=1}^n 3(b_i+1)(M_i+1)$ states. We call the decisions made in one period a \emph{period plan}. A period plan specifies which question, if any, is submitted and, for each possible availability outcome, which question is served and whether each agent searches, works, or waits. The number of period plans grows exponentially with $n$, so exact dynamic programming is practical only for small portfolios. The following sections develop additional structure in the problem and introduce PCP, an efficient algorithm for larger portfolios.

\section{When should a single agent search, ask, or work?}\label{sec:structure}

We begin with a single task to understand how an agent should choose among searching, asking, and working. In this setting, the optimal policy admits an explicit characterization that reveals how costs and timing shape the decision to ask. The same forces govern each agent's choices in a portfolio and help explain the behavior of the optimal policy and PCP in Sections~\ref{sec:algorithm} and~\ref{sec:numerics}. The single-task value functions also serve as building blocks for PCP. Because this section considers one task, we drop the index $i$.

\subsection{When to stop searching and ask}\label{subsec:blocked}

We start with the simplest case: a task whose independent work is finished ($x=0$). We call such a task \emph{blocked} because only its requirement remains; its agent can therefore only search or wait for the human. The question is how many failed searches the agent should tolerate before asking, and whether it should continue searching afterward. The trade-off is subtle because these costs arise at different times: the submission fee is paid when the question is submitted, the answer fee when the human responds, and delay costs accrue while the task waits. The next theorem shows that an optimal policy can be characterized by two thresholds on the number of failed searches: one before asking and one after.

\begin{theorem}[Blocked-task thresholds]\label{thm:thresholds}
For one task with $x=0$, there exists an optimal policy with an asking threshold $k_{\rm ask}$: an unsubmitted task searches below this threshold and submits at or above it, where ties between searching and submitting are broken in the same direction at every failure count. Conditional on a pending question and an unavailable human, search is strictly preferable to waiting if and only if
\begin{equation}\label{eq:pendingcut}
 q(k)\left(\eta+\frac wa\right)>c.
\end{equation}
Thus pending search also stops at a failure-count threshold. The two thresholds have no universal ordering.
\end{theorem}

Theorem~\ref{thm:thresholds} shows that a blocked agent needs to track only how many of its searches have failed. Appendix~\ref{app:blocked} gives the exact characterization of $k_{\rm ask}$, and Appendix~\ref{app:proof-blocked} proves the theorem. Before asking, the agent searches until $k_{\rm ask}$ failures and then submits its question. After asking, it keeps searching in periods when the human is unavailable for as long as~\eqref{eq:pendingcut} holds. Condition~\eqref{eq:pendingcut} compares the cost of one more search with its expected benefit: with probability $q(k)$, the search succeeds and avoids the answer fee $\eta$ and an expected wait of $1/a$ periods for the human, each of which costs $w$. The submission fee is avoidable before asking and sunk afterward, so the two thresholds can fall in either order. Whether the agent searches therefore depends on its current belief and on whether it has already asked.

For example, let $w=\eta=1$, $a=0.5$, $c=0.1$, $p=0.6$, and $\alpha=0.5$. Condition~\eqref{eq:pendingcut} holds for $k\le4$, so a pending agent keeps searching while the human is unavailable until five searches have failed, whatever the submission fee. The asking threshold rises with the submission fee: the agent asks at once when $\kappa\le1.5$, after one failure when $\kappa=3$, after three failures when $\kappa=8$, and after six failures when $\kappa=60$. With a small fee, the agent asks at once and keeps searching while it waits for the answer. With a very large fee, it searches beyond the point at which a pending agent would stop, because one more success still saves the fee.

\subsection{The influence of independent work}\label{subsec:overlap}

Most tasks still have work left when an agent discovers that a requirement is missing. In the procurement example, the agent may still need to compute landed costs when it finds that the acceptable lead time is unknown. We now ask how such independent work changes the decision to ask. One might expect that work makes asking less urgent, because the agent has something useful to do while it searches. The next theorem shows that the opposite can hold. To isolate the effect of work, it compares two tasks that are identical except that one has a unit of work left, and it sets $\alpha=1$, so that a single search reveals whether the requirement is in the records.

\begin{theorem}[Work overlap and human demand]\label{thm:overlap}
Consider one task with $\alpha=1$, $w,c,\kappa,\eta>0$, and $a\in(0,1]$. Define
\begin{equation}\label{eq:thetas}
 \theta_0=\frac{w+ac}{\kappa+w+a\eta},\qquad
 \theta_1=\frac{w+ac}{\kappa+(1-a)w+a\eta}.
\end{equation}
If
\begin{equation}\label{eq:overlapregion}
 \max\{c/\eta,\theta_0\}<p<\min\{1,\theta_1\},
\end{equation}
the optimal initial actions differ between two otherwise identical tasks:
\begin{enumerate}[label=(\roman*)]
 \item With zero work units, search first and submit only after failure.
 \item With one work unit, submit immediately. If the human is available, receive the answer while working. Otherwise, search.
\end{enumerate}
Adding that work unit raises the probability of a human answer by $ap$ and reduces expected searches by $a$. The interval in~\eqref{eq:overlapregion} is nonempty exactly when
\begin{equation}\label{eq:nonempty}
 c<\eta,\qquad
 0<\kappa<w\left[\frac{\eta}{c}-(1-a)\right].
\end{equation}
\end{theorem}

Theorem~\ref{thm:overlap} shows that giving an agent useful independent work can make it ask the human sooner and therefore increase human demand. Within the interval~\eqref{eq:overlapregion}, adding one work unit switches the optimal first move from searching to asking. The two cutoffs have direct interpretations. Above $\theta_0$, a task without work prefers to search before asking. Below $\theta_1$, a task with one work unit prefers to ask immediately. The lower bound $p>c/\eta$ ensures that searching remains worth its cost even when an answer is on its way.

The intuition is that independent work changes the cost of waiting. Without work, waiting for the human wastes the agent's time, so an agent with a promising search first tries to resolve the requirement itself. With work, waiting is productive: an answer received now overlaps work that must be done anyway, whereas a search displaces that work. In the procurement example, an agent that still has to compute landed costs should ask the manager about the lead time right away and do the computation while it waits, whereas an agent with nothing left to do should check earlier orders first. Numerically, for $w=\kappa=\eta=1$, $c=0.1$, and $a=0.5$, the interval is $0.42<p<0.525$, and at $p=0.5$ the probability of a human answer rises from $0.5$ to $0.75$. Figure~\ref{fig:regions} shows how the region varies with $a$. The proof is in Appendix~\ref{app:proof-overlap}.

\subsection{The influence of question withdrawal}\label{subsec:withdraw}

In our model, an agent may continue searching after asking and withdraw its question if a search succeeds. Some agent interfaces can disable this option. In those cases, asking acts as a handoff: once the agent asks, it stops searching for the requirement and waits for the answer, although it may continue other work. We now ask how much the option to withdraw is worth and how it affects demand on the human. To isolate the effect of withdrawal, consider the \emph{no-withdrawal model}, which differs from ours only in that an agent cannot search after submitting a question. Withdrawal adds an option and therefore cannot increase the optimal cost. Because some questions may be withdrawn before service, one might also expect withdrawal to reduce demand on the human. The next result shows that it can have the opposite effect.

\begin{proposition}[Withdrawable questions]\label{prop:withdraw}
Consider one task with $\alpha=1$, one work unit, $w,c,\kappa,\eta>0$, and $a\in(0,1)$. With $\theta_1$ as in~\eqref{eq:thetas}, define
\begin{equation}\label{eq:theta_nw}
 \theta_{\rm nw}=\frac{a(w+c)}{a\kappa+a\eta+(1-a)w}.
\end{equation}
If
\begin{equation}\label{eq:withdrawregion}
 \max\{c/\eta,\theta_{\rm nw}\}<p<\min\{1,\theta_1\},
\end{equation}
the optimal initial actions differ between the two models:
\begin{enumerate}[label=(\roman*)]
 \item Without withdrawal, search first and submit only after failure.
 \item With withdrawal, submit immediately. If the human is available, receive the answer while working. Otherwise, search, and withdraw the question if the search succeeds.
\end{enumerate}
Allowing withdrawal lowers the optimal expected cost by
\begin{equation}\label{eq:withdrawsaving}
 w+ac-p\bigl[\kappa+(1-a)w+a\eta\bigr]>0
\end{equation}
and the expected completion time by $1-(1-a)p$. It raises the probability of a human answer by $ap$ and the probability of submission by $p$, and it reduces expected searches by $a$. The interval in~\eqref{eq:withdrawregion} is nonempty exactly when
\begin{equation}\label{eq:withdrawnonempty}
 c<\eta,\qquad
 \kappa<w\left[\frac{\eta}{c}-(1-a)\right],\qquad
 (2a-1)w<a(\kappa+\eta-c).
\end{equation}
\end{proposition}

Within the interval~\eqref{eq:withdrawregion}, withdrawal switches the optimal first move from searching to asking, lowering total cost while raising human demand. The new cutoff $\theta_{\rm nw}$ plays the role of $\theta_0$: above it, a task that cannot withdraw its question prefers to search before asking. Unlike Theorem~\ref{thm:overlap}, the proposition requires $a<1$, because at $a=1$ the cutoffs $\theta_{\rm nw}$ and $\theta_1$ coincide and the interval is empty. Appendix~\ref{app:proof-withdraw} gives the proof.

Intuitively, without withdrawal, asking forfeits the chance to resolve the requirement autonomously, so an agent with a promising search postpones its question and delays its independent work while searching. Withdrawal removes this forfeit and makes early questions cheaper. For the parameters of the earlier illustration, the interval is $0.367<p<0.525$. At $p=0.5$, withdrawal lowers the optimal cost from $3.60$ to $3.55$: faster completion saves $0.75$ in delay costs and fewer searches save $0.05$, which outweighs the additional $0.50$ in submission fees and $0.25$ in answer fees, and the probability that the human answers rises from $0.50$ to $0.75$. The one-work-unit policy in Theorem~\ref{thm:overlap} relies on withdrawal in exactly this way, and Figure~\ref{fig:regions} shows that the two regions overlap substantially.

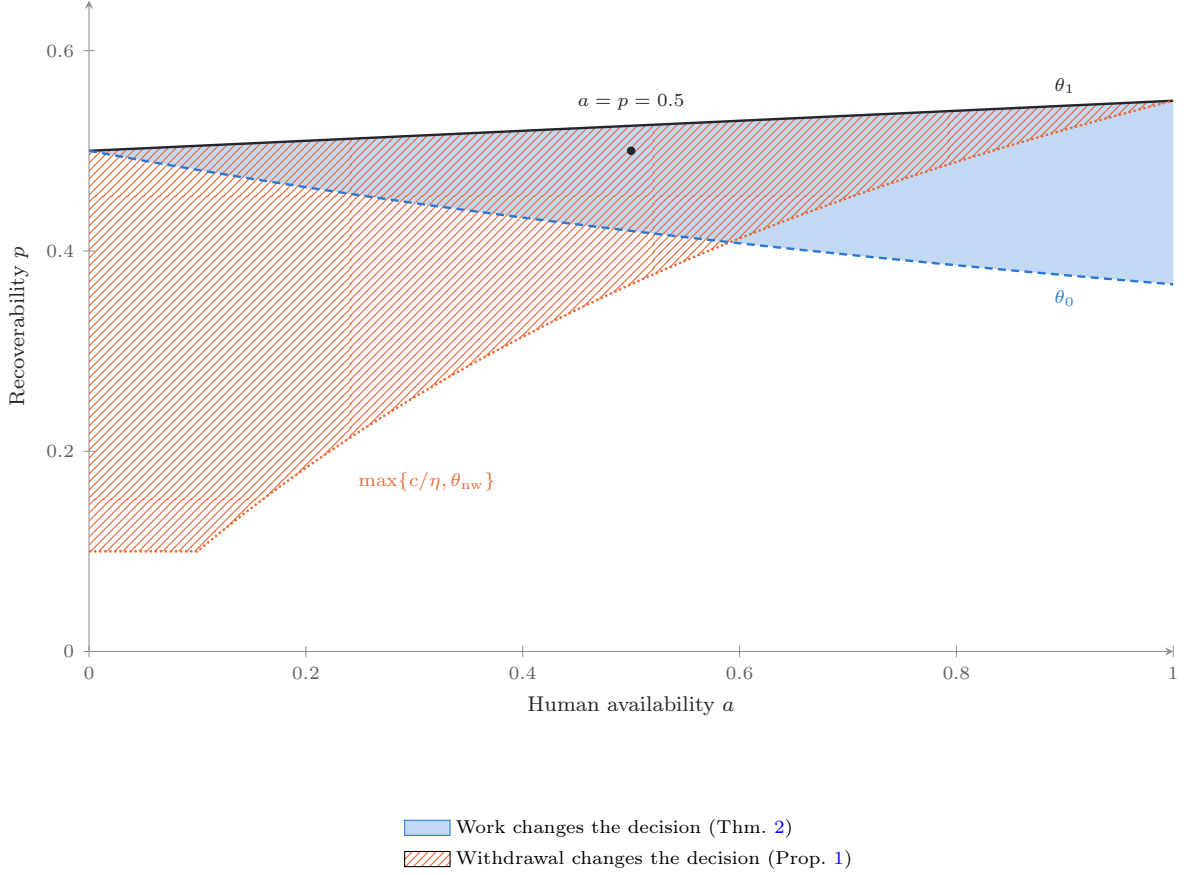
\begin{figure}[t]
\centering
\begin{tikzpicture}
\begin{axis}[
  width=\columnwidth, height=0.64\columnwidth,
  xmin=0, xmax=1, ymin=0, ymax=0.65,
  xlabel={Human availability $a$}, ylabel={Recoverability $p$},
  xtick={0,0.2,0.4,0.6,0.8,1}, ytick={0,0.2,0.4,0.6},
  tick label style={font=\scriptsize, text=pcpmuted}, label style={font=\footnotesize, text=pcpink},
  axis line style={draw=pcpgrid}, tick style={draw=pcpgrid},
  axis lines=left, clip=true,
  legend style={font=\scriptsize, draw=none, fill=none,
                at={(0.5,-0.24)}, anchor=north, cells={anchor=west}, row sep=1pt},
  legend image code/.code={\fill[#1] (0cm,-0.09cm) rectangle (0.32cm,0.11cm);},
]
 \addplot[name path=thone, draw=pcpink, line width=0.9pt, domain=0:1, samples=101, forget plot] {(1+0.1*x)/2};
 \addplot[name path=thzero, draw=pcpblue, line width=0.9pt, densely dashed, domain=0:1, samples=101, forget plot] {(1+0.1*x)/(2+x)};
 \addplot[name path=thnw, draw=pcporange, line width=0.9pt, densely dotted, domain=0:1, samples=201, forget plot] {max(0.1, 1.1*x/(1+x))};
 \addplot[pcpblue, fill opacity=0.28, draw=none, area legend] fill between[of=thone and thzero];
 \addlegendentry{Work changes the decision (Thm.~\ref{thm:overlap})}
 \addplot[pattern=north east lines, pattern color=pcporange, draw=none, area legend] fill between[of=thone and thnw];
 \addlegendentry{Withdrawal changes the decision (Prop.~\ref{prop:withdraw})}
 \node[font=\scriptsize, text=pcpink, anchor=south] at (axis cs:0.9,0.548) {$\theta_1$};
 \node[font=\scriptsize, text=pcpblue, anchor=north] at (axis cs:0.9,0.37) {$\theta_0$};
 \node[font=\scriptsize, text=pcporange, anchor=west] at (axis cs:0.24,0.17) {$\max\{c/\eta,\theta_{\rm nw}\}$};
 \fill[pcpink] (axis cs:0.5,0.5) circle (1.6pt);
 \node[font=\scriptsize, text=pcpink, anchor=south] at (axis cs:0.5,0.532) {$a=p=0.5$};
\end{axis}
\end{tikzpicture}
\caption{Regions of Theorem~\ref{thm:overlap} and Proposition~\ref{prop:withdraw} for $w=\kappa=\eta=1$ and $c=0.1$. In both, an agent with one work unit and withdrawable questions asks immediately, whereas the same agent without work (above $\theta_0$) or without withdrawal (above $\theta_{\rm nw}$) would search first. Both regions also require $p>c/\eta=0.1$, which binds only for the withdrawal region when $a<0.1$. The withdrawal region closes at $a=1$, where $\theta_{\rm nw}=\theta_1$.}
\label{fig:regions}
\end{figure}
Together, the two results offer a management insight: we should not judge agents or interface designs solely by how rarely they ask questions. Useful independent work can make the optimal agent ask sooner and the human answer more often, and the option to withdraw a question can do so while lowering total cost. Both mechanisms persist beyond the single search attempt and single work unit assumed in the two results. When search can be repeated, one more unit of work raises the optimal expected number of human answers in at least 88\% of the settings we examine, and withdrawal lowers the optimal cost in every setting while raising the number of answers in 19\% to 40\% of them; both effects also appear when four tasks share one human (more results can be found in Appendix~\ref{app:num-mech}). Theorem~\ref{thm:overlap} and Proposition~\ref{prop:withdraw} describe each agent's incentives in isolation. In a portfolio, however, every question can delay other agents waiting for the human. The next section shows how to account for this interaction.

\section{Coordinating a portfolio: the Pairwise Coordination Policy}\label{sec:algorithm}

We now consider portfolios of tasks that share one human. Each agent still faces the tradeoffs of Section~\ref{sec:structure}, but its choices also affect the other agents: when the human answers one question, service for other waiting questions is delayed. Exact dynamic programming captures these interactions but becomes intractable as the portfolio grows (Section~\ref{sec:setup}). We therefore seek a computationally efficient policy that accounts for this congestion.

\subsection{Task-local priorities can fail}\label{subsec:noindex}

A simple way to coordinate the agents is to assign each waiting question a priority score based only on its task, such as its weight or the probability that a search will find the requirement, and have the human answer the question with the highest score. Index policies of this kind are optimal (or numerically perform well) in many scheduling and bandit problems, including restless-bandit models of shared assistance \citep{dahiyaEtAl2022}, and they are easy to implement. However, the next proposition and example show that no such task-local score is optimal in our setting: which question should be answered first depends on the other tasks in the portfolio.

\begin{proposition}[No optimal task-local priority]\label{prop:noindex}
There is no score $r$, computed from a task's own parameters and current state together with the common parameters $a$ and $\eta$, such that in every portfolio each task with the largest score is an optimal choice for the next human service.
\end{proposition}

\begin{example}[Portfolio-dependent priority]\label{ex:priority}
Let $a=1$, $x_i=0$, $\alpha_i=1$, $c_i=0.01$, $\kappa_i=0.05$, and $\eta=0.05$. Tasks $A,B,C$ have $(w_i,p_i)=(4,0.75),(1,0.10),(0.5,0)$. Optimizing all remaining decisions conditional on first human service gives
\begin{center}
\begin{tabular}{lrrrr}
\toprule
Portfolio & Serve $A$ & Serve $B$ & Serve $C$ & Serve none\\
\midrule
$A,B$ & $6.100$ & $6.135$ & --- & $7.260$\\
$A,B,C$ & $7.650$ & $7.360$ & $7.860$ & $8.935$\\
\bottomrule
\end{tabular}
\end{center}
\end{example}

Proposition~\ref{prop:noindex} implies that the human's priorities must depend on the portfolio as a whole, and Example~\ref{ex:priority} shows it. Tasks $A$ and $B$ have the same parameters and states in both portfolios. In $\{A,B\}$, serving $A$ is uniquely optimal, which requires $r(A)>r(B)$. In $\{A,B,C\}$, serving $B$ is uniquely optimal, which requires $r(B)>r(A)$. The reversal has a simple explanation. Task $A$ is urgent and likely to be found in the records, task $B$ is unlikely to be found, and task $C$ can only be resolved by the human. In $\{A,B\}$, the urgent task is served first while $B$ searches, because $B$'s likely wait for a later answer costs little. Once $C$ is present, serving $A$ first would leave both $B$ and $C$ waiting for the human with high probability. Serving $B$ first while $A$ searches is better, because $A$'s search succeeds with probability $0.75$ and then leaves the human free for $C$. Appendix~\ref{app:proof-noindex} gives the calculation of the example. Although no task-local priority is optimal in general, approximate indices, including priorities that respond to a system-wide attention price, can still be useful, as numerical experiments in Appendix~\ref{app:num-mech} illustrate. A good policy should nonetheless account for how tasks interact through the human. The next subsection develops a pairwise approach.

\subsection{Pairwise prices of shared attention}\label{subsec:pairwise}

To account for interactions without solving the full dynamic program, we estimate the portfolio's future cost using smaller problems that can be solved exactly: each task on its own and each pair of tasks sharing one human. Let $V_i(s_i)$ be the optimal cost of task $i$ when it has its own human, and let $V_{ij}(s_i,s_j)$ be the optimal cost of tasks $i$ and $j$ when they share one human. Both problems use the same availability and cost parameters and the search caps from Lemma~\ref{lem:reform}. Appendix~\ref{app:implementation} describes how we compute these values using the reduced dynamic program~\eqref{eq:app-dp}. Define
\begin{align}
 \Delta_{ij}(s_i,s_j)&=V_{ij}(s_i,s_j)-V_i(s_i)-V_j(s_j),\label{eq:delta}\\
 \widehat V_{\rm PCP}(s)&=\sum_iV_i(s_i)+\sum_{i<j}\Delta_{ij}(s_i,s_j).\label{eq:critic}
\end{align}
We call $\widehat V_{\rm PCP}$ the \emph{pairwise approximation} and $\Delta_{ij}$ the \emph{pairwise price of shared attention}: the additional expected cost that tasks $i$ and $j$ incur when they share one human, compared with each having a human of its own. The price is nonnegative, because separate human capacity relaxes the pair's shared-capacity constraint (Lemma~\ref{app:lem-projection}). Resolved tasks contribute only their deterministic remaining-work cost and have zero prices with other tasks. Setting every price to zero gives the \emph{singleton approximation} $\sum_iV_i(s_i)$, which captures each task's own tradeoffs but ignores congestion. Before using the approximation, we identify cases in which it gives the optimal value.

\begin{proposition}[Optimal cases]\label{prop:pair_exact}
The pairwise approximation~\eqref{eq:critic} gives the optimal value when at most two requirements remain unresolved. It also gives the optimal value for any number of unresolved tasks when no remaining search can succeed and $x_i=0$ for every unresolved task. In the latter case,
\begin{equation}\label{eq:humanboundary}
 \Delta_{ij}=\frac{\min(w_i,w_j)}a,
\end{equation}
and descending-weight human service is optimal.
\end{proposition}

Proposition~\ref{prop:pair_exact} shows that the pairwise approximation loses nothing in two situations. The first is a consistency property of the construction: when at most two tasks still need their requirements, the approximation contains the exact value of that pair. The second arises when every remaining requirement must come from the human and no work remains. Serving in descending weight is then Smith's rule for jobs of equal length \citep{smith1956}, and the price of attention takes a transparent form. Appendix~\ref{app:proof-pair} gives the proof, and Corollary~\ref{app:cor-interp} extends the second case to tasks with at most one unit of work.

When neither task can avoid human service and neither has work left, the task served second waits an extra $1/a$ periods in expectation. Thus, the pairwise price equals the lower delay weight multiplied by this expected wait. For example, if the manager is available with probability $0.5$ and two blocked tasks have delay weights $4$ and $1$, the task with weight $1$ is served second and waits two extra periods on average, so $\Delta_{ij}=2$. For general portfolios, $\widehat V_{\rm PCP}$ need not be either a lower or an upper bound on $V^*$. Nevertheless, Section~\ref{sec:cert} establishes guarantees for PCP based on its action-selection rule.

\subsection{Pairwise Coordination Policy}\label{subsec:policy}

Pairwise Coordination Policy (PCP) uses the pairwise approximation as an estimate of the cost to go. In each period, it chooses the plan that minimizes the immediate cost plus the expected approximate cost of the next state. Formally, let $g_d(s)$ be the expected immediate cost of a period plan $d$ (Section~\ref{sec:setup}) and $P_d$ its transition law. Let $\mathcal D_+(s)$ contain the plans that satisfy the reductions of Lemma~\ref{lem:reform} and change the state with positive probability. Definition~\ref{app:def-reduced} in Appendix~\ref{app:dp} defines this set formally, and Proposition~\ref{app:prop-dp} shows that restricting attention to it preserves the optimum. PCP selects
\begin{equation}\label{eq:greedy}
 d_{\rm PCP}(s)\in\arg\min_{d\in\mathcal D_+(s)}
 \left\{g_d(s)+\sum_{s'}P_d(s,s')\widehat V_{\rm PCP}(s')\right\},
\end{equation}
executes the selected plan, observes the next state, and repeats. By Proposition~\ref{prop:pair_exact}, PCP is optimal from any state on the two exact boundaries when the minimization in~\eqref{eq:greedy} is exact. The \emph{singleton policy} makes the same choice with the singleton approximation in place of $\widehat V_{\rm PCP}$. Comparing the two policies isolates the value of pricing shared attention. Each answer occupies the human for a period and delays the other tasks waiting for answers. The pairwise prices charge this delay to the decision that causes it, whereas the singleton policy ignores it and therefore tends to ask too often and search too little. Section~\ref{sec:numerics} shows that this congestion effect is the main source of PCP's advantage. Algorithm~\ref{alg:pcp} summarizes PCP.

When multiple plans attain the minimum in~\eqref{eq:greedy}, PCP breaks ties lexicographically: it first prefers plans that serve a question, then plans that submit one, and finally plans that use fewer searches. The discussion following Theorem~\ref{thm:pcp_guarantee} explains the rationale for this rule, and Section~\ref{subsec:computation} describes how PCP is computed.

\begin{algorithm}[!htb]
\caption{Pairwise Coordination Policy (PCP)}\label{alg:pcp}
\begin{algorithmic}[1]
\Require initial state $s$, parameters $a$, $\eta$, and $(w_i,p_i,\alpha_i,c_i,\kappa_i)_{i=1}^n$, caps $(M_i)_{i=1}^n$, and a tie-breaking rule
\Ensure a period plan $d_{\rm PCP}(s)$ in each period until all tasks complete
\While{some task is incomplete}
 \State obtain $V_i$ and $V_{ij}$ for unresolved tasks and pairs
 \State $d\gets$ a minimizer of~\eqref{eq:greedy} over $\mathcal D_+(s)$
 \State submit $d$'s question, if any, and observe availability
 \State execute $d$'s service and agent actions
 \State observe the outcomes and update $s$
\EndWhile
\end{algorithmic}
\end{algorithm}

\subsection{Performance guarantees}\label{sec:cert}

Because PCP relies on an approximation, we want to know how far its cost can be above the optimum. We consider two cases. When no remaining search can succeed, every unresolved requirement must be resolved by the human, and the problem has enough structure for a worst-case bound. When search remains useful, we give a bound that can be checked for each instance. For a policy $\pi$, let $J^\pi(s)$ denote its expected total cost from state $s$, so that $V^*(s)=\inf_\pi J^\pi(s)$. Let $J^{\rm PCP}(s)$ denote PCP's expected total cost from state $s$. Set $H=\widehat V_{\rm PCP}$, let $\omega(s)=\sum_{i\text{ incomplete}}w_i$ denote the delay cost in the current period, and define the one-period Bellman defect
\begin{equation}\label{eq:pcp_residual}
 \rho(s)=\min_{d\in\mathcal D_+(s)}\bigl\{g_d(s)+P_dH(s)\bigr\}-H(s),
\end{equation}
where $P_dH(s)=\sum_{s'}P_d(s,s')H(s')$, with $H=\rho=0$ at terminal states. The next theorem bounds PCP's cost relative to the optimum, first when every unresolved requirement must be resolved by the human and then in the full model.

\begin{theorem}[PCP performance guarantees]\label{thm:pcp_guarantee}
Use exact singleton and pair values, the caps of Lemma~\ref{lem:reform}, and exact minimization in~\eqref{eq:greedy}.
\begin{enumerate}[label=(\roman*)]
 \item \textbf{Human-resolution regime.} Suppose no remaining search can succeed. Let $\mathcal A$ contain the unresolved tasks and define
 \begin{align*}
 K&=\sum_{i:z_i=R}w_ix_i+
       \sum_{i\in\mathcal A}(\eta+\kappa_i\ind{z_i=U}),\\
 B&=\sum_{i\in\mathcal A}w_i\max\{x_i,1\},\qquad
 w_{\mathcal A}=\sum_{i\in\mathcal A}w_i,\\
 Q&=\frac1a\left(w_{\mathcal A}+\sum_{\substack{i<j\\i,j\in\mathcal A}}\min(w_i,w_j)\right).
 \end{align*}
 For every tie-breaking rule,
 \begin{equation}\label{eq:pcp_twofactor}
 J^{\rm PCP}(s)\le K+B+Q-w_{\mathcal A}
 \le 2V^*(s)-K-w_{\mathcal A}\le2V^*(s).
 \end{equation}
 If ties are broken toward plans that neither submit nor serve, there are instances in this regime with $J^{\rm PCP}/V^*$ arbitrarily close to $\varphi=(1+\sqrt5)/2\approx1.618$.
 \item \textbf{Full search model.} On a set of capped states containing the initial state and closed under every positive-probability successor of every plan in $\mathcal D_+(s)$, suppose constants $\epsilon_-\ge0$ and $\epsilon_+\in[0,1)$ satisfy
 \begin{equation}\label{eq:residual_condition}
 -\epsilon_-\omega(s)\le\rho(s)\le\epsilon_+\omega(s)
 \end{equation}
 at every unfinished state. Then
 \begin{equation}\label{eq:pcp_full_ratio}
 J^{\rm PCP}(s)\le\frac{H(s)}{1-\epsilon_+}
 \le\frac{1+\epsilon_-}{1-\epsilon_+}V^*(s).
 \end{equation}
\end{enumerate}
\end{theorem}

Part~(i) shows that when no search can still succeed, PCP costs at most twice the optimum for any availability, weights, work, and fees, whichever way ties are broken. The bound is computable from the current state alone, without $V^*$ or an evaluation of PCP. Appendix~\ref{app:human} shows that it holds more generally, whenever the search caps of Lemma~\ref{lem:reform} rule out further search.

The guarantee relies on a property of PCP's decisions (Lemma~\ref{app:lem-priority} in Appendix~\ref{app:human}): once a task has at most one unit of work remaining, PCP never serves a lower-weight task ahead of it and never leaves an available human-service opportunity unused while the task remains unresolved. Each additional period of service delay can then be charged to a distinct pair, at a cost of at most $\min(w_i,w_j)/a$—the pairwise price in Proposition~\ref{prop:pair_exact}. The factor of two comes from comparing PCP's cost with two lower bounds on the optimum: one based on remaining work and the other on the queue for human answers.

Although the bound holds for every tie rule, ties affect PCP's actual cost, which motivates the tie rule of Section~\ref{subsec:policy}. When the human is always available and tasks still have ample work, a pairwise price can be zero even if the human cannot answer all the tasks in time, because any two of them can be answered before either needs its answer. Answering a question now and answering it later then have the same approximate value, and a rule that postpones questions at such ties may wait until the prices turn positive, when the human can no longer keep up. This is the source of the lower bound $\varphi$ in part~(i), which Appendix~\ref{app:golden} derives for portfolios of identical tasks. On the same portfolios, preferring plans that serve a question is optimal. Appendix~\ref{app:proof-human} proves part~(i).

Part~(ii) covers portfolios in which search remains useful. It gives an instance-specific certificate: if the pairwise approximation nearly satisfies the Bellman equation, then PCP is nearly optimal. The condition depends only on the singleton and pair tables and the one-period plans, so it can be checked without computing $V^*$ or evaluating PCP. However, checking it at every state in the theorem's set costs about as much as solving the exact dynamic program. Thus, it is most useful for confirming PCP's quality on instances that can still be analyzed exactly, as in Section~\ref{subsec:num-guarantee}. Appendix~\ref{app:defect} expresses $\rho$ in terms of disagreements between singleton and pair optimal decisions and extends the bound to inexact minimization, and Appendix~\ref{app:proof-defect} gives the proof.

Intuitively, $\rho(s)$ measures the difference between the best one-period look-ahead value and the approximation itself. If this difference is small relative to the delay cost incurred in each period, approximation errors cannot accumulate faster than that cost. For example, if $\epsilon_-=\epsilon_+=0.05$, the ratio is at most $1.05/0.95\approx1.105$. The guarantee in part~(ii) is conditional: the model does not generally ensure that the residual bound in~\eqref{eq:residual_condition} holds. 

\subsection{Computation}\label{subsec:computation}

PCP requires $n$ singleton value tables and $n(n-1)/2$ pair value tables. Each pair table is computed by solving an exact two-task dynamic program over the product of the two tasks' state spaces; task $i$ has at most $3(b_i+1)(M_i+1)$ states. Tables can be reused across decisions. The pairwise approximation is a sum of singleton and pair terms, each depending on the search outcomes of at most two tasks, so evaluating one candidate plan requires $O(n^2)$ table lookups. Because the number of plans grows exponentially with $n$, our implementation uses coordinate search (Appendix~\ref{app:implementation}). For each of the $n+1$ submission choices and each availability outcome, coordinate search starts from two plans and applies a fixed number of sweeps over the service choice and the agents' search-or-work choices. A decision therefore evaluates $O(n^2)$ candidate plans and requires $O(n^4)$ table lookups. Coordinate search approximates the minimization in~\eqref{eq:greedy}. In the setting of Theorem~\ref{thm:pcp_guarantee}(i), where no search choices remain, it evaluates every plan, and the factor-two bound applies to it under a mild condition on its numerical tolerance that holds in all our experiments (Appendix~\ref{app:implementation}).

\section{Numerical study}\label{sec:numerics}

We evaluate PCP in two settings. Section~\ref{subsec:num-sim} uses synthetic portfolios in which search outcomes and the human's availability are drawn from the model. Section~\ref{subsec:num-llm} reports a case study in which LLM agents carry out the tasks themselves. Appendix~\ref{app:numerics} describes the instances, the implementation, and the evaluation in detail.

We compare PCP with four policies. The singleton policy of Section~\ref{sec:algorithm} makes the same greedy choice as PCP using the singleton approximation, so the difference between the two isolates the value of pricing shared attention. It breaks ties in the same way as PCP. The other three policies are simple rules for when to ask. \emph{Always Ask} never searches. Whenever no question is pending, it submits the question of the unresolved task with the largest weight, and the human answers whenever available. \emph{Random-order Ask} does the same but chooses the next question at random. \emph{Search Once Then Ask} lets each agent work until it is blocked, search once, and ask only if that search fails. For small portfolios, we also compute the optimal policy by dynamic programming.

\subsection{Synthetic experiments}\label{subsec:num-sim}

Our main benchmark contains 864 portfolios of three or four tasks. It crosses the amount of independent work $b\in\{0,1,2,4\}$, the availability $a\in\{0.25,0.5,1\}$, the recoverability $p\in\{0.3,0.7\}$, the search success probability $\alpha\in\{0.8,1\}$, and the search cost $c\in\{0.1,0.5,1.5\}$, with three random draws of the task weights from $\{0.5,1,2,4\}$ for each combination. The submission and answer fees are $\kappa_i=0.1$ and $\eta=0.5$. We evaluate every policy exactly and report its gap to the optimal cost and its saving relative to other policies, both averaged over portfolios. Section~\ref{subsec:num-scale} turns to larger portfolios, which we evaluate by simulation, and Section~\ref{subsec:num-robust} to misestimated inputs.

\subsubsection{Performance against the optimal policy}\label{subsec:num-bench}

Table~\ref{tab:gaps} reports the gap to the optimal cost on the benchmark.

\begin{table}[t]
\centering
\caption{Performance on 864 portfolios. Mean, 95th percentile (P95), and largest (Worst) gap to the optimal cost, $100(J-V^*)/V^*$ (\%). Opt.\ is the share of portfolios in which the policy's cost equals the optimal cost up to a relative difference of $10^{-8}$ (\%). Saving is PCP's mean saving relative to the policy, $100(J-J^{\rm PCP})/J$, averaged over portfolios (\%).}
\label{tab:gaps}
\small\setlength{\tabcolsep}{4pt}
\begin{tabular}{@{}lrrrrr@{}}
\toprule
Policy & Mean & P95 & Worst & Opt. & Saving\\
\midrule
 PCP & \textbf{0.07} & \textbf{0.32} & \textbf{1.87} & \textbf{56} & \\
 PCP, coordinate search & \textbf{0.07} & 0.35 & \textbf{1.87} & \textbf{56} & \\
 PCP, no-submission ties & 0.52 & 4.37 & 15.2 & 47 & 0.43\\
 Singleton & 1.24 & 5.78 & 21.0 & 38 & 1.10\\
 Always Ask & 30.9 & 119.2 & 236.2 & 19 & 18.2\\
 Random-order Ask & 58.8 & 188.2 & 330.9 & 9 & 30.6\\
 Search Once Then Ask & 43.1 & 80.2 & 138.3 & 0 & 28.7\\
\bottomrule
\end{tabular}
\end{table}

PCP is close to optimal throughout. Its mean gap is 0.07\%, its 95th percentile is 0.32\%, and its largest gap is 1.87\%. It is exactly optimal in 56\% of portfolios. Coordinate search loses almost nothing relative to exact minimization of~\eqref{eq:greedy} and attains the same cost in 853 of 864 portfolios. The singleton policy has a mean gap of 1.24\% and a largest gap of 21.0\%. The rules for when to ask are far from optimal: PCP saves 18.2\%, 30.6\%, and 28.7\% on average relative to Always Ask, Random-order Ask, and Search Once Then Ask. Always Ask comes closest when its two simplifications cost little: when the human is always available (mean gap 6.3\%), when long independent work hides the wait for an answer (10.5\% for $b=4$), or when search rarely succeeds (13.2\% for $p=0.3$). It is worst when the human is scarce (60.1\% for $a=0.25$). Search Once Then Ask has the opposite weaknesses. It pays for searches that are unlikely to succeed and delays questions that the human could answer at once, so its gap is largest for $p=0.3$ (51.9\%) and for $a=1$ (55.5\%).

The tie rule matters as well. Breaking ties toward not submitting raises PCP's largest gap from 1.87\% to 15.2\%. PCP's service-preferring rule is never materially worse, and it saves up to 13.2\%. Every portfolio in which it saves more than 1\% has four units of work per task and an always-available human, which is the synchronized-work mechanism discussed after Theorem~\ref{thm:pcp_guarantee} (Appendix~\ref{app:golden}).

Relative to the singleton policy, PCP saves 1.1\% on average and up to 17.4\%, and Figure~\ref{fig:regimes} shows where these savings arise. They are largest when search is costly and smallest when independent work is long. With four tasks and no independent work, PCP saves 3.9\% on average when the search cost is $1.5$ and almost nothing when it is $0.1$. The explanation lies in the search decisions of the singleton policy. When search is cheap, every agent searches regardless of the other tasks. When search is costly, the singleton policy searches much less, because it credits a search only with the agent's own benefit. In these portfolios with a search cost of $1.5$, the singleton policy makes 1.5 searches per portfolio on average, against 2.7 for both PCP and the optimal policy. A successful search also frees the human for the other tasks, and PCP accounts for this benefit through its pairwise prices. With long work ($b=4$), both policies can schedule questions around the work, and the savings shrink. The availability of the human has a smaller and less regular effect (panel~b).

\begin{figure}[t]
\centering
\includegraphics[width=\textwidth]{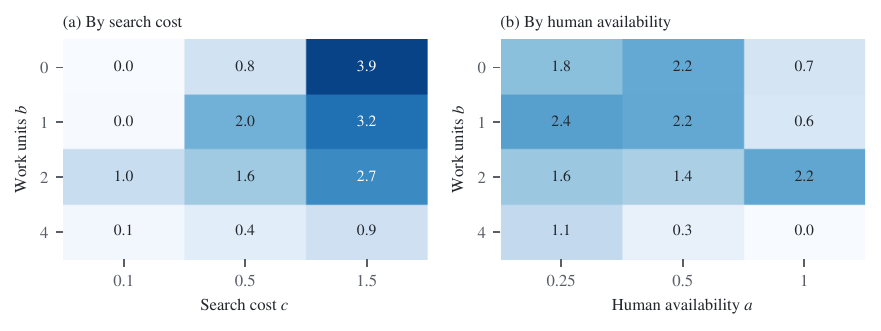}
\caption{Mean saving (\%) of PCP relative to the singleton policy for portfolios of four tasks, by the amount of independent work and (a) the search cost or (b) the availability of the human. Each cell averages 36 portfolios.}
\label{fig:regimes}
\end{figure}

This explanation predicts how PCP changes the use of the human, and the data confirm it. In 299 of the 317 portfolios in which PCP saves more than 0.5\% relative to the singleton policy, PCP uses fewer human answers and more searches. On average, PCP's expected number of answers per portfolio (2.53) is close to that of the optimal policy (2.55), whereas the singleton policy asks more (2.67). This finding complements Section~\ref{sec:structure}. The incentive to ask early when work can overlap an answer is private to each task, so the singleton policy already captures it. The congestion that each question imposes on the other tasks is shared, and pricing it leads agents to resolve more requirements themselves.

\subsubsection{The guarantees in practice}\label{subsec:num-guarantee}

We next examine how informative Theorem~\ref{thm:pcp_guarantee} is. For part~(i), we solve 180 portfolios of three to five tasks in which no search can succeed, with heterogeneous weights, up to six units of work, and $a\in\{0.5,1\}$. PCP is exactly optimal in 175 of them, and its largest ratio to the optimal cost is 1.06. The bound in~\eqref{eq:pcp_twofactor} lies between the optimal cost and 1.40 times the optimal cost, with a median of 1.18. The tie rule again matters. With ties broken toward not submitting, PCP is optimal in only 135 portfolios, and its largest ratio is 1.15. In the homogeneous family of Appendix~\ref{app:golden}, the worst ratio under that rule rises from 1.24 with five tasks to 1.54 with thirty, approaching $\varphi$, whereas PCP's service-preferring rule is optimal throughout.

For part~(ii), we compute conservative values of $\epsilon_+$ and $\epsilon_-$ (Appendix~\ref{app:num-evaluation}) over all states reachable from the initial state under any plan, for 815 benchmark portfolios. We omit 49 of the 54 four-task portfolios with $\alpha=0.8$ and $b=4$, whose state spaces are the largest. The condition $\epsilon_+<1$ holds in every one of them, with $\epsilon_+$ at most 0.71. The certified ratio $(1+\epsilon_-)/(1-\epsilon_+)$ has a median of 1.17, is at most $1.25$ in 67\% of portfolios, and is at most $2$ in 96\%. PCP's actual ratio never exceeds 1.019, so the certificate is informative but conservative.

\subsubsection{Larger portfolios}\label{subsec:num-scale}

For portfolios of up to 48 tasks, we compute PCP's decisions from precomputed singleton and pair tables (Appendix~\ref{app:implementation}) and estimate expected costs by simulation. The 56 instances cross the number of tasks $n\in\{4,8,12,16,24,32,48\}$ with $b\in\{0,2\}$, $a\in\{0.5,1\}$, and $p\in\{0.3,0.7\}$, with $\alpha=0.8$, $c=0.5$, and random weights. Every policy is simulated on the same 200 sample paths of availability and search outcomes, so that differences between policies are estimated precisely.

\begin{figure}[t]
\centering
\includegraphics[width=\textwidth]{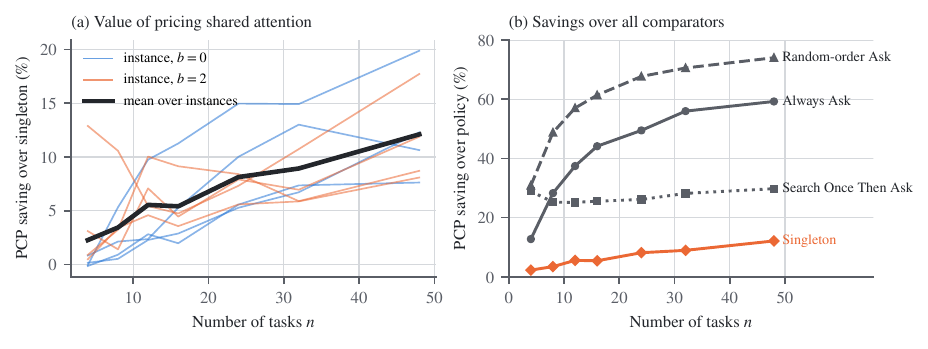}
\caption{Larger portfolios ($\alpha=0.8$, $c=0.5$, 200 common sample paths per instance). (a) PCP's mean saving relative to the singleton policy for each instance (thin lines, colored by the amount of work) and the mean over the eight instances of each size (thick line). (b) PCP's saving relative to each policy, averaged over the instances of each size.}
\label{fig:scale}
\end{figure}

Figure~\ref{fig:scale} shows that pricing shared attention becomes more valuable as the portfolio grows. PCP's saving relative to the singleton policy rises from 2.3\% with four tasks to 12.1\% with 48 tasks, and reaches 19.9\% in one instance. More tasks compete for the human, so the congestion that the singleton policy ignores grows. Consistent with Section~\ref{subsec:num-bench}, PCP uses fewer answers than the singleton policy, by 0.2 per portfolio with four tasks and 5.1 with 48. Its savings relative to Always Ask and Random-order Ask also grow, to 59\% and 74\% with 48 tasks, and its saving relative to Search Once Then Ask stays between 25\% and 30\% at every size.

PCP remains fast at this scale. A decision takes about 1\,ms for up to 12 tasks, 6\,ms for 24 tasks, and 44\,ms for 48 tasks on a personal computer with an Apple processor, after a one-time construction of the tables that takes 1.3\,s for 48 tasks.

\subsubsection{Robustness to misestimated inputs}\label{subsec:num-robust}

PCP uses the recoverability $p$, the search success probability $\alpha$, and the availability $a$, which a deployment would have to estimate. To assess sensitivity, we build PCP and the singleton policy with wrong inputs and evaluate them exactly under the true parameters, for the 144 four-task portfolios of the benchmark with the first weight draw. With correct inputs, PCP's mean gap on these portfolios is 0.10\%. Halving $p$, multiplying it by $1.5$, multiplying $\alpha$ by $0.75$, or shifting $a$ by $\pm0.2$ raises PCP's mean gap to at most 0.74\%, and its largest gap from 1.76\% to at most 10.26\%. Across the same scenarios, the singleton policy's mean gap ranges from 0.53\% to 4.02\%, and it is largest when recoverability is underestimated. When availability is underestimated by 0.2, the singleton policy does better than PCP, with a mean gap of 0.53\% against 0.74\% and a largest gap of 5.79\% against 10.26\%. On average, PCP is robust to these estimation errors, but individual portfolios can lose up to 10.3\%. Appendix~\ref{app:num-results} reports every scenario.

\subsection{A case study with LLM agents}\label{subsec:num-llm}

The experiments of Section~\ref{subsec:num-sim} simulate the outcomes of search and work. Here, LLM agents carry out the tasks themselves. The policies decide only when each agent searches, works, or asks, and which question the human answers.

\subsubsection{Design}\label{subsec:num-llm-design}

Each task is a procurement decision. The agent must choose the cheapest of four suppliers whose delivery time meets a maximum lead time that the person who delegated the task has not stated. The task has two units of independent work, which do not need the missing requirement. In the first, the agent extracts each supplier's current approved quote from a list that also contains superseded quotes, drafts, and quotes for other purchases. In the second, it computes each supplier's landed cost. A search gives the agent a packet of six documents about the purchase. The packet contains the binding requirement only if the search succeeds, and its other documents are decoys, such as superseded revisions, drafts, supplier recommendations, and requirements of other purchases. The agent must therefore judge whether the packet settles the requirement. The human is scripted and returns the true requirement when it answers. Once the requirement is known and the work is done, the agent chooses the supplier. Every step of every agent is carried out by \texttt{gpt-6.1-sol}, called through the Codex command-line interface. Each step is a separate call without tool use that receives only the data the step needs, including the agent's own earlier outputs.

We use portfolios of four tasks of weight one, two units of work per task, $p=0.3$, $\alpha=0.8$, $\kappa_i=0.1$, and $\eta=0.5$, and cross the availability $a\in\{0.25,1\}$ with the search cost $c\in\{0.1,1.5\}$. For each of the four settings, we draw ten portfolios and run PCP, the singleton policy, Always Ask, and Search Once Then Ask on each. All policies face the same documents, the same outcome for each search attempt, and the same availability of the human in each period. Appendix~\ref{app:num-llm} gives the documents, the prompt, and the execution protocol.

\subsubsection{Results}\label{subsec:num-llm-results}

\begin{table}[!t]
\centering
\caption{Case study with LLM agents (four tasks, ten portfolios per setting). Mean cost, with the lowest in each setting in bold, and PCP's mean saving (\%) relative to each benchmark, with standard deviations across portfolios in parentheses. Search Once denotes Search Once Then Ask.}
\label{tab:llm}
\small
\begin{tabular}{@{}cc rrrr rrr@{}}
\toprule
& & \multicolumn{4}{c}{Mean cost} & \multicolumn{3}{c}{PCP's saving (\%) relative to}\\
\cmidrule(lr){3-6}\cmidrule(l){7-9}
$a$ & $c$ & PCP & Singleton & Always Ask & Search Once & Singleton & Always Ask & Search Once\\
\midrule
 1 & 0.1 & \textbf{12.6} (1.1) & 12.9 (0.9) & 13.4 (0.0) & 20.1 (4.6) & 2.0 (7.2) & 5.8 (7.9) & 34.8 (12.1)\\
 1 & 1.5 & \textbf{13.4} (0.0) & \textbf{13.4} (0.0) & \textbf{13.4} (0.0) & 25.3 (4.2) & 0.0 (0.0) & 0.0 (0.0) & 45.6 (9.7)\\
 0.25 & 0.1 & \textbf{31.5} (20.4) & 33.1 (20.8) & 42.1 (16.8) & 44.5 (22.5) & 2.5 (14.7) & 27.2 (24.5) & 23.8 (29.0)\\
 0.25 & 1.5 & \textbf{40.3} (25.3) & 46.6 (23.3) & 46.6 (23.3) & 45.4 (22.1) & 11.2 (27.2) & 11.2 (27.2) & 13.9 (20.7)\\
\bottomrule
\end{tabular}
\end{table}

The agents carried out every step correctly. All 640 quote extractions, all 640 cost calculations, and all 640 supplier choices are correct, and the agents recognized all 91 successful searches and rejected all 276 packets without the binding requirement. Because the agents always recognized whether a search had succeeded, the differences in cost between the policies come from coordination alone. This follows from the design of the documents and does not test the assumption in general.

Pricing shared attention matters only when the human is scarce and search is costly. When the human is always available and search is costly, PCP, the singleton policy, and Always Ask ask every question without searching, one per period, so they cost the same in every portfolio. When search is cheap, the singleton policy already searches often, and under the model PCP's saving relative to the singleton policy averages 0.0\% per portfolio with an available human and 0.6\% with a scarce one.

When the human is scarce and search is costly, PCP saves 11.2\% on average relative to the singleton policy, close to the average saving of 9.4\% that the model predicts (Appendix~\ref{app:num-llm}). The singleton policy never searches in this setting. For a single agent, a search that costs $1.5$ and succeeds with probability $0.24$ is not worth its cost, so every agent waits for the human. PCP searches 3.6 times per portfolio and asks 3.0 questions, against 4.0 for the singleton policy, because a successful search also frees the human for the other agents. This is the congestion effect of Section~\ref{subsec:num-bench}, here with agents that must read the documents to learn whether a search has succeeded.

The standard deviation of this saving is large (27.2\%), because the cost of a portfolio depends mainly on when the human happens to be available, and ten portfolios cover few of these histories. Since the agents made no errors, every decision in the run can also be evaluated under the model. In this setting, each decision of the singleton policy raises the expected cost by 1.50 on average relative to the optimal decision, compared with 0.011 for PCP. Under the model, PCP's advantage in this setting therefore does not hinge on the ten portfolios drawn. Relative to the rules for when to ask, PCP saves 27.2\% and 23.8\% on average when the human is scarce and search is cheap. Search Once Then Ask is the most expensive policy whenever the human is always available, because it delays questions that the human could answer at once.

\section{Conclusion}\label{sec:conclusion}

As organizations delegate more work to AI agents, the attention of the people who oversee them becomes a shared and scarce resource. We studied how agents should divide their effort between searching for missing information, continuing independent work, and asking a human whom they all share, and which waiting question the human should answer.

Three insights emerge. First, how often an agent asks is a poor measure of how well it uses human attention. Useful independent work and withdrawable questions can both make optimal agents ask more, and withdrawal does so while lowering total cost. Second, the value of a question depends on the whole portfolio. Each answer delays the other tasks that wait for the human, and adding a task can reverse which of two others should be answered first. Third, this congestion can be priced with small problems that are solved exactly. PCP's pairwise prices of attention lead agents to search more and ask less when the human is scarce. On average, PCP is within a fraction of a percent of the optimal cost in small portfolios, and its saving relative to a policy that ignores congestion grows with the number of tasks.

These insights bear directly on the design of agent systems. Interfaces should let agents keep working and searching after they ask, and should let them withdraw questions that they resolve themselves. Systems that run several agents for one person should coordinate their questions centrally and charge each question for the attention it consumes. The case study shows that this coordination pays off when the person is rarely available and searching is costly, and matters little when the person answers at once.

Our model has limitations that suggest directions for future work. First, an agent never acts on a guessed requirement, and every answer and successful search is correct. In practice, an agent can proceed on its best guess and redo work if the guess turns out wrong, and a record it retrieves can be outdated. Adding an action that proceeds on a guess at the risk of rework, together with searches that can return a wrong requirement, would turn the decision to ask into a choice between preventing errors and repairing them. Second, the human answers at most one question per period and is available independently over time. People often check in for a session, answer several questions at once, and stay available for a while. A model with sessions and persistent availability would make the timing of questions more strategic: questions should be ready when a session starts, and a human who has just become available is likely to remain so. Third, PCP takes the recoverability, search success, and availability probabilities as known, and Section 6.1.4 shows that misestimating them, especially underestimating availability, can be costly in individual portfolios. A controller that updates these estimates from observed searches and check-ins would face a trade-off between learning the parameters and acting on them. On the theoretical side, an unconditional constant-factor guarantee for portfolios in which search remains useful is still open.

\FloatBarrier

\bibliographystyle{informs2014}
\bibliography{references,literature_review}
\clearpage
\appendix
\numberwithin{equation}{section}
\numberwithin{figure}{section}
\numberwithin{table}{section}
\numberwithin{proposition}{section}
\numberwithin{corollary}{section}
\numberwithin{definition}{section}
\numberwithin{example}{section}
\numberwithin{lemma}{section}
\section{Additional Results and Discussion}
\label{app:additional}

This appendix collects results that complement the main text. All proofs are in Appendix~\ref{app:proofs}.

\paragraph{Conventions}
All values are remaining costs from the displayed state, and fees paid earlier are sunk. The delay cost of a period is charged to every task that is incomplete at its start, and work and resolution achieved during a period take effect at its end. For a joint state $s$, let $\mathcal A(s)=\{i:z_i\ne R\}$ be the set of unresolved tasks, $m=|\mathcal A(s)|$, and $\omega(s)=\sum_{i\text{ incomplete}}w_i$ the delay cost incurred in the current period. Unless stated otherwise, sums over pairs $i<j$ run over unresolved tasks. A state is terminal when every task is complete.

\subsection{The reduced dynamic program}
\label{app:dp}

\paragraph{Beliefs}
Let $\Lambda_i(k)=1-p_i+p_i(1-\alpha_i)^k$ be the probability that the first $k$ searches of task $i$ fail, and let
\begin{equation}\label{eq:app-posterior}
 \beta_i(k)=\frac{p_i(1-\alpha_i)^k}{\Lambda_i(k)}
\end{equation}
be the posterior probability that its requirement is recoverable after $k$ failures. The next-search success probability~\eqref{eq:belief} is $q_i(k)=\alpha_i\beta_i(k)$. Both $\beta_i(k)$ and $q_i(k)$ are nonincreasing in $k$, and $q_i(k)\to0$ (Lemma~\ref{app:lem-belief}). Because recoverability and search outcomes are independent across tasks, task $i$'s future evolution depends on its history only through $(x_i,k_i,z_i)$, so the joint state is a sufficient statistic.

\paragraph{Search caps}
The next result shows that the caps of Lemma~\ref{lem:reform}(iii) are sufficient. Here $b_i$ may be any upper bound on task $i$'s remaining work, for example its initial amount.

\begin{proposition}[Sufficient search caps]\label{app:prop-caps}
For a portfolio of $n$ tasks with remaining work at most $b_i$, let
\begin{equation}\label{eq:app-bmax}
 B_{\max}=\sum_iw_ib_i+\frac na\sum_iw_i+\sum_i\kappa_i+n\eta,
 \qquad
 M_i=\min\{k\ge0:q_i(k)B_{\max}\le c_i\}.
\end{equation}
\begin{enumerate}[label=(\roman*)]
 \item Every $M_i$ is finite, and $0\le V^*(s)\le B_{\max}$ at every state reachable from the initial state.
 \item Prohibiting every search of task $i$ at failure counts $k_i\ge M_i$, for all $i$, does not change the optimal value. The same holds for any larger caps $M_i'\ge M_i$.
 \item For every subset $I$ of tasks, the optimal value of the tasks in $I$ sharing one human, computed with the caps $(M_i)_{i\in I}$, equals its value without caps.
\end{enumerate}
\end{proposition}

Part~(iii) is what allows PCP's singleton and pair values to be computed with the caps of the full portfolio. We call the model with the caps of Proposition~\ref{app:prop-caps} the \emph{capped model}. Its state space is finite: $x_i\in\{0,\ldots,b_i\}$, $k_i\in\{0,\ldots,M_i\}$, and $z_i\in\{U,P,R\}$.

\paragraph{Reduced plans}
A period plan $d$ specifies the question submitted at the start of the period, if any, and, for each availability outcome of positive probability, the question served, if any, and each agent's action: search, work, or wait. Let $g_d(s)$ be its expected immediate cost, which is $\omega(s)$ plus the expected submission, search, and answer fees, and let $P_d(s,\cdot)$ be its transition law.

\begin{definition}[Reduced plans]\label{app:def-reduced}
At an unfinished state $s$ of the capped model, $\mathcal D_{\rm red}(s)$ is the set of feasible plans $d$ such that
\begin{enumerate}[label=(\alph*)]
 \item at most one new question is submitted, and a newly submitted question is served if the human is available;
 \item no agent searches while its question is being served;
 \item no task is searched at a failure count $k_i\ge M_i$;
 \item in each availability outcome, every agent with $x_i>0$ performs one unit of work unless it searches.
\end{enumerate}
Let $\mathcal D_+(s)=\{d\in\mathcal D_{\rm red}(s):P_d(s,s)<1\}$.
\end{definition}

Rule~(d) applies to every agent with work left: to agents whose question is being served, to agents with a pending or unsubmitted question, and to agents whose requirement is already resolved. Human service is optional, and an agent without work may wait. Rule~(a) restricts new submissions, not the number of pending questions, because a question submitted in a period in which the human is unavailable remains pending. The set $\mathcal D_+(s)$ is nonempty at every unfinished state: if some task has work, every plan in $\mathcal D_{\rm red}(s)$ changes the state. Otherwise, submitting an unresolved task's question, or serving it if it is already pending, changes the state with positive probability.

\paragraph{The dynamic program}
Define the integer rank
\begin{equation}\label{eq:app-rank}
 \mathrm{rk}(s)=\sum_ix_i+\sum_{i\in\mathcal A(s)}\bigl[(M_i-k_i)^++1+\ind{z_i=U}\bigr].
\end{equation}

\begin{proposition}[Ranked dynamic program]\label{app:prop-dp}
In the capped model, every transition that changes the state strictly decreases $\mathrm{rk}$. The optimal value satisfies $V^*(s)=0$ at terminal states and, at every unfinished state,
\begin{equation}\label{eq:app-dp}
 V^*(s)=\min_{d\in\mathcal D_+(s)}
 \frac{g_d(s)+\sum_{s'\ne s}P_d(s,s')V^*(s')}{1-P_d(s,s)}.
\end{equation}
This value equals the optimal value~\eqref{eq:objective} of the original model, without caps and with unrestricted plans, and the recursion gives the same values if the minimum is taken over all feasible plans with $P_d(s,s)<1$. There is an optimal deterministic stationary policy that selects a plan of $\mathcal D_+(s)$ at every unfinished state. In particular, the three properties of Lemma~\ref{lem:reform} hold simultaneously.
\end{proposition}

Each ratio in~\eqref{eq:app-dp} is the expected cost of repeating plan $d$ until the state changes and then continuing optimally. Because every distinct successor has lower rank, \eqref{eq:app-dp} computes $V^*$ by induction on rank, starting from terminal states. From an initially unsubmitted portfolio, the induction has depth $\mathrm{rk}(s_0)=\sum_i(b_i+M_i+2)$.

\paragraph{The one-period Bellman operator}
For a function $L$ on the capped state space, with $L=0$ at terminal states, define at unfinished states
\begin{equation}\label{eq:app-T}
 (\mathcal TL)(s)=\min_{d\in\mathcal D_+(s)}\bigl\{g_d(s)+P_dL(s)\bigr\},
 \qquad P_dL(s)=\sum_{s'}P_d(s,s')L(s').
\end{equation}
Unlike~\eqref{eq:app-dp}, the operator $\mathcal T$ keeps the self-transition. It has $V^*$ as a fixed point, and PCP~\eqref{eq:greedy} and the Bellman defect~\eqref{eq:pcp_residual} apply it to the pairwise approximation.

\subsection{The complete policy for a blocked task}
\label{app:blocked}

Consider one task with no work left and suppress its index. Let $V^U_k$ and $V^P_k$ be the optimal values of the unsubmitted and pending states with $k$ failures, $(0,k,U)$ and $(0,k,P)$, and let $M$ be the cap of Proposition~\ref{app:prop-caps}. Define
\begin{equation}\label{eq:app-blocked-notation}
 W=\eta+\frac wa,\qquad
 Y_k=c+(1-q_k)V^P_{k+1},\qquad
 \psi_k=w+Y_k-V^P_k-\kappa q_k,
\end{equation}
where $q_k=q(k)$, $V^P_k=W$ for $k\ge M$, and $Y_k=c+(1-q_k)W$ for $k\ge M$. Here $W$ is the cost of waiting for the human without further search, because the expected wait is $1/a$ periods. $Y_k$ is the cost of one more search from a pending state, including its continuation but excluding the current period's delay cost, because a successful search completes a blocked task. Finally, $\psi_k$ is the cost of searching once more and asking after a failure, minus the cost of asking now.

\begin{proposition}[Complete blocked-task policy]\label{app:prop-blocked}
For one task with $x=0$:
\begin{enumerate}[label=(\roman*)]
 \item \textbf{Pending values.} $V^P_k=\min\{w+Y_k,\;w+a\eta+(1-a)Y_k,\;W\}$, and both $V^P_k$ and $Y_k$ are nondecreasing in $k$.
 \item \textbf{Pending policy.} The following pending policy is optimal. If $Y_k<\eta$, search in both availability outcomes, declining available service. If $\eta\le Y_k<W$, accept service when the human is available and search otherwise. If $Y_k\ge W$, accept service when available and wait otherwise. The last regime holds exactly when $q_kW\le c$. At the boundaries, the adjacent plans tie.
 \item \textbf{Asking threshold.} The sequence $\psi_k$ is nondecreasing on $\{0,\ldots,M\}$ and $\psi_M\ge w>0$. With $k_{\rm ask}=\min\{k:\psi_k\ge0\}\le M$, the policy that searches while unsubmitted at $k<k_{\rm ask}$ and submits at $k=k_{\rm ask}$ is optimal. Moreover, $V^U_k=\kappa+V^P_k$ for $k\ge k_{\rm ask}$ and $V^U_k<\kappa+V^P_k$ for $k<k_{\rm ask}$.
 \item \textbf{Ordering.} The pending policy does not depend on $\kappa$, whereas $\psi_k$ depends on $\kappa$ only through $-\kappa q_k$. If $p>0$ and $0<\alpha<1$, then $k_{\rm ask}\to\infty$ as $\kappa\to\infty$. Conversely, there are parameters for which $k_{\rm ask}=0$ while a pending agent with no failures searches when the human is unavailable.
\end{enumerate}
\end{proposition}

Proposition~\ref{app:prop-blocked} completes Theorem~\ref{thm:thresholds}. In the first regime of part~(ii), a pending agent declines available service when one more search, together with its continuation, costs less than the answer fee. Because $Y_k\ge c$, this requires $\eta>c$. Because $Y_k$ is nondecreasing, a question that the agent is willing to have answered at one failure count remains so at every later count. Under the submission rule of part~(iii), therefore, declined service arises only for questions that are pending while search still looks promising, such as questions pending at the outset or questions submitted with a zero fee.

The zero-fee case also shows why Theorem~\ref{thm:thresholds} refers to a consistent choice at ties. If $\kappa=0$ and $Y_k<\eta$, then $\psi_k=0$, so submitting now and postponing submission are equally good, and the convention that resolves ties toward asking gives $k_{\rm ask}=0$. The resulting policy submits a question and then declines its service, which violates Lemma~\ref{lem:reform}(ii). The alternative convention, which submits at the first $k$ with $Y_k>\eta$, is also optimal and satisfies Lemma~\ref{lem:reform}(ii). Both conventions are threshold policies. With $\kappa>0$, submitting a question and then declining its available service in the same period is strictly dominated by postponing the submission.

Part~(iii) also gives $k_{\rm ask}$ explicitly as a function of the fee. Whenever $q_k>0$, $\psi_k\ge0$ is equivalent to $\kappa\le t_k$, where $t_k=(w+Y_k-V^P_k)/q_k$ does not depend on $\kappa$. Whenever $q_k=0$, which happens for every $k$ if $p=0$ and for $k\ge1$ if $\alpha=1$, $\psi_k=w+Y_k-V^P_k\ge0$ for every fee, because $V^P_k\le w+Y_k$ by part~(i), and we set $t_k=+\infty$. With this convention, $k_{\rm ask}=\min\{k:\kappa\le t_k\}$. For the example after Theorem~\ref{thm:thresholds}, with $w=\eta=1$, $a=0.5$, $c=0.1$, $p=0.6$, and $\alpha=0.5$, the recursion of part~(i) gives $(t_0,\ldots,t_6)\approx(1.60,\,3.06,\,5.96,\,11.72,\,23.00,\,46.13,\,93.07)$, which reproduces the values of $k_{\rm ask}$ reported there. The pending threshold follows from part~(ii): $q_4W\approx0.129>c>q_5W\approx0.067$.

Part~(iv) is the formal version of the statement that the two thresholds have no universal ordering. The example after Theorem~\ref{thm:thresholds} shows both orders: a large submission fee raises $k_{\rm ask}$ above the pending search threshold, and a small one gives $k_{\rm ask}=0$ while a pending agent keeps searching. Appendix~\ref{app:proof-blocked} gives a second example of the latter order.

\subsection{Implementing the Pairwise Coordination Policy}
\label{app:implementation}

\paragraph{Value tables}
PCP uses the values $V_i(s_i)$ and $V_{ij}(s_i,s_j)$ of one- and two-task portfolios, computed by~\eqref{eq:app-dp} with the caps of the full portfolio. By Proposition~\ref{app:prop-caps}(iii), these caps preserve the optimum of each subproblem, so the tables are exact optimal values and carry no truncation error. A resolved task has $V_i(x_i,k_i,R)=w_ix_i$, and its pair value with any other task is additive, because its remaining work uses only its own agent. Hence, with $C_R(s)=\sum_{i\notin\mathcal A(s)}w_ix_i$,
\begin{equation}\label{eq:app-critic-unresolved}
 \widehat V_{\rm PCP}(s)=C_R(s)+\sum_{i\in\mathcal A(s)}V_i(s_i)
 +\sum_{\substack{i<j\\ i,j\in\mathcal A(s)}}\Delta_{ij}(s_i,s_j),
\end{equation}
and only pairs of unresolved tasks require tables. The tables can be computed lazily at the states PCP visits and stored for reuse. In each availability outcome, the next states of distinct tasks are conditionally independent, and each takes at most two values. Evaluating $g_d(s)+P_d\widehat V_{\rm PCP}(s)$ for one plan therefore needs only the marginal law of each task and of each pair, which takes $O(n^2)$ table lookups.

\paragraph{The progress guard}
PCP restricts its minimization in~\eqref{eq:greedy} to $\mathcal D_+(s)$, which excludes plans that leave the state unchanged with probability one. The next lemma shows that this restriction alone ensures that PCP, and any other selector that respects it, completes the portfolio in finite expected time. Let
\begin{equation}\label{eq:app-epsa}
 \varepsilon_a=\begin{cases}\min\{a,1-a\},&0<a<1,\\ 1,&a=1.\end{cases}
\end{equation}

\begin{lemma}[Progress guard]\label{app:lem-progress}
Every plan in $\mathcal D_+(s)$ changes the state with probability at least $\varepsilon_a$. Consequently, under any policy, possibly history dependent or randomized, that selects a plan in $\mathcal D_+(s)$ at every unfinished state $s$, the expected time until all tasks are complete is at most $\mathrm{rk}(s)/\varepsilon_a$, and the expected total cost is finite.
\end{lemma}

\paragraph{Coordinate search}
The number of plans in $\mathcal D_+(s)$ grows exponentially in $n$, because each agent chooses between search and work in each availability outcome. Given the submission decision, the objective in~\eqref{eq:greedy} is the probability-weighted sum of two terms, one for each availability outcome, so the two outcomes can be optimized separately. The progress guard couples them only through the requirement that at least one outcome change the state, which can be handled by retaining, for each outcome, the best candidate that changes the state and the best candidate that does not. Our implementation enumerates the submission decision: no submission, or the question of one unsubmitted task. For each submission decision and availability outcome, it optimizes the remaining choices by coordinate search from two deterministic starting plans, one in which no agent searches and one in which every eligible agent searches. Each sweep optimizes the service choice over all eligible questions and then each agent's search-or-work choice in turn. The reference implementation uses two sweeps. It then selects the best combination of the visited candidates that satisfies the guard. Plans whose objectives exceed the minimum $m$ by at most the absolute tolerance $\tau(s)=10^{-12}\max\{1,|m|\}$ are treated as tied. Among tied plans, PCP prefers a plan that serves a question, then a plan that submits a question, then a plan with fewer searches, and finally the first plan in a fixed order of task indices. We call this the \emph{service-preferring} rule. The \emph{no-submission} rule, used for comparison, prefers no submission and then no service among tied plans. Appendix~\ref{app:numerics} reports the settings used in each experiment.

Coordinate search need not find a minimizer of~\eqref{eq:greedy}, so results that assume exact minimization apply to it only where its candidates include a minimizer. One case is important: in the human-resolution regime of Appendix~\ref{app:human}, no search choices remain, and a sweep that tries every eligible service choice evaluates every plan in $\mathcal D_+(s)$. Because of the tolerance, the selected plan is then within $\tau(s)$ of the minimum. This does not affect Lemma~\ref{app:lem-priority} or Theorem~\ref{thm:pcp_guarantee}(i) as long as, at every state visited, $\tau(s)$ is smaller than the weight of every unresolved task and every positive difference between the weights of two unresolved tasks, because the objective differences used in their proofs are at least that large (Appendix~\ref{app:proof-human}). The condition concerns the absolute tolerance $\tau(s)$, which grows with the objective, not the factor $10^{-12}$. In our experiments every objective is below $10^5$, so $\tau(s)<10^{-7}$, whereas every weight and every positive difference between two weights is at least $0.5$. Appendix~\ref{app:defect} quantifies the effect of inexact minimization in general, and Appendix~\ref{app:golden} shows that the tie rule matters.

\subsection{The price of attention when human service is unavoidable}
\label{app:human}

We say that a state $s$ is in the \emph{human-resolution regime} if no unresolved task has an admissible search, that is, $k_i\ge M_i$ for every $i\in\mathcal A(s)$. Because $q_i(k_i)=0$ implies $k_i\ge M_i$, the regime contains every state at which no remaining search can succeed, which is the setting of Theorem~\ref{thm:pcp_guarantee}(i). The proof of the theorem uses only that no admissible search remains, so the theorem, like every result in this subsection, holds throughout the regime. In this regime, every unresolved task needs exactly one human answer, every agent works until its work is finished, and the regime is invariant. Throughout this subsection, write $\bar a=1-a$, with the convention $0^0=1$ when $a=1$, let $f_i=\eta+\kappa_i\ind{z_i=U}$ be task $i$'s unpaid fees, and let $G$ be a geometric random variable on $\{1,2,\ldots\}$ with success probability $a$.

\begin{lemma}[Singleton and pair values]\label{app:lem-humanpair}
In the human-resolution regime, for unresolved tasks $i$ and $j$,
\begin{align}
 V_i(s_i)&=f_i+w_i\left(x_i+\frac{\bar a^{x_i}}a\right),\label{eq:app-human-singleton}\\
 \Delta_{ij}(s_i,s_j)&=\frac1a\,\E\!\left[\min\Bigl\{w_i\bar a^{(x_i-G)^+},\;w_j\bar a^{(x_j-G)^+}\Bigr\}\right].\label{eq:app-human-pair}
\end{align}
The price $\Delta_{ij}$ is nondecreasing in each task's weight, nonincreasing in each task's remaining work, and satisfies $0\le\Delta_{ij}\le\min(w_i,w_j)/a$, with equality on the right when $x_i,x_j\le1$. When $a=1$, $\Delta_{ij}=\min(w_i,w_j)\ind{x_i\le1,\ x_j\le1}$.
\end{lemma}

The term $w_i\bar a^{(x_i-t)^+}/a$ in~\eqref{eq:app-human-pair} is the expected additional delay cost of task $i$ if its answer is postponed from the first available period $t$ to the next one: postponement matters only if the next opportunity arrives after task $i$'s work is finished, which happens with probability $\bar a^{(x_i-t)^+}$. The pairwise price of attention is the expected cost of postponing the cheaper of the two tasks. Independent work therefore hides congestion: the more work both tasks have, the more likely the human can answer both before either task needs its answer. Proposition~\ref{prop:pair_exact}'s price $\min(w_i,w_j)/a$ is the case in which neither task has more than one unit of work. Appendix~\ref{app:golden} shows that, when $a=1$, this can hide a bottleneck that only groups of three or more tasks create.

The next lemma formalizes the property of PCP described after Theorem~\ref{thm:pcp_guarantee}.

\begin{lemma}[PCP service priorities]\label{app:lem-priority}
Let $s$ be in the human-resolution regime, and let PCP select an exact minimizer of~\eqref{eq:greedy} with any tie rule. If some unresolved task $i$ has $x_i\le1$, then PCP serves a question whenever the human is available, and every task $j$ it serves satisfies $w_j\ge w_i$.
\end{lemma}

The proof of Theorem~\ref{thm:pcp_guarantee}(i) also yields a sharper, interpolating form of the factor-two bound.

\begin{corollary}[Interpolating bound and an exact boundary]\label{app:cor-interp}
In the human-resolution regime, let $K$, $B$, $Q$, and $w_{\mathcal A}$ be as in Theorem~\ref{thm:pcp_guarantee}(i), and let PCP select exact minimizers. Then
\begin{equation}\label{eq:app-interp}
 \frac{J^{\rm PCP}(s)}{V^*(s)}\le1+\frac{\min\{B,Q\}-w_{\mathcal A}}{K+\max\{B,Q\}}\le2.
\end{equation}
If, in addition, $x_i\le1$ for every unresolved task, then $\widehat V_{\rm PCP}(s)=V^*(s)=K+Q$, and PCP is optimal from $s$.
\end{corollary}

The ratio bound in~\eqref{eq:app-interp} is close to one when either independent work ($B$) or congestion ($Q$) dominates, and it can approach two only when $B$ and $Q$ are comparable and both are large relative to $K$ and $w_{\mathcal A}$. Hard instances for pairwise prices therefore balance work against congestion, as in the family of Appendix~\ref{app:golden}. The second statement extends the exact boundary of Proposition~\ref{prop:pair_exact} from tasks without work to tasks with at most one unit of work, because an answer received in the next available period overlaps that unit.

\subsection{Tie-breaking and a golden-ratio example}
\label{app:golden}

Theorem~\ref{thm:pcp_guarantee}(i) holds for every tie rule. The next lemma shows that ties can nonetheless matter, because the pairwise approximation can assign the same value to serving a question now and serving it later.

\begin{lemma}[Homogeneous portfolios]\label{app:lem-phi}
Consider $n\ge2$ initially unsubmitted tasks with $w_i=1$, common work $b_i=b$, and $p_i=0$, and let $a=1$ and $\kappa_i=\eta=0$. Let PCP use the no-submission tie rule of Appendix~\ref{app:implementation}. Then
\begin{equation}\label{eq:app-phi}
 V^*=\sum_{j=1}^n\max\{b,j\},\qquad
 J^{\rm PCP}=\begin{cases}n(n+1)/2,&b\le1,\\ nb+(n-1)(n-2)/2,&b\ge2.\end{cases}
\end{equation}
On this class, $J^{\rm PCP}/V^*\le\varphi=(1+\sqrt5)/2$, and the ratio tends to $\varphi$ along $b=\lfloor n/\varphi\rfloor$ as $n\to\infty$. PCP with any tie rule that prefers, among tied plans, a plan that serves a question, including the service-preferring rule, is optimal on this class.
\end{lemma}

For example, with $n=12$ and $b=9$, the optimal cost is $114$, whereas PCP with the no-submission rule costs $163$, about 43\% more. Every pair of tasks can be answered within its remaining work, so every pairwise price at the next state is zero until two units of work remain, and serving a question early has the same value as waiting. The no-submission rule then postpones all questions until the prices turn positive, at which point the human can no longer keep up. The bottleneck is visible only to groups of three or more tasks. For the no-submission rule, no factor smaller than $\varphi$ can therefore replace the factor two of Theorem~\ref{thm:pcp_guarantee}(i). We do not know whether the factor two is tight for the service-preferring rule or for any other tie rule.

\subsection{The pairwise Bellman defect}
\label{app:defect}

This subsection analyzes Theorem~\ref{thm:pcp_guarantee}(ii). Let $\mathcal S$ be a finite set of states of the capped model that contains the initial state and is \emph{forward closed}: for every unfinished $s\in\mathcal S$ and every $d\in\mathcal D_+(s)$, every $s'$ with $P_d(s,s')>0$ belongs to $\mathcal S$. The set of all capped states reachable from the initial state is forward closed. Write $H=\widehat V_{\rm PCP}$, with $H=0$ at terminal states, so that $\rho=\mathcal TH-H$ by~\eqref{eq:pcp_residual} and~\eqref{eq:app-T}. For a selector $\pi$ that chooses $\pi(s)\in\mathcal D_+(s)$, for example coordinate search, define the selection error
\begin{equation}\label{eq:app-selection-error}
 e_\pi(s)=g_{\pi(s)}(s)+P_{\pi(s)}H(s)-(\mathcal TH)(s)\ge0.
\end{equation}
Exact PCP has $e_\pi\equiv0$.

\paragraph{Decomposition}
Allocate a plan's expected immediate cost to tasks by
\[
 g_i^d(s)=w_i\ind{i\text{ incomplete}}+c_i\Pr_d\{i\text{ searches}\}
 +\kappa_i\ind{i\text{ submits}}+\eta\Pr_d\{i\text{ is served}\},
\]
so that $\sum_ig_i^d=g_d$. Define the singleton and pair regrets of plan $d$,
\begin{align*}
 \lambda_i^d(s)&=g_i^d(s)+\E_d[V_i(S'_i)]-V_i(s_i),\\
 \lambda_{ij}^d(s)&=g_i^d(s)+g_j^d(s)+\E_d[V_{ij}(S'_i,S'_j)]-V_{ij}(s_i,s_j).
\end{align*}

\begin{lemma}[Pairwise decomposition of the Bellman defect]\label{app:lem-defect}
Every $\lambda_i^d$ and $\lambda_{ij}^d$ is nonnegative. With $n$ the total number of tasks, including resolved ones,
\begin{equation}\label{eq:app-defect-full}
 g_d+P_dH-H=\sum_{i<j}\lambda_{ij}^d-(n-2)\sum_i\lambda_i^d,
\end{equation}
where the sums run over all tasks. Equivalently, with $m=|\mathcal A(s)|$,
\begin{equation}\label{eq:app-defect}
 \rho(s)=\min_{d\in\mathcal D_+(s)}\left\{\sum_{\substack{i<j\\ i,j\in\mathcal A(s)}}\lambda_{ij}^d(s)-(m-2)\sum_{i\in\mathcal A(s)}\lambda_i^d(s)\right\},
\end{equation}
where an empty sum is zero.
\end{lemma}

The regret $\lambda_i^d$ measures how much worse plan $d$, projected onto task $i$, is than an optimal decision for task $i$ if it had a human of its own, and $\lambda_{ij}^d$ is the analogous regret in the two-task problem. The defect is small when a single portfolio plan is nearly optimal for every pair problem at once, after the singleton regrets, each of which appears in $m-1$ pair regrets, are netted out. It is computed from the singleton and pair tables and one-period plans only. When $m\le2$, the defect is zero.

\paragraph{A guarantee with approximate selection}
The next result extends Theorem~\ref{thm:pcp_guarantee}(ii) to selectors that only approximate the minimization in~\eqref{eq:greedy}, and it separates the two halves of condition~\eqref{eq:residual_condition}.

\begin{proposition}[Defect guarantee with approximate selection]\label{app:prop-defect}
Let $\epsilon_+,\epsilon_-,\delta\ge0$ with $\epsilon_++\delta<1$.
\begin{enumerate}[label=(\roman*)]
 \item If $\rho(s)+e_\pi(s)\le(\epsilon_++\delta)\,\omega(s)$ at every unfinished state reachable from $s_0$ under $\pi$, then $J^\pi(s_0)\le H(s_0)/(1-\epsilon_+-\delta)$.
 \item If $\rho(s)\ge-\epsilon_-\,\omega(s)$ at every unfinished state of a forward-closed set $\mathcal S$ containing $s_0$, then $H(s_0)\le(1+\epsilon_-)V^*(s_0)$.
\end{enumerate}
In particular, if $-\epsilon_-\omega\le\rho\le\epsilon_+\omega$ and $e_\pi\le\delta\omega$ at every unfinished state of $\mathcal S$, then
\begin{equation}\label{eq:app-defect-ratio}
 J^\pi(s_0)\le\frac{H(s_0)}{1-\epsilon_+-\delta}\le\frac{1+\epsilon_-}{1-\epsilon_+-\delta}\,V^*(s_0).
\end{equation}
Exact PCP has $\delta=0$, which gives Theorem~\ref{thm:pcp_guarantee}(ii).
\end{proposition}

Part~(i) controls the policy. Its condition involves no minimization, because $\rho+e_\pi=g_{\pi(s)}+P_{\pi(s)}H-H$, and it needs to hold only at the states the selector reaches. Part~(ii) compares the approximation with the optimum. An optimal policy may visit states that PCP never reaches, so part~(ii) requires the lower condition on a forward-closed set. Checking it uses neither $V^*$ nor the cost of PCP, but it requires $\rho(s)$ at every state of such a set, which costs about as much as exact dynamic programming. Both conditions are sufficient, and the model does not guarantee them. Bounds of this type, which control a greedy policy through the Bellman residual of its value approximation, are standard for stochastic shortest-path problems \citep[e.g.,][]{hansen2011}.

\section{Proofs}
\label{app:proofs}

This appendix gives proofs for the main-text results in order of appearance, together with proofs of the related results in Appendix~\ref{app:additional}. Appendix~\ref{app:proof-belief} collects two facts about beliefs used in several proofs. The conventions of Appendix~\ref{app:additional} apply throughout.

\subsection{Belief dynamics}
\label{app:proof-belief}

\begin{lemma}[Posterior beliefs]\label{app:lem-belief}
For a task with $p<1$ and $\alpha\in(0,1]$, $\beta(k)$ in~\eqref{eq:app-posterior} is the posterior probability that the requirement is recoverable after $k$ failed searches, and $q(k)=\alpha\beta(k)$ is the success probability of the next search. Moreover,
\[
 \beta(k+1)=\frac{\beta(k)(1-\alpha)}{1-\alpha\beta(k)}\le\beta(k),
\]
so $\beta(k)$ and $q(k)$ are nonincreasing in $k$, and $q(k)\to0$ as $k\to\infty$.
\end{lemma}

\begin{proof}
Let $\chi\in\{0,1\}$ indicate recoverability. Conditional on $\chi=1$, the first $k$ searches fail with probability $(1-\alpha)^k$; conditional on $\chi=0$, they fail with probability one. Bayes' rule gives $\beta(k)=p(1-\alpha)^k/\Lambda(k)$, and the next search succeeds with probability $\alpha\beta(k)$, which is~\eqref{eq:belief}. Applying Bayes' rule to one more failure gives the displayed recursion. Its denominator satisfies $1-\alpha\beta(k)\ge1-\beta(k)=(1-p)/\Lambda(k)>0$, and $1-\alpha\le1-\alpha\beta(k)$ gives the inequality. Finally, $\Lambda(k)\ge1-p>0$ and $(1-\alpha)^k\to0$, so $\beta(k)\to0$ and $q(k)\to0$.
\end{proof}

Because search outcomes are conditionally independent given recoverability, a task with $k$ failures evolves exactly like a fresh task with prior $\beta(k)$ and the same $\alpha$. The next lemma shows that a more optimistic belief cannot hurt.

\begin{lemma}[Optimism cannot hurt]\label{app:lem-optimism}
Consider two portfolios that are identical except that task $i$'s posterior recoverability probability is $\beta_+$ in the first and $\beta_-\le\beta_+$ in the second. Then the optimal value of the first portfolio is at most that of the second. The same holds for the truncated problems of Appendix~\ref{app:proof-reform}. In particular, if $s$ and $s'$ differ only in that task $i$ has $k_i$ failures in $s$ and $k_i+1$ in $s'$, then $V^*(s)\le V^*(s')$.
\end{lemma}

\begin{proof}
If $\beta_+=0$, then $\beta_-=0$ and the portfolios coincide. Otherwise, let $\chi$ be task $i$'s recoverability indicator in the first portfolio, so $\Pr(\chi=1)=\beta_+$, and let $Z$ be an independent Bernoulli random variable with success probability $\beta_-/\beta_+$, drawn privately by the controller. Define a \emph{virtual} task $i$ that is recoverable exactly when $\chi Z=1$, and whose $t$th search succeeds exactly when the actual $t$th search succeeds and $Z=1$. The virtual task is recoverable with probability $\beta_-$; conditional on recoverability its searches succeed independently with probability $\alpha_i$; and otherwise they all fail. Its outcomes therefore have the law of task $i$ in the second portfolio, and the outcomes of the other tasks and of availability are unaffected.

Fix any policy $\pi$ for the second portfolio and run it on the virtual history, which consists of the actual history of the other tasks and the virtual history of task $i$. As long as the actual task $i$ is unresolved, execute $\pi$'s actions verbatim; then every virtual search success is also an actual one. If the actual task is resolved by a search that the virtual history records as a failure, which requires $Z=0$, continue to simulate $\pi$: record every later virtual search of task $i$ as a failure without performing it, letting the agent wait instead; do not submit task $i$'s question; and leave the human idle in any period in which $\pi$ serves task $i$. All other tasks receive the same actions and have the same outcomes. Task $i$ performs the same work in the same periods and is resolved no later than in the virtual history, so it completes no later, and its search, submission, and answer fees are no greater. Conditional on $Z$ and on the randomization of $\pi$, the constructed policy is a deterministic nonanticipative policy for the first portfolio. Its expected cost is therefore at least the optimal value of the first portfolio, and it is at most the cost of $\pi$ in the second. Taking the infimum over $\pi$ proves the claim.

The truncated problems of Appendix~\ref{app:proof-reform} require a fixed completion rule from period $T$ onward. The same construction applies, with the completion rule executed on the actual state; if the actual task $i$ is already resolved, $F$ skips it, and because $F$ serves in a fixed index order, skipping a task can only move every other task's answer earlier. The last statement follows because $\beta_i(k_i+1)\le\beta_i(k_i)$ by Lemma~\ref{app:lem-belief}.
\end{proof}

\subsection{Proofs of Lemma~\ref{lem:reform} and Propositions~\ref{app:prop-caps} and~\ref{app:prop-dp}}
\label{app:proof-reform}

We first establish a dynamic-programming principle for the capped model that holds for any restriction of the plans, then prove the caps, and then construct the reductions.

\begin{lemma}[Ranked recursion]\label{app:lem-rankdp}
In the capped model, every transition to a different state strictly decreases the rank~\eqref{eq:app-rank}. Suppose that each unfinished state $s$ has a set $\mathcal D(s)$ of feasible plans containing at least one plan with $P_d(s,s)<1$, and let $V_{\mathcal D}(s)$ be the infimum of the expected total cost from $s$ over policies, possibly history dependent and randomized, that select a plan in $\mathcal D(S_t)$ in every period. Then $V_{\mathcal D}=0$ at terminal states and, at unfinished states,
\begin{equation}\label{eq:app-rankdp}
 V_{\mathcal D}(s)=\min_{d\in\mathcal D(s):\,P_d(s,s)<1}
 \frac{g_d(s)+\sum_{s'\ne s}P_d(s,s')V_{\mathcal D}(s')}{1-P_d(s,s)}.
\end{equation}
A deterministic stationary policy that selects a minimizer at every unfinished state attains $V_{\mathcal D}$.
\end{lemma}

\begin{proof}
Consider a transition from $s$ to $s'\ne s$. No agent completes more than one work unit in a period, so no $x_i$ increases. A failed search raises $k_i$ by one and is admissible only if $k_i<M_i$, so it lowers $(M_i-k_i)^+$ by one. A submission changes $z_i$ from $U$ to $P$ and lowers $\ind{z_i=U}$ by one. A resolution removes task $i$ from $\mathcal A$ and lowers the rank by $(M_i-k_i)^++1+\ind{z_i=U}\ge1$. No component of the rank increases, and at least one changes because $s'\ne s$. Hence the rank strictly decreases. Rank zero consists of terminal states.

We prove the recursion and its attainment by induction on rank. Fix an unfinished state $s$ and suppose the claims hold at every state of lower rank. The minimum in~\eqref{eq:app-rankdp} is over a finite nonempty set; let $r^*$ be its value and $d^*$ a minimizer. Repeating $d^*$ until the state changes and then following stationary optimal policies at the lower-rank successors costs $\sum_{t\ge0}P_{d^*}(s,s)^t[g_{d^*}(s)+\sum_{s'\ne s}P_{d^*}(s,s')V_{\mathcal D}(s')]=r^*$.

Conversely, fix any admissible policy and let $\tau\in\{1,2,\ldots,\infty\}$ be the number of periods it spends in $s$ before the state first changes. On the event $\{\tau>t\}$, let $d_t\in\mathcal D(s)$ be the plan used in period $t$, which may depend on the history and on randomization. Conditional on the history up to period $t$, that period costs $g_{d_t}(s)$ in expectation, and the next state is $s'\ne s$ with probability $P_{d_t}(s,s')$. Because the state is a sufficient statistic, the cost from a state $s'$ onward is at least $V_{\mathcal D}(s')$ whatever the past. Hence the policy's expected cost is at least
\[
 \E\sum_{t\ge0}\ind{\tau>t}\Bigl[g_{d_t}(s)+\sum_{s'\ne s}P_{d_t}(s,s')V_{\mathcal D}(s')\Bigr]
 \ge r^*\,\E\sum_{t\ge0}\ind{\tau>t}\bigl(1-P_{d_t}(s,s)\bigr)=r^*\Pr(\tau<\infty).
\]
The inequality holds termwise: for $P_d(s,s)<1$ it is the definition of $r^*$, and for $P_d(s,s)=1$ its right side is zero. If $\Pr(\tau=\infty)>0$, the cost is infinite, because every period at an unfinished state costs at least $\omega(s)>0$. Otherwise the cost is at least $r^*$. Hence $V_{\mathcal D}(s)=r^*$, attained by the stationary policy above.
\end{proof}

\begin{proof}[Proof of Proposition~\ref{app:prop-caps}]
\emph{Step 1: a uniform bound.} Let $F$ be the following completion rule: submit every unsubmitted question, never search, let every agent with work perform it, and in every available period serve the pending question of the unresolved task with the smallest index. From a state reachable from the initial state, at most $n$ questions remain, every task has at most $b_i$ units of work, and the $n$th available period arrives after $n/a$ periods in expectation. Under $F$, each task completes by the later of these two events, so $F$ costs at most $\sum_iw_i(b_i+n/a)+\sum_i\kappa_i+n\eta=B_{\max}$. Since costs are nonnegative, $0\le V^*\le B_{\max}$ at every reachable state.

\emph{Step 2: truncated problems.} For a horizon $T\ge0$ and $0\le t\le T$, let $V^T_t(s)$ be the optimal cost from state $s$ at period $t$ when the controller must follow $F$ from period $T$ onward, without caps. Then $V^T_T(s)$ is the cost of $F$, and $V^T_t(s)=\min_d\{g_d(s)+\E_d[V^T_{t+1}(S')]\}$ for $t<T$, where the minimum is over the finitely many feasible plans. Each $V^T_t$ lies in $[0,B_{\max}]$, because following $F$ immediately is feasible, and by Lemma~\ref{app:lem-optimism} it is nonincreasing in each task's posterior.

Let $t<T$ and let $d$ be a plan that, in some availability outcome, searches task $i$ at a failure count $k_i\ge M_i$. Let $d'$ be the plan that replaces this search by waiting. If task $i$ is being served in that outcome, the search cannot change the next state and only adds $c_i$. Otherwise, condition on the availability outcome and on the outcomes of all other searches, which are independent of task $i$'s search. Let $s_S$ and $s_F$ be the next states after a success and a failure of task $i$'s search, and $s_N$ the next state without it; $s_N$ differs from $s_F$ only in having one failure fewer. With $q=q_i(k_i)\le q_i(M_i)\le c_i/B_{\max}$, the conditional cost of $d$ minus that of $d'$ is
\[
 c_i+qV^T_{t+1}(s_S)+(1-q)V^T_{t+1}(s_F)-V^T_{t+1}(s_N)
 \ge c_i-q\bigl[V^T_{t+1}(s_F)-V^T_{t+1}(s_S)\bigr]\ge c_i-qB_{\max}\ge0,
\]
using $V^T_{t+1}(s_N)\le V^T_{t+1}(s_F)$ and $0\le V^T_{t+1}\le B_{\max}$. Averaging preserves the inequality. Replacing every such search in turn shows that the minimum defining $V^T_t$ is attained by a plan that respects the caps. Because $F$ never searches, backward induction on $t$ shows that $V^T_0$ equals the value of the truncated problem in which searches at $k_i\ge M_i$ are prohibited.

\emph{Step 3: the limit.} Let $V^{\rm cap}$ be the optimal value of the capped model. Every capped policy is a policy, and every capped truncated policy is a capped policy, so $V^*\le V^{\rm cap}\le V^T_0$ for every $T$. Conversely, let $\pi$ be any policy with finite cost $J^\pi$ and let $\pi_T$ follow $\pi$ for $T$ periods and $F$ thereafter. Then $V^T_0\le J^{\pi_T}\le J^\pi+B_{\max}\Pr^\pi(\tau>T)$, where $\tau$ is the completion time of the portfolio. Because $w_i>0$, $\E^\pi\tau\le J^\pi/\min_iw_i<\infty$, so $\Pr^\pi(\tau>T)\to0$ and $\limsup_TV^T_0\le J^\pi$. Taking the infimum over $\pi$ gives $\lim_TV^T_0=V^*$, hence $V^{\rm cap}=V^*$. By Lemma~\ref{app:lem-rankdp} with all feasible capped plans, $V^{\rm cap}$ is attained by a deterministic stationary policy, which is therefore optimal in the original model.

\emph{Step 4: the remaining claims.} Each $M_i$ is finite because $q_i(k)\to0$ and $c_i>0$. For caps $M_i'\ge M_i$, every $M$-capped policy is $M'$-capped, so $V^*\le V^{M'}\le V^{M}=V^*$. For a subset $I$ of tasks, the bound $B^I_{\max}$ defined by~\eqref{eq:app-bmax} for $I$ alone is at most $B_{\max}$, so the caps that the proposition assigns to $I$ are at most $M_i$. Applying the previous sentence to the portfolio $I$ proves~(iii).
\end{proof}

\begin{proof}[Proof of Lemma~\ref{lem:reform}(i)]
Let $\pi$ be any policy for the capped model. Construct a policy $\pi_1$ as follows. Maintain \emph{virtual} work counts $\tilde x_i$, initially equal to $x_i$, and let the virtual history be the actual history with $x_i$ replaced by $\tilde x_i$. In each period, $\pi_1$ computes the plan $d$ that $\pi$ selects at the virtual history and makes the same submission and service decisions. In each availability outcome, agent $i$ searches if $d$ searches task $i$ and does not serve it; otherwise, the agent works if $x_i>0$ and waits if $x_i=0$. The virtual count $\tilde x_i$ decreases by one whenever $d$ lets agent $i$ work and $\tilde x_i>0$.

The actual and virtual histories differ only in work counts. A search that $\pi$ performs on a task being served affects nothing, because the task is resolved at the end of the period either way; if $\pi$ conditions later decisions on the outcome of such a search, $\pi_1$ feeds it an independently simulated outcome with the correct probability, so $\pi_1$ may be randomized. The other searches and their outcomes are identical. By induction, $x_i\le\tilde x_i$ in every period: whenever the virtual count decreases, the actual count decreases as well unless it is already zero. The virtual history is a function of the actual history, so $\pi_1$ is nonanticipative. Every task is resolved at the same time in both histories and finishes its work no later in the actual one, so it completes no later; search fees do not increase, and submission and answer fees are unchanged. Hence $J^{\pi_1}\le J^\pi$. By construction, $\pi_1$ never searches a task while its question is being served, and every agent with work left works unless it searches, both before and after resolution.
\end{proof}

\begin{proof}[Proof of Lemma~\ref{lem:reform}(ii)]
Let $\pi_1$ be any policy for the capped model. Construct a policy $\pi_2$ that follows $\pi_1$ on a virtual history in which some questions may be virtually submitted but actually unsubmitted. At the start of each period, compute the plan $d$ that $\pi_1$ selects at the virtual history, drawing any randomization now. This determines $\pi_1$'s virtual submissions and the question $j$ that $\pi_1$ would serve if the human is available, if any. If $j$ exists and is actually unsubmitted, $\pi_2$ submits $j$; it submits no other question. If the human is available, $\pi_2$ serves $j$. In both availability outcomes, $\pi_2$ uses $d$'s agent actions, and the virtual history is updated according to $d$.

The construction is feasible. Binary availability is the only information revealed between submission and service, so $j$ is known at the start of the period. Moreover, $j$ is virtually submitted and unresolved, so it is actually pending or newly submitted, and newly submitted questions are immediately eligible. Search, work, and service are the same in both histories, so all completion times, search fees, and answer fees coincide. Each actual submission of a question occurs at most once and only if $\pi_1$ has virtually submitted that question, so submission fees do not increase. Hence $J^{\pi_2}\le J^{\pi_1}$. The policy $\pi_2$ submits at most one question per period and serves it if the human is available. Because it uses $\pi_1$'s agent actions and serves the same questions, it inherits every property of Lemma~\ref{lem:reform}(i) and the caps from $\pi_1$.
\end{proof}

\begin{proof}[Proof of Proposition~\ref{app:prop-dp} and of Lemma~\ref{lem:reform}]
The rank claim is part of Lemma~\ref{app:lem-rankdp}. By Proposition~\ref{app:prop-caps}, the capped model has the optimal value of the original model, and by Lemma~\ref{app:lem-rankdp} applied to all feasible capped plans, it has a deterministic stationary optimal policy $\pi^*$; this proves that~\eqref{eq:app-dp} with all feasible plans gives $V^*$. Apply the construction in the proof of Lemma~\ref{lem:reform}(i) to $\pi^*$ and the construction in the proof of Lemma~\ref{lem:reform}(ii) to the result. The final policy $\pi_2$ has cost at most $V^*$, so it is optimal, and in every period it selects a plan in $\mathcal D_{\rm red}(S_t)$. Let $V_{\rm red}$ be the value of Lemma~\ref{app:lem-rankdp} with $\mathcal D=\mathcal D_{\rm red}$; Appendix~\ref{app:dp} shows that $\mathcal D_{\rm red}(s)$ contains a plan with $P_d(s,s)<1$. Then $V^*\le V_{\rm red}\le J^{\pi_2}=V^*$, so $V_{\rm red}=V^*$. This is~\eqref{eq:app-dp}, because the minimum in~\eqref{eq:app-rankdp} is over $\mathcal D_+(s)$ when $\mathcal D=\mathcal D_{\rm red}$. The stationary policy that attains it selects plans in $\mathcal D_+(s)$ and therefore satisfies the three properties of Lemma~\ref{lem:reform} simultaneously; property~(iii) holds with the caps of Proposition~\ref{app:prop-caps}. Finally, by the rank argument, every state-changing transition advances work, failures, submission, or resolution.
\end{proof}

\subsection{Proofs of Theorem~\ref{thm:thresholds} and Proposition~\ref{app:prop-blocked}}
\label{app:proof-blocked}

We use the notation~\eqref{eq:app-blocked-notation}. A task with $x=0$ has no work, so rule~(d) of Definition~\ref{app:def-reduced} is vacuous. Let $\Pi(Y)=\min\{w+Y,\;w+a\eta+(1-a)Y,\;W\}$. Two facts about $\Pi$ are used repeatedly. First, $\Pi$ is nondecreasing. Second, $\Pi(Y)=W$ whenever $Y\ge W$, because then $w+Y>W$ and $w+a\eta+(1-a)Y\ge w+a\eta+(1-a)W=W$.

\begin{proof}[Proof of Proposition~\ref{app:prop-blocked}]
\emph{Part (i).} At a pending state $(0,k,P)$ with $k<M$, the plans are described by the action in the available outcome (serve, search, or wait) and in the unavailable outcome (search or wait). By~\eqref{eq:app-dp}, their ratios are as follows. Searching in both outcomes gives $w+Y_k$. Serving when available and searching otherwise gives $w+a\eta+(1-a)Y_k$. Serving when available and waiting otherwise gives $(w+a\eta)/a=W$. Searching when available and waiting otherwise gives $(w+aY_k)/a\ge w+Y_k$, and, for $a<1$, waiting when available and searching otherwise gives $w/(1-a)+Y_k\ge w+Y_k$. Waiting in both outcomes leaves the state unchanged. Hence $V^P_k=\Pi(Y_k)$. For $k\ge M$, only the third plan is available and $V^P_k=W$; since $Y_k=c+(1-q_k)W\ge W$ there because $q_kW\le q_kB_{\max}\le c$, also $V^P_k=\Pi(Y_k)$. When $a=1$, the unavailable outcome has probability zero, and the same formulas hold.

For monotonicity, note first that $V^P_{M-1}=\Pi(Y_{M-1})\le W=V^P_M$. Suppose $V^P_{k+1}\le V^P_{k+2}$. Because $q_k\ge q_{k+1}$ and values are nonnegative, $Y_k=c+(1-q_k)V^P_{k+1}\le c+(1-q_{k+1})V^P_{k+2}=Y_{k+1}$, and therefore $V^P_k=\Pi(Y_k)\le \Pi(Y_{k+1})=V^P_{k+1}$. Downward induction from $k=M-1$ proves that both sequences are nondecreasing on $\{0,\ldots,M\}$. Beyond $M$, $V^P_k=W$ is constant and $Y_k=c+(1-q_k)W$ is nondecreasing because $q_k$ is nonincreasing.

\emph{Part (ii).} Comparing the three terms in $\Pi$: $w+Y\le w+a\eta+(1-a)Y$ exactly when $Y\le\eta$; for $a<1$, $w+a\eta+(1-a)Y\le W$ exactly when $Y\le W$, because $W-w-a\eta=(1-a)W$; and $Y<\eta$ implies $w+Y<w+\eta\le W$. These comparisons give the three regimes, with ties at the boundaries. For $a=1$ the second and third plans coincide, and the statement holds trivially.

It remains to show that $Y_k\ge W$ exactly when $q_kW\le c$. If $q_kW>c$, then $Y_k\le c+(1-q_k)W<W$ because $V^P_{k+1}\le W$. If $q_kW\le c$, then $q_jW\le c$ for all $j\ge k$. For any such $j$, $V^P_{j+1}=W$ implies $Y_j=c+(1-q_j)W\ge W$ and hence $V^P_j=\Pi(Y_j)=W$. Downward induction from $V^P_M=W$ gives $V^P_{k+1}=W$ and $Y_k=c+(1-q_k)W\ge W$.

\emph{Part (iii).} At an unsubmitted state $(0,k,U)$ with $k<M$, waiting leaves the state unchanged, submitting gives $\kappa+V^P_k$, and searching without submitting gives $w+c+(1-q_k)V^U_{k+1}$. Hence
\begin{equation}\label{eq:app-unsubmitted}
 V^U_k=\min\bigl\{\kappa+V^P_k,\;w+c+(1-q_k)V^U_{k+1}\bigr\},
 \qquad V^U_M=\kappa+W.
\end{equation}
Searching once and submitting after a failure costs $w+c+(1-q_k)(\kappa+V^P_{k+1})=\kappa+V^P_k+\psi_k$, which explains the definition of $\psi_k$. Because
\[
 w+Y-\Pi(Y)=\max\bigl\{0,\;a(Y-\eta),\;w+Y-W\bigr\}
\]
is nondecreasing in $Y$, $Y_k$ is nondecreasing, and $-\kappa q_k$ is nondecreasing, $\psi_k$ is nondecreasing on $\{0,\ldots,M\}$. At the cap, $\psi_M=w+c-q_M(W+\kappa)\ge w+c-q_MB_{\max}\ge w$, because $W+\kappa\le B_{\max}$. Hence $k_{\rm ask}\le M$.

For $k\ge k_{\rm ask}$, we show $V^U_k=\kappa+V^P_k$ by downward induction from $k=M$. If $V^U_{k+1}=\kappa+V^P_{k+1}$, the search option in~\eqref{eq:app-unsubmitted} equals $\kappa+V^P_k+\psi_k\ge\kappa+V^P_k$, so submitting is optimal at $k$. For $k<k_{\rm ask}$, $\psi_k<0$, so searching once and submitting after a failure is strictly cheaper than submitting now; hence $V^U_k<\kappa+V^P_k$, and searching is the unique optimal action. Thus the policy that searches below $k_{\rm ask}$ and submits at $k_{\rm ask}$ is optimal.

\emph{Part (iv).} By Proposition~\ref{app:prop-caps}(ii), the pending values do not depend on which sufficient cap is used, and no submission fee is paid from a pending state, so $V^P_k$, $Y_k$, and the pending policy do not depend on $\kappa$. If $p>0$ and $0<\alpha<1$, then $q_k>0$ for every $k$, so for any fixed $\bar k$ and all sufficiently large $\kappa$ we have $\psi_k<0$ for $k\le\bar k$, that is, $k_{\rm ask}>\bar k$.

For the converse example, let $w=\eta=1$, $a=0.5$, $c=0.1$, $p=0.2$, and $\alpha=0.5$. Then $W=3$, $q_0=0.1$, $q_1=0.05/0.9<0.1$, and $q_0W=0.3>c$, so by part~(ii) a pending agent with no failures searches when the human is unavailable. Every pending period costs at least $w$, so $V^P_2\ge1$. Hence $w+Y_1\ge1+0.1+0.9\cdot1=2$, $w+a\eta+(1-a)Y_1\ge1.5$, and $W=3$, so $V^P_1\ge1.5$ and $Y_0\ge0.1+0.9\cdot1.5=1.45$. The identity for $w+Y-\Pi(Y)$ then gives $\psi_0\ge a(Y_0-\eta)-\kappa q_0\ge0.225-0.1\kappa$, which is nonnegative for $\kappa\le2.25$. Hence $k_{\rm ask}=0$.
\end{proof}

\begin{proof}[Proof of Theorem~\ref{thm:thresholds}]
Part~(iii) of Proposition~\ref{app:prop-blocked} gives the asking threshold; the zero-fee discussion after Proposition~\ref{app:prop-blocked} gives the consistent choice at ties. For the pending statement, let $a<1$; for $a=1$ the unavailable outcome is unreachable. After an unavailable realization, searching leads to the expected continuation cost $Y_k$ and waiting returns to the same state, with continuation cost $V^P_k$. Hence search is strictly preferable exactly when $Y_k<V^P_k=\Pi(Y_k)$. Because $Y<w+Y$ always and, for $a<1$, $Y<w+a\eta+(1-a)Y$ exactly when $Y<W$, this holds exactly when $Y_k<W$, that is, by part~(ii), when $q_k(\eta+w/a)>c$. Because $q_k$ is nonincreasing, the set of such $k$ is an initial segment, so pending search stops at a failure-count threshold. The absence of a universal ordering is part~(iv).
\end{proof}

\subsection{Proof of Theorem~\ref{thm:overlap}}
\label{app:proof-overlap}

Let $\alpha=1$. The first search succeeds with probability $p$, and after a failure the posterior is zero, so the task can be resolved only by the human. Because $p\eta>c$ in the region~\eqref{eq:overlapregion} and $B_{\max}\ge\eta$, we have $q(0)B_{\max}>c$ and $q(1)=0$, so the cap is $M=1$: the model with one search attempt is the exact $\alpha=1$ case, not a truncation. Write
\[
 W=\eta+\frac wa,\qquad Y_0=c+(1-p)W,\qquad Y_1=c+pw+(1-p)W=Y_0+pw.
\]
Once the search attempt has failed, a pending task with zero or one unit of work has value $W$: it pays $\eta$ once, and serving at the first available period $G$ while working completes it at $\max\{x,G\}=G$, with $\E G=1/a$. An unsubmitted task in the same situation has value $\kappa+W$, because submitting at once costs $\kappa+W$ and working first costs at least $w+\kappa+W$.

\paragraph{Zero remaining work}
Searching before submitting costs
\[
 S_0=w+c+(1-p)(\kappa+W).
\]
Because $p\eta>c$, we have $Y_0<W$, so by Proposition~\ref{app:prop-blocked}(i) the value of submitting immediately is
\[
 A_0=\kappa+\min\bigl\{w+Y_0,\;w+a\eta+(1-a)Y_0\bigr\}.
\]
Direct subtraction gives
\[
 S_0-(\kappa+w+Y_0)=-p\kappa<0,
 \qquad
 S_0-\bigl[\kappa+w+a\eta+(1-a)Y_0\bigr]=w+ac-p(\kappa+w+a\eta).
\]
Hence, if $p>\theta_0$, searching first is strictly better than both ways of submitting immediately. Waiting is dominated, and there is no work. Thus search-first is the unique optimal initial action, and $V^U_0=S_0$.

\paragraph{One remaining work unit}
Consider first the pending state $(1,0,P)$. In each availability outcome, compare the continuation costs of the agent's options. Searching leaves the work unit undone: on success, the work costs $w$, and on failure, the pending value is $W$, so searching has continuation cost $Y_1$. Working leads to $(0,0,P)$, whose value is $V^P_0=\min\{w+Y_0,\;w+a\eta+(1-a)Y_0,\;W\}$. If $Y_0<\eta$, then $V^P_0-Y_1=(1-p)w>0$. If $\eta\le Y_0<W$, then $V^P_0-Y_1=w+a\eta-aY_0-pw=a(p\eta-c)>0$. Hence searching is strictly better than working when the human is unavailable. When the human is available, serving while working completes the task and has continuation cost $\eta$, and
\[
 Y_1-\eta=c+\frac wa-p\left(\eta+\frac{(1-a)w}a\right)>0
 \quad\text{whenever}\quad p<\theta_1,
\]
because $\theta_1\le(w+ac)/(a\eta+(1-a)w)$. Hence serving is strictly best when available. The value of submitting immediately is therefore
\[
 A_1=\kappa+w+a\eta+(1-a)Y_1.
\]
Without an initial submission, waiting is dominated, and the remaining options are search and work. Searching first costs
\[
 S_1=w+c+pw+(1-p)(\kappa+W)=S_0+pw,
\]
and working first costs $w+V^U_0=w+S_0=S_1+(1-p)w>S_1$. Finally,
\[
 S_1-A_1=w+ac-p\bigl[\kappa+(1-a)w+a\eta\bigr]>0
 \quad\text{whenever}\quad p<\theta_1.
\]
This exhausts the initial decisions and proves the policy switch throughout the region~\eqref{eq:overlapregion}.

\paragraph{Nonemptiness}
Because $0<\theta_0<\theta_1$, the region is nonempty exactly when $c/\eta<1$, $\theta_0<1$, and $c/\eta<\theta_1$. The first condition is $c<\eta$, which implies the second because $\kappa>0$. The third is equivalent to $c\kappa+c(1-a)w<\eta w$, that is, $\kappa<w[\eta/c-(1-a)]$. Together with $\kappa>0$, these give~\eqref{eq:nonempty}.

\paragraph{Resource use}
Without work, the agent searches once, and a submission and a human answer occur only after a failure, with probability $1-p$. With one work unit, a question is always submitted; a search occurs only if the human is initially unavailable, with probability $1-a$; and that search withdraws the question with probability $(1-a)p$. The probability of a human answer is therefore $1-(1-a)p$, which exceeds $1-p$ by $ap$, and the expected number of searches falls from $1$ to $1-a$.

\paragraph{The numerical illustration}
For $w=\kappa=\eta=1$, $c=0.1$, and $a=0.5$, $\theta_0=1.05/2.5=0.42$, $\theta_1=1.05/2=0.525$, and $c/\eta=0.1$, so the region is $0.42<p<0.525$. At $p=0.5$, the probability of a human answer is $1-p=0.5$ without work and $1-(1-a)p=0.75$ with one work unit.
\qed

\subsection{Proof of Proposition~\ref{prop:withdraw}}
\label{app:proof-withdraw}

As in Appendix~\ref{app:proof-overlap}, $\alpha=1$ makes the one-attempt model exact in both models. We use $W$, $Y_0$, $Y_1$, $S_0$, $S_1$, and $A_1$ from that proof. When no further search is possible, either because the attempt has failed or because withdrawal is not allowed, a pending task with $x\in\{0,1\}$ has value $W$ and an unsubmitted task has value $\kappa+W$, by the same argument as there.

\paragraph{(i) Without withdrawal}
At the initial state $(1,0,U)$, waiting returns the same state at cost $w>0$ and is dominated. Submitting now forfeits the search and has value $\kappa+W$. Searching without submitting costs $S_1$: on success, one work unit remains and costs $w$, and on failure, the value is $\kappa+W$. Working without submitting leads to $(0,0,U)$, whose value in this model is $\min\{\kappa+W,S_0\}$; this option therefore costs $\min\{w+\kappa+W,\;S_1+(1-p)w\}>\min\{\kappa+W,S_1\}$. Finally,
\[
 S_1-(\kappa+W)=w+c-p\left(\kappa+\eta+\frac{(1-a)w}a\right)
 =\frac{a(w+c)-p\bigl[a\kappa+a\eta+(1-a)w\bigr]}a,
\]
which is negative exactly when $p>\theta_{\rm nw}$. Hence search-first is uniquely optimal, with value $S_1$.

\paragraph{(ii) With withdrawal}
The one-work-unit part of Appendix~\ref{app:proof-overlap} uses only $c/\eta<p<\theta_1$ to show that the best plan that submits immediately is the stated one, with value $A_1$, and that $S_1>A_1$. That argument also compares working first with searching first using $V^U_0=S_0$, which requires $p>\theta_0$. We replace this step with a direct comparison. Without a submission, working first costs $w+V^U_0$, where $V^U_0=\min\{S_0,\kappa+V^P_0\}$. First, $w+S_0=S_1+(1-p)w>A_1$. Second, $w+\kappa+V^P_0-A_1=V^P_0-a\eta-(1-a)Y_1$. Since $p<\theta_1$ gives $Y_1>\eta$, that is, $Y_0-\eta>-pw$, each of the three candidates for $V^P_0$ gives a positive difference:
\begin{align*}
 w+Y_0-a\eta-(1-a)Y_1&=w\bigl[1-(1-a)p\bigr]+a(Y_0-\eta)>(1-p)w>0,\\
 w+a\eta+(1-a)Y_0-a\eta-(1-a)Y_1&=w\bigl[1-(1-a)p\bigr]>0,\\
 W-a\eta-(1-a)Y_1&=w+(1-a)(p\eta-c)+\frac{(1-a)^2pw}a>0,
\end{align*}
where the last line uses $p\eta>c$. Hence working first is strictly worse than $A_1$. Waiting is dominated, so submitting immediately with the stated contingent actions is uniquely optimal, and the value with withdrawal is $A_1$.

\paragraph{(iii) Cost and resource use}
The cost reduction is $S_1-A_1=w+ac-p[\kappa+(1-a)w+a\eta]$, which is~\eqref{eq:withdrawsaving} and positive because $p<\theta_1$. Without withdrawal, the agent searches once. With probability $p$ it completes in period~2, and otherwise it submits in period~2 and completes after a further geometric wait for the human. Its expected completion time is $2p+(1-p)(1+1/a)$, and a submission and an answer each occur with probability $1-p$. With withdrawal, a submission always occurs. With probability $a$, the task completes in period~1 after one answer. Otherwise, the agent searches once; on success, the question is withdrawn and the task completes in period~2, and on failure, it completes after a further geometric wait and one answer. Its expected completion time is $a+(1-a)[2p+(1-p)(1+1/a)]$, the answer probability is $a+(1-a)(1-p)=1-(1-a)p$, and the expected number of searches is $1-a$. Subtraction gives the stated changes: the expected completion time falls by $1-(1-a)p$, the answer probability rises by $ap$, the submission probability rises by $p$, and expected searches fall by $a$. As a check, $w[1-(1-a)p]+ac-\kappa p-\eta ap$ equals the cost reduction. For the illustration in the main text, $w=\kappa=\eta=1$, $c=0.1$, $a=0.5$, and $p=0.5$, so $W=3$, $Y_1=2.1$, $S_1=3.60$, and $A_1=3.55$; the four terms of the check are $0.75$, $0.05$, $-0.5$, and $-0.25$. Moreover, $\theta_{\rm nw}=0.55/1.5\approx0.367$, so the region~\eqref{eq:withdrawregion} is $0.367<p<0.525$.

\paragraph{Nonemptiness}
Because $\theta_{\rm nw}$ and $\theta_1$ are positive, the region~\eqref{eq:withdrawregion} is nonempty exactly when $c/\eta<1$, $c/\eta<\theta_1$, $\theta_{\rm nw}<1$, and $\theta_{\rm nw}<\theta_1$. The first is $c<\eta$, and the second is $\kappa<w[\eta/c-(1-a)]$, as in Appendix~\ref{app:proof-overlap}. The third rearranges to $(2a-1)w<a(\kappa+\eta-c)$. For the fourth, write $N_1=w+ac$ and $D_1=\kappa+(1-a)w+a\eta$, so that $\theta_1=N_1/D_1$ and $\theta_{\rm nw}=[N_1-(1-a)w]/[D_1-(1-a)\kappa]$. Cross-multiplying positive denominators shows that, for $a<1$, $\theta_{\rm nw}<\theta_1$ is equivalent to $wD_1>\kappa N_1$, that is, $(1-a)w^2+a(\eta w-\kappa c)>0$. The second condition gives $\kappa c<w[\eta-(1-a)c]$, so the left side exceeds $(1-a)w^2+a(1-a)wc>0$. Thus the fourth condition is implied by the second, which yields~\eqref{eq:withdrawnonempty}.
\qed

\subsection{Proof of Proposition~\ref{prop:noindex} and verification of Example~\ref{ex:priority}}
\label{app:proof-noindex}

\begin{proof}[Proof of Proposition~\ref{prop:noindex}]
Suppose such a score $r$ exists. In Example~\ref{ex:priority}, tasks $A$ and $B$ have the same parameters and states in both portfolios, and $a$ and $\eta$ are common. In $\{A,B\}$, serving $A$ first is uniquely optimal, so $B$ cannot have the largest score, and $r(A)>r(B)$. In $\{A,B,C\}$, serving $B$ first is uniquely optimal, so $r(B)>r(A)$. This is a contradiction. The strict optimality claims are verified below.
\end{proof}

\paragraph{Setup}
All tasks are initially unsubmitted and have no work, $\alpha_i=1$, $a=1$, $c_i=0.01$, $\kappa_i=0.05$, and $\eta=0.05$; tasks $A$, $B$, $C$ have $(w_i,p_i)=(4,0.75)$, $(1,0.10)$, and $(0.5,0)$. Every period is available. Each task has at most one useful search, after whose failure it can be resolved only by the human, and searching $C$ is useless because $p_C=0$. Every task's eventual answer costs $f_i=\kappa_i+\eta=0.1$. By Lemma~\ref{lem:reform}, we may restrict attention to plans that submit at most the question that is served, never search a served task, and change the state with positive probability. This restriction preserves the optimal cost conditional on each choice of first service, because the constructions in Appendix~\ref{app:proof-reform} serve the same questions in the same periods as the policy they modify. Let $V^{\rm h}_X$ denote the optimal cost of a set $X$ of tasks that all require the human. By Proposition~\ref{prop:pair_exact}, these tasks are served in descending weight order, one per period, so
\[
 V^{\rm h}_A=4.1,\quad V^{\rm h}_B=1.1,\quad V^{\rm h}_C=0.6,\quad
 V^{\rm h}_{AB}=6.2,\quad V^{\rm h}_{AC}=5.2,\quad V^{\rm h}_{BC}=2.2,\quad
 V^{\rm h}_{ABC}=7.8.
\]
For example, $V^{\rm h}_{ABC}=4\cdot1+1\cdot2+0.5\cdot3+f_A+f_B+f_C=7.8$.

\paragraph{Intermediate values}
We need the values of states in which some tasks have searched without success and others have not yet searched; write $X^{\rm u}$ for an unsearched task and $X^{\rm f}$ for a task whose search has failed. A single unsearched task has values $V(A^{\rm u})=\min\{4.1,\;4+0.01+0.25V^{\rm h}_A\}=4.1$ and $V(B^{\rm u})=\min\{1.1,\;1+0.01+0.9V^{\rm h}_B\}=1.1$. Enumerating the first-period decisions in the same way gives
\begin{gather*}
 V(A^{\rm u},B^{\rm u})=6.100,\quad
 V(A^{\rm f},B^{\rm u})=6.100,\quad V(A^{\rm u},B^{\rm f})=6.135,\\
 V(A^{\rm u},C)=5.200,\quad V(B^{\rm u},C)=2.200,
\end{gather*}
\[
 V(A^{\rm f},B^{\rm u},C)=7.650,\qquad V(A^{\rm u},B^{\rm f},C)=7.360.
\]
Once a task is served, its search history no longer matters, so serving $A$ while $B$ searches costs $7.650$ at $(A^{\rm f},B^{\rm u},C)$, as in the corresponding entry of the table below. The alternatives cost $7.8$ (serving $A$ without search), $10.8$ (serving $B$), $11.6$ (serving $C$ while $B$ searches), and $13.05$ (declining service while $B$ searches). Similarly, $V(A^{\rm u},B^{\rm f},C)$ is attained by serving $B$ while $A$ searches, at $7.360$, against $7.8$, $7.985$, and $9.11$ for serving $A$, serving $C$ while $A$ searches, and declining service while $A$ searches. The remaining values follow in the same way.

\paragraph{First-period decisions}
Each entry of the following table is the first period's delay and fee cost plus the optimal continuation, for a given first service and set of concurrent searches. A dash denotes a search of the served task, which is dominated.
\begin{center}
\begin{tabular}{llrrrr}
\toprule
Portfolio & Service & No search & Search $A$ & Search $B$ & Search $A,B$\\
\midrule
$A,B$ & $A$ & 6.200 & --- & 6.100 & ---\\
$A,B$ & $B$ & 9.200 & 6.135 & --- & ---\\
$A,B$ & None & dominated & 7.360 & 10.9415 & 7.260\\
\midrule
$A,B,C$ & $A$ & 7.800 & --- & 7.650 & ---\\
$A,B,C$ & $B$ & 10.800 & 7.360 & --- & ---\\
$A,B,C$ & $C$ & 11.700 & 7.960 & 11.5415 & 7.860\\
$A,B,C$ & None & dominated & 9.0725 & 12.654 & 8.935\\
\bottomrule
\end{tabular}
\end{center}
For example, in $\{A,B\}$, serving $A$ while $B$ searches costs $5+0.1+0.01+0.9V^{\rm h}_B=6.100$, and serving $B$ while $A$ searches costs $5+0.1+0.01+0.25V^{\rm h}_A=6.135$. In $\{A,B,C\}$, serving $A$ while $B$ searches costs $5.5+0.1+0.01+0.1V^{\rm h}_C+0.9V^{\rm h}_{BC}=7.650$; serving $B$ while $A$ searches costs $5.5+0.1+0.01+0.75V^{\rm h}_C+0.25V^{\rm h}_{AC}=7.360$; serving $C$ while $A$ and $B$ search costs $5.5+0.1+0.02+0.225V^{\rm h}_{AB}+0.025V^{\rm h}_A+0.675V^{\rm h}_B=7.860$; and declining service while both search costs $5.5+0.02+0.225V^{\rm h}_{ABC}+0.025V^{\rm h}_{AC}+0.675V^{\rm h}_{BC}+0.075V^{\rm h}_C=8.935$. Declining service while only $A$ searches costs $5.51+0.75V(B^{\rm u},C)+0.25V(A^{\rm f},B^{\rm u},C)=9.0725$. Declining service without any search leaves the state unchanged and is dominated. Taking the minimum of each row gives the table of Example~\ref{ex:priority}. In $\{A,B\}$, serving $A$ is uniquely optimal; in $\{A,B,C\}$, serving $B$ is uniquely optimal.

\subsection{Projection, nonnegative prices, and the proof of Proposition~\ref{prop:pair_exact}}
\label{app:proof-pair}

\begin{lemma}[Projection onto a task block]\label{app:lem-projection}
For any nonempty set $I$ of tasks, let $V_I$ be the optimal value of the tasks in $I$ when they share one human with the original availability process and the caps of Proposition~\ref{app:prop-caps}. For any feasible policy $\pi$ of the full portfolio, let $J_i^\pi(s)$ be the expected cost attributable to task $i$, namely its delay cost and its own search, submission, and answer fees. Then $V_I(s_I)\le\sum_{i\in I}J_i^\pi(s)$. In particular, $\Delta_{ij}=V_{ij}-V_i-V_j\ge0$.
\end{lemma}

\begin{proof}
Run $\pi$ in the system that contains only the tasks in $I$, simulating the excluded tasks with independently drawn search outcomes and using the same availability process. Whenever $\pi$ serves an excluded task, leave the human idle. The tasks in $I$ receive the same actions as in the full system. Because the excluded tasks' outcomes are independent of everything else and the state is a sufficient statistic, the simulated system has the same joint law as the full system, so the tasks in $I$ incur the same expected costs. Conditional on the simulated outcomes, this is a feasible policy for the block. It may search a task of $I$ beyond its cap if $\pi$ does, so it is feasible for the block without caps. Its cost is therefore at least the block's optimal value without caps, which equals $V_I(s_I)$ by Proposition~\ref{app:prop-caps}(iii). Applying the inequality to an optimal policy for the pair $\{i,j\}$ and the two singleton blocks gives $V_i+V_j\le V_{ij}$.
\end{proof}

\begin{proof}[Proof of Proposition~\ref{prop:pair_exact}]
A resolved task's remaining cost is its deterministic work cost, which uses neither another agent nor the human, so its value is additive with any other block. Hence the pairwise approximation equals~\eqref{eq:app-critic-unresolved}. With zero or one unresolved task, \eqref{eq:app-critic-unresolved} is the exact additive value. With two unresolved tasks $i$ and $j$, the singleton terms cancel inside the single pairwise price, leaving $C_R+V_{ij}=V^*$.

Now suppose that no remaining search can succeed and that every unresolved task has $x_i=0$. Each unresolved task must be served once and pays $f_i=\eta+\kappa_i\ind{z_i=U}$. Under any policy, the $\ell$th service among unresolved tasks occurs no earlier than the $\ell$th available period $T_\ell$, and a task completes when it is served. By the rearrangement inequality, pathwise $\sum_{i\in\mathcal A}w_iC_i\ge\sum_\ell w_{(\ell)}T_\ell$, where $w_{(1)}\ge w_{(2)}\ge\cdots$ are the sorted weights. Serving in descending weight order in every available period attains this bound. Since $\E T_\ell=\ell/a$ and $\sum_\ell\ell w_{(\ell)}=\sum_iw_i+\sum_{i<j}\min(w_i,w_j)$,
\begin{equation}\label{eq:app-allhuman}
 V^*=C_R+\sum_{i\in\mathcal A}f_i+\frac1a\Bigl(\sum_{i\in\mathcal A}w_i+\sum_{i<j}\min(w_i,w_j)\Bigr).
\end{equation}
Applied to one and two tasks, \eqref{eq:app-allhuman} gives $V_i=f_i+w_i/a$ and $\Delta_{ij}=\min(w_i,w_j)/a$, which is~\eqref{eq:humanboundary}. Substituting into~\eqref{eq:app-critic-unresolved} shows that the pairwise approximation equals~\eqref{eq:app-allhuman}.

Both boundaries are invariant: the number of unresolved tasks cannot increase, and in the second boundary resolved tasks leave $\mathcal A$ and the remaining ones keep $x_i=0$ and no useful search. Hence $\widehat V_{\rm PCP}=V^*$ at every state of the boundary and at all its successors. An exact minimizer of~\eqref{eq:greedy} then minimizes $g_d(s)+P_dV^*(s)$ over $\mathcal D_+(s)$, whose minimum is $V^*(s)$ by Proposition~\ref{app:prop-dp}. A policy that attains this minimum at every state and satisfies the progress guard has cost $V^*$, by telescoping $V^*$ along its trajectory and Lemma~\ref{app:lem-progress}. This proves that PCP is optimal on both boundaries.
\end{proof}

\subsection{Proofs for Appendix~\ref{app:implementation}}
\label{app:proof-implementation}

\begin{proof}[Proof of Lemma~\ref{app:lem-progress}]
Fix an availability outcome and a deterministic plan $d\in\mathcal D_{\rm red}(s)$. If $d$ submits a question, serves a question, lets an agent search, or lets an agent work, the state changes with probability one: a search either resolves its requirement or raises its failure count, which is below the cap. Otherwise, the state does not change. A submission is made in both availability outcomes. Hence the probability that $d$ changes the state is one of $0$, $a$, $1-a$, and $1$, and it is at least $\varepsilon_a$ if $d\in\mathcal D_+(s)$; for $a=1$ only the available outcome has positive probability. A randomized choice among such plans has the same lower bound. Under the policy, the conditional probability that the state changes in a period, given the past, is therefore at least $\varepsilon_a$. By Lemma~\ref{app:lem-rankdp}, each change lowers the rank by at least one, so at most $\mathrm{rk}(s)$ changes precede completion, and each takes at most $1/\varepsilon_a$ periods in expectation. Hence the expected completion time is at most $\mathrm{rk}(s)/\varepsilon_a$. Per-period costs are bounded on the finite capped model, so the expected total cost is finite.
\end{proof}

\subsection{Proofs for the human-resolution regime}
\label{app:proof-human}

Throughout this subsection, the state is in the human-resolution regime, $\bar a=1-a$, and $G$ is geometric on $\{1,2,\ldots\}$ with success probability $a$. For a task with work $x$, let $h(t)=w\max\{x,t\}$ be its delay cost if its question is answered in period $t$. Because $\Pr(G>t)=\bar a^t$,
\begin{equation}\label{eq:app-tail}
 \E\bigl[h(t+G)-h(t)\bigr]=w\,\E\bigl[(t+G-x)^+-(t-x)^+\bigr]=\frac wa\,\bar a^{(x-t)^+}
\end{equation}
for every $t\ge0$. In particular, $\E\max\{x,G\}=x+\bar a^x/a$.

\begin{proof}[Proof of Lemma~\ref{app:lem-humanpair}]
Every unresolved task must be answered exactly once and pays $f_i$ exactly once. Under any policy, task $i$ completes no earlier than $\max\{x_i,T_i\}$, where $T_i$ is the period in which it is answered, and every agent working in every period attains this.

\emph{Singleton.} The answer period satisfies $T_i\ge G$, where $G$ is the first available period, and answering at $G$ is feasible. Since $h_i$ is nondecreasing, $V_i=f_i+\E h_i(G)=f_i+w_i(x_i+\bar a^{x_i}/a)$ by~\eqref{eq:app-tail}.

\emph{Pair.} Under any policy for the pair, let $T_{(1)}<T_{(2)}$ be the two answer periods, let $\iota$ be the task answered second, and let $\mathcal F$ be the information available at the end of period $T_{(1)}$, which determines $\iota$. The delay cost equals
\[
 h_i(T_{(1)})+h_j(T_{(1)})+\bigl[h_\iota(T_{(2)})-h_\iota(T_{(1)})\bigr].
\]
Because availability is independent over time, $T_{(2)}-T_{(1)}$ is at least a geometric random variable independent of $\mathcal F$. By~\eqref{eq:app-tail}, the conditional expectation of the bracket given $\mathcal F$ is at least $\zeta_\iota(T_{(1)})\ge\zeta(T_{(1)})$, where $\zeta_k(t)=(w_k/a)\bar a^{(x_k-t)^+}$ and $\zeta=\min\{\zeta_i,\zeta_j\}$. Both $h_k$ and $\zeta$ are nondecreasing and $T_{(1)}\ge G$, so the expected delay cost is at least $\E[h_i(G)+h_j(G)+\zeta(G)]$. Answering at $G$ the task with the larger $\zeta_k(G)$, and the other task at the next available period, attains this bound; it is feasible because the controller can submit, at the start of every period, the question it would answer if the human were available. Subtracting the singleton values gives~\eqref{eq:app-human-pair}.

The monotonicity claims hold pointwise inside the expectation, because $\bar a\le1$. Since $\bar a^{(x-G)^+}\le1$, $\Delta_{ij}\le\min(w_i,w_j)/a$, with equality when $x_i,x_j\le1$ because $G\ge1$. When $a=1$, $G=1$ and $\bar a^{(x-1)^+}=\ind{x\le1}$.
\end{proof}

\begin{proof}[Proof of Lemma~\ref{app:lem-priority}]
By Lemma~\ref{app:lem-humanpair}, $\widehat V_{\rm PCP}(s)=\sum_{i\in\mathcal A}f_i+\Phi(s)$, where
\[
 \Phi(s)=C_R(s)+\sum_{i\in\mathcal A}w_i\Bigl(x_i+\frac{\bar a^{x_i}}a\Bigr)+\sum_{i<j}\Delta_{ij}(x_i,x_j)
\]
depends only on the remaining work and on which tasks are unresolved. In this regime, no plan searches and every agent with work works. A submission of question $j$ pays $\kappa_j$ now and removes $\kappa_j$ from $f_j$ at every successor, and an answer pays $\eta$ and removes the remaining $f_j=\eta$. Hence, for every $d\in\mathcal D_+(s)$,
\[
 g_d(s)+P_d\widehat V_{\rm PCP}(s)=\omega(s)+\sum_{i\in\mathcal A}f_i+(1-a)\,\Phi(s'_0)+a\,\Phi(s'_d),
\]
where $s'_0$ is the successor in which every agent works and no question is answered, which is the same for all plans, and $s'_d$ is the successor in the available outcome. Every unresolved task can be served in the available outcome: a pending question can be served directly, and an unsubmitted one can be submitted and served under rule~(a). Thus exact PCP chooses the service in the available outcome to maximize $\Gamma_j=\Phi(s'_0)-\Phi(s'_{(j)})$ over unresolved tasks $j$, where $s'_{(j)}$ is the successor in which $j$ is served, with serving no one counted as $\Gamma=0$ and ties resolved by the tie rule.

Let $y_k=(x_k-1)^+$ be the work after the period. Serving $j$ replaces $j$'s term $w_j(y_j+\bar a^{y_j}/a)$ by the resolved cost $w_jy_j$ and removes its pairwise prices, so
\[
 \Gamma_j=\frac{w_j\bar a^{y_j}}a+\sum_{k\in\mathcal A\setminus\{j\}}\Delta_{jk}(y_j,y_k).
\]
If $x_i\le1$, then $y_i=0$ and $\Gamma_i\ge w_i/a>0$, so PCP serves some task. Suppose $w_j<w_i$. The price $\Delta_{ij}$ appears in both $\Gamma_i$ and $\Gamma_j$ and cancels in their difference. For every other task $k$, $\Delta_{ik}(0,y_k)\ge\Delta_{jk}(y_j,y_k)$, because $w_i\ge w_j\bar a^{(y_j-G)^+}$ pointwise in~\eqref{eq:app-human-pair}. Moreover, $w_i/a>w_j\bar a^{y_j}/a$. Hence $\Gamma_i>\Gamma_j$, and PCP does not serve $j$.

In terms of the objective in~\eqref{eq:greedy}, serving $i$ instead of $j$ lowers it by $a(\Gamma_i-\Gamma_j)\ge w_i-w_j$, and serving $i$ instead of no one lowers it by $a\Gamma_i\ge w_i$. A selector that evaluates every service choice and returns a plan within an absolute tolerance $\tau(s)$ of the minimum therefore has the same two properties whenever $\tau(s)$ is smaller than these margins, that is, smaller than $w_i$ and than every positive difference $w_i-w_j$, as stated in Appendix~\ref{app:implementation}.
\end{proof}

\begin{proof}[Proof of Theorem~\ref{thm:pcp_guarantee}(i)]
Let $f_i$ and $w_{\mathcal A}=\sum_{i\in\mathcal A}w_i$ be as above, so that $K=C_R+\sum_{i\in\mathcal A}f_i$.

\emph{Lower bounds on the optimum.} Every unresolved task completes no earlier than its remaining work and no earlier than period $1$, so $V^*\ge K+B$. Removing the remaining work of unresolved tasks can only lower the optimum, because every policy remains feasible without the work and completes every task no later; the resulting value is $K+Q$ by~\eqref{eq:app-allhuman}. Hence
\begin{equation}\label{eq:app-optimum-floor}
 V^*\ge K+\max\{B,Q\}.
\end{equation}

\emph{PCP's delay cost.} Let $T_i$ be the period in which PCP answers task $i$. Every agent works until its work is finished, so task $i$ completes at $C_i=\max\{x_i,T_i\}$. Let $\ell_i=\max\{x_i,1\}$ and let $\mathcal E_i=\{T_i<\ell_i\}$ be the event that the answer arrives before the work is finished. Pathwise,
\[
 C_i=(\ell_i-1)+\ind{\mathcal E_i}+\sum_{t\ge\ell_i}\ind{T_i\ge t}.
\]
Indeed, on $\mathcal E_i$ the task completes at $\ell_i=x_i$ and the sum is empty, and otherwise the sum has $T_i-\ell_i+1$ terms. Let $N_i=\sum_{t\ge\ell_i}\ind{T_i\ge t,\ \text{the human is available in period }t}$. The event $\{T_i\ge t\}$ is determined before period $t$'s availability is drawn, so by independence and Tonelli's theorem,
\begin{equation}\label{eq:app-predictable}
 \E\sum_{t\ge\ell_i}\ind{T_i\ge t}=\frac1a\,\E N_i.
\end{equation}
This does not require the service order to be independent of availability. In every period counted by $N_i$, task $i$ is unresolved with at most one unit of work left, and the human is available. By Lemma~\ref{app:lem-priority}, PCP serves some task $j$ with $w_j\ge w_i$ in that period. If $j=i$, the period is $T_i$, which can be counted only on $\mathcal E_i^c$. If $j\ne i$, charge the period to the pair $\{i,j\}$, at cost $w_i=\min(w_i,w_j)$. Each pair is charged at most once, because only the member answered first can be served while the other waits, and each task is answered once. Hence, pathwise,
\[
 \sum_{i\in\mathcal A}w_iN_i\le\sum_{i\in\mathcal A}w_i\ind{\mathcal E_i^c}+\sum_{i<j}\min(w_i,w_j).
\]
Combining this with~\eqref{eq:app-predictable}, the identity for $C_i$, and $a\le1$,
\[
 \sum_{i\in\mathcal A}w_i\E C_i\le B-w_{\mathcal A}+\sum_iw_i\Pr(\mathcal E_i)+\frac1a\sum_iw_i\Pr(\mathcal E_i^c)+\frac1a\sum_{i<j}\min(w_i,w_j)\le B-w_{\mathcal A}+Q.
\]
PCP pays each unresolved task's fees $f_i$ exactly once and the resolved work cost $C_R$, so
\begin{equation}\label{eq:app-pcp-primitive}
 J^{\rm PCP}\le K+B+Q-w_{\mathcal A}.
\end{equation}
The progress guard ensures that all expectations are finite (Lemma~\ref{app:lem-progress}).

\emph{Conclusion.} By~\eqref{eq:app-optimum-floor}, $B+Q\le2(V^*-K)$, so~\eqref{eq:app-pcp-primitive} gives $J^{\rm PCP}\le2V^*-K-w_{\mathcal A}\le2V^*$, which is~\eqref{eq:pcp_twofactor}. The argument uses exact minimization only in the service choice, through Lemma~\ref{app:lem-priority}, and holds for every tie rule. It therefore applies to any implementation that enumerates the admissible submissions and optimizes the service choice exactly in this regime. It also applies to the coordinate implementation of Appendix~\ref{app:implementation} whenever its absolute tolerance satisfies the condition stated there.

\emph{Lower bound $\varphi$ under the no-submission tie rule.} This is Lemma~\ref{app:lem-phi}, proved in Appendix~\ref{app:proof-phi}.
\end{proof}

\begin{proof}[Proof of Corollary~\ref{app:cor-interp}]
By~\eqref{eq:app-pcp-primitive} and~\eqref{eq:app-optimum-floor}, $J^{\rm PCP}-V^*\le B+Q-w_{\mathcal A}-\max\{B,Q\}=\min\{B,Q\}-w_{\mathcal A}$, and dividing by $V^*\ge K+\max\{B,Q\}$ gives the first inequality in~\eqref{eq:app-interp}. The second holds because $\min\{B,Q\}\le\max\{B,Q\}$. Note that $\min\{B,Q\}\ge w_{\mathcal A}$, because $B\ge w_{\mathcal A}$ and $Q\ge w_{\mathcal A}/a$. If $x_i\le1$ for all unresolved tasks, then $B=w_{\mathcal A}$, so $J^{\rm PCP}\le K+Q\le V^*$ and hence $J^{\rm PCP}=V^*=K+Q$. Moreover, $x_i+\bar a^{x_i}/a=1/a$ for $x_i\in\{0,1\}$ and $\Delta_{ij}=\min(w_i,w_j)/a$ by Lemma~\ref{app:lem-humanpair}, so $\widehat V_{\rm PCP}(s)=K+Q$ by~\eqref{eq:app-critic-unresolved}.
\end{proof}

\subsection{Proof of Lemma~\ref{app:lem-phi}}
\label{app:proof-phi}

Because $p_i=0$, the caps are zero and the portfolio is in the human-resolution regime from the start. With $a=1$, the human can answer one question in every period. The $j$th question answered is answered no earlier than period $j$, and its task completes no earlier than $\max\{b,j\}$, so $V^*\ge\sum_{j=1}^n\max\{b,j\}$; answering one question per period from period~1 attains this bound.

If $b\le1$, every unresolved task has at most one unit of work, and Corollary~\ref{app:cor-interp} shows that PCP is optimal, with cost $\sum_{j=1}^nj=n(n+1)/2$.

Let $b\ge2$. All agents work in every period, so all unresolved tasks always have the same remaining work $x$. With $a=1$ and zero fees, the quantity $\Gamma_j$ in the proof of Lemma~\ref{app:lem-priority} is $\ind{y=0}+(m-1)\ind{y\le1}$, where $y=(x-1)^+$ and $m$ is the number of unresolved tasks. If $x\ge3$, then $y\ge2$ and $\Gamma_j=0$ for every $j$, so every plan ties, and the no-submission rule neither submits nor serves. If $x=2$, then $y=1$ and $\Gamma_j=m-1>0$, so PCP serves a question. Afterwards every unresolved task has $x\le1$, and by Lemma~\ref{app:lem-priority} PCP serves one question per period until none remain. The first question is therefore answered in period $b-1$, when $x=2$, and the completion times are $b,b,b+1,\ldots,b+n-2$. Their sum is $nb+(n-1)(n-2)/2$.

For the ratio, $2V^*=n^2+b^2+n-b$ when $b\le n$, and $V^*=nb$ when $b\ge n$. If $2\le b\le n$,
\[
 \frac{J^{\rm PCP}}{V^*}=\frac{2nb+n^2-3n+2}{n^2+b^2+n-b}\le\frac{n^2+2nb}{n^2+b^2}\le\varphi.
\]
The first inequality holds because the cross-multiplied difference, $(n^2+b^2)(2-3n)-(n^2+2nb)(n-b)$, is negative. The second follows from $\varphi(n^2+b^2)-(n^2+2nb)=\varphi(b-n/\varphi)^2\ge0$, which uses $\varphi-1=1/\varphi$. If $b\ge n$, the ratio is $1+(n-1)(n-2)/(2nb)<3/2<\varphi$. Along $b=\lfloor n/\varphi\rfloor$, the ratio converges to $(1+2/\varphi)/(1+1/\varphi^2)=\sqrt5/(3-\varphi)=\varphi$.

Finally, consider a tie rule that prefers serving a question among tied plans. Then PCP serves a question in every period: when $x\ge3$ all plans tie, and when $x\le2$ serving is strictly better or tied, with every remaining task alike. Hence it answers one question per period from period~1, which is optimal.
\qed

\subsection{Proofs for the Bellman defect}
\label{app:proof-defect}

\begin{proof}[Proof of Lemma~\ref{app:lem-defect}]
Project plan $d$ onto task $i$ by keeping task $i$'s submission, service, and agent actions and replacing any service of another task by no service. This is a feasible plan for the singleton problem, with immediate cost $g_i^d$ and next-state law equal to the marginal of $S'_i$ under $d$; the projection uses only the current state. The singleton value satisfies the Bellman inequality $V_i\le g+P V_i$ for every feasible singleton plan, including plans that may leave the state unchanged, because using the plan for one period and continuing optimally thereafter is feasible. Hence $\lambda_i^d\ge0$, and the same argument for pairs gives $\lambda_{ij}^d\ge0$.

For $n\ge2$, the definitions give $\widehat V_{\rm PCP}=\sum_{i<j}V_{ij}-(n-2)\sum_iV_i$, with sums over all tasks, because each $V_i$ appears in $n-1$ pair values. Since $\sum_ig_i^d=g_d$ and each $g_i^d$ appears $n-1$ times in the pair sum and $n-2$ times in the subtracted singleton sum, \eqref{eq:app-defect-full} follows; for $n=1$ it reduces to $g_d+P_dV_1-V_1=\lambda_1^d$.

If task $i$ is resolved, rule~(d) makes it work in every period until it is finished, so $\lambda_i^d=0$. Its pair value with any task $j$ is additive at the current state and at every successor, so $\lambda_{ij}^d=\lambda_j^d$, which is zero if $j$ is also resolved. The $n-m$ resolved tasks therefore contribute $(n-m)\sum_{j\in\mathcal A}\lambda_j^d$ to the pair sum in~\eqref{eq:app-defect-full}, which leaves the coefficient $-(m-2)$ on $\sum_{j\in\mathcal A}\lambda_j^d$. Minimizing over $\mathcal D_+(s)$ gives~\eqref{eq:app-defect}.
\end{proof}

\begin{proof}[Proof of Proposition~\ref{app:prop-defect}]
Write $H=\widehat V_{\rm PCP}$ and recall that $g_d(s)\ge\omega(s)>0$ at unfinished states and $H\ge0$, with $H=0$ at terminal states.

\emph{Part (i).} By the definitions of $\rho$ and $e_\pi$, at every unfinished state $s$ reachable under $\pi$,
\[
 g_{\pi(s)}(s)+P_{\pi(s)}H(s)-H(s)=\rho(s)+e_\pi(s)\le(\epsilon_++\delta)\,\omega(s)\le(\epsilon_++\delta)\,g_{\pi(s)}(s),
\]
so $(1-\epsilon_+-\delta)\,g_{\pi(s)}(s)\le H(s)-P_{\pi(s)}H(s)$. The policy $\pi$ satisfies the progress guard, so it completes the portfolio with probability one (Lemma~\ref{app:lem-progress}). Summing along its trajectory up to period $T$ gives $(1-\epsilon_+-\delta)\,\E\sum_{t<T}g(S_t)\le H(s_0)-\E H(S_T)\le H(s_0)$, and monotone convergence as $T\to\infty$ yields $J^\pi(s_0)\le H(s_0)/(1-\epsilon_+-\delta)$.

\emph{Part (ii).} Let $\pi^*$ be the optimal deterministic stationary policy of Proposition~\ref{app:prop-dp}, which selects plans in $\mathcal D_+(s)$; starting from $s_0$, it stays in $\mathcal S$ because $\mathcal S$ is forward closed. For every unfinished $s\in\mathcal S$ and $d=\pi^*(s)$,
\[
 g_d(s)+P_dH(s)-H(s)\ge\rho(s)\ge-\epsilon_-\,\omega(s)\ge-\epsilon_-\,g_d(s),
\]
so $H(s)\le(1+\epsilon_-)\,g_d(s)+P_dH(s)$. Summing along the trajectory of $\pi^*$ up to period $T$ gives $H(s_0)\le(1+\epsilon_-)\,\E\sum_{t<T}g(S_t)+\E H(S_T)$. Because $H$ is bounded on the finite state space and $S_T$ is terminal for all large $T$ almost surely, $\E H(S_T)\to0$ by dominated convergence, so $H(s_0)\le(1+\epsilon_-)V^*(s_0)$.

\emph{Combination.} The states reachable under $\pi$ belong to $\mathcal S$, because $\pi$ selects plans in $\mathcal D_+(s)$ and $\mathcal S$ is forward closed. Hence the combined conditions imply those of parts~(i) and~(ii), and~\eqref{eq:app-defect-ratio} follows. Neither argument uses a discount factor or removes self-transitions from $\mathcal T$.
\end{proof}

\begin{proof}[Proof of Theorem~\ref{thm:pcp_guarantee}(ii)]
Exact PCP has $e_\pi\equiv0$. Apply Proposition~\ref{app:prop-defect} with $\delta=0$ on the set $\mathcal S$ of Theorem~\ref{thm:pcp_guarantee}(ii).
\end{proof}

\section{Details of the Numerical Experiments}
\label{app:numerics}

This appendix documents the policies, instances, and evaluation methods of the synthetic experiments in Section~\ref{subsec:num-sim} and of the LLM case study in Section~\ref{subsec:num-llm}. It also describes the computing environment, shows that the mechanisms of Section~\ref{sec:structure} extend beyond the settings of the theorems, and reports additional results. All experiments use the exact dynamic program of Proposition~\ref{app:prop-dp} with the caps of Proposition~\ref{app:prop-caps}.

\subsection{Policies and implementation}
\label{app:num-policies}

\paragraph{PCP and the singleton policy}
Both policies minimize~\eqref{eq:greedy} over $\mathcal D_+(s)$, with the pairwise and the singleton approximation, respectively, and use exact singleton and pair values. In the benchmark of Section~\ref{subsec:num-bench} we report both exact minimization, which enumerates every plan, and the coordinate search of Appendix~\ref{app:implementation}. The larger portfolios of Section~\ref{subsec:num-scale} and the experiment with misestimated inputs in Section~\ref{subsec:num-robust} use coordinate search. Both policies break ties with the service-preferring rule of Appendix~\ref{app:implementation}, and Table~\ref{tab:gaps} and Section~\ref{subsec:num-guarantee} also report PCP with the no-submission rule.

\paragraph{Rules for when to ask}
Always Ask submits, whenever no question is pending, the question of the unresolved task with the largest weight, breaking ties by task index. It never searches, and the human answers the pending question with the largest weight whenever available. Random-order Ask is identical except that the next task is chosen uniformly at random among the unresolved tasks without a question. Search Once Then Ask lets every agent work until its independent work is finished, then search once if the search can succeed, and makes the task eligible for a question after that search fails. Among eligible tasks, it submits the question of the task with the largest weight, and pending tasks do not search. All three rules obey Lemma~\ref{lem:reform}(ii), so at most one new question is submitted per period.

\subsection{Instances}
\label{app:num-instances}

Table~\ref{tab:app-instances} lists the instances of the synthetic experiments. The benchmark draws the weights with a pseudorandom generator seeded by the name of the portfolio. The homogeneous family is that of Lemma~\ref{app:lem-phi}, and its closed forms agree with the dynamic program for all $n\le5$. In the experiment with misestimated inputs, both policies compute their value tables and decisions from the misestimated parameters, with $p$ capped at $0.95$ and $a$ restricted to $[0.1,1]$, and are evaluated under the true parameters. In the mechanism experiments, the no-withdrawal model forbids search by an agent whose question has been submitted.

\begin{table}[!htbp]
\centering\small
\caption{Instances of the synthetic experiments. Unless noted otherwise, all tasks in a portfolio share $b$, $p$, $\alpha$, and $c$, $\kappa_i=0.1$, $\eta=0.5$, the weights are drawn independently and uniformly from $\{0.5,1,2,4\}$, and all tasks start unsubmitted with no failed searches. The benchmark has three weight draws for each combination of the factors, and the human-resolution regime and the mechanism portfolios have five. (a)~$w_i=1$ and $\kappa_i=\eta=0$. (b)~$w=\kappa=\eta=1$, with 81 combinations of $a$ and $p$ for each $\alpha$ and $b$. (c)~The weights are fixed across $b$. (d)~Each task draws its own $w_i$, $p_i$, $c_i$, and $\kappa_i\in\{0.05,0.1\}$, and $\eta\in\{0.05,0.1\}$ is common. (e)~$n\in\{4,8,12,16,24,32,48\}$.}
\label{tab:app-instances}
\setlength{\tabcolsep}{3pt}
\resizebox{\linewidth}{!}{%
\begin{tabular}{@{}llrllllll@{}}
\toprule
Experiment & Section & Instances & $n$ & $b$ & $a$ & $p$ & $\alpha$ & $c$\\
\midrule
 Benchmark & \ref{subsec:num-sim} & 864 & 3, 4 & 0, 1, 2, 4 & 0.25, 0.5, 1 & 0.3, 0.7 & 0.8, 1 & 0.1, 0.5, 1.5\\
 Human-resolution regime & \ref{subsec:num-guarantee} & 180 & 3, 4, 5 & 0, 1, 2, 3, 4, 6 & 0.5, 1 & 0 & -- & --\\
 Homogeneous family$^{\rm a}$ & \ref{subsec:num-guarantee} & all $(n,b)$ & 2--30 & $0,\ldots,1.5n$ & 1 & 0 & -- & --\\
 Larger portfolios & \ref{subsec:num-scale} & 56 & 4--48$^{(\rm e)}$ & 0, 2 & 0.5, 1 & 0.3, 0.7 & 0.8 & 0.5\\
 Misestimated inputs & \ref{subsec:num-robust} & 144 & \multicolumn{6}{l}{four-task benchmark portfolios with the first weight draw}\\
 Mechanisms, one task$^{\rm b}$ & \ref{app:num-mech} & 1,296 & 1 & 0, 1, 2, 3 & $0.1,\ldots,0.9$ & $0.1,\ldots,0.9$ & 0.5, 0.65, 0.8, 1 & 0.1\\
 Mechanisms, portfolios$^{\rm c}$ & \ref{app:num-mech} & 90 & 4 & 0, 1, 2, 3 & 0.25, 0.5, 1 & 0.3, 0.7 & 1 & 0.1, 0.5, 1.5\\
 Priority reversals$^{\rm d}$ & \ref{app:num-mech} & 4,000 & 3 & 0 & 1 & $0,0.1,\ldots,0.9$ & 1 & 0.01, 0.05, 0.2\\
\bottomrule
\end{tabular}
}
\end{table}

\subsection{Evaluation}
\label{app:num-evaluation}

\paragraph{Exact evaluation}
For the benchmark and the mechanism experiments, the expected cost and the expected numbers of searches, submissions, and answers of every policy follow from a recursion over the states it reaches, with self-loops eliminated as in~\eqref{eq:app-dp}. The optimal cost follows from Proposition~\ref{app:prop-dp}. Savings are averages of portfolio-level percentage differences. On a Macbook Pro with an Apple M5 processor, the exact solution of a four-task portfolio with $\alpha=0.8$ takes 5.7\,s on average.

\paragraph{Certificates}
For Theorem~\ref{thm:pcp_guarantee}(ii), we compute $\rho(s)$ by~\eqref{eq:pcp_residual} at every state of the exact dynamic program, which is the set of states reachable from the initial state under any plan of $\mathcal D_+$ and is therefore forward closed. We report $\epsilon_+=\max_s\rho(s)^+/\widetilde\omega(s)$ and $\epsilon_-=\max_s\rho(s)^-/\widetilde\omega(s)$, where $\widetilde\omega(s)=\sum_{i:z_i\ne R}w_i\le\omega(s)$ is the weight of the unresolved tasks and the maxima run over the states with $\widetilde\omega(s)>0$. Because $\widetilde\omega\le\omega$, these constants satisfy~\eqref{eq:residual_condition} and can only overstate the smallest constants that do. At a state with $\widetilde\omega(s)=0$, every requirement is resolved, the remaining cost is the deterministic work of the resolved tasks, and $\rho(s)=0$, so~\eqref{eq:residual_condition} holds there for any constants.

\paragraph{Simulation}
For larger portfolios, each instance is simulated on 200 sample paths. A path fixes, in advance, the recoverability of every task, the availability of the human in every period, and a uniform random number for every search attempt, and every policy is run on the same paths. The table-based implementation of PCP reproduces the exact expected cost of the reference implementation on 40 benchmark portfolios, and its simulated costs agree with the exact values within sampling error.

\subsection{LLM case study}
\label{app:num-llm}

\paragraph{Tasks}
Each task asks the agent to choose a supplier for one purchase order. A task has a purchase identifier, a quantity drawn from $\{40,60,80,100\}$, and four suppliers. Each supplier has a unit price drawn uniformly from $\{15,\ldots,60\}$, a fixed shipping charge drawn from $\{0,\ldots,90\}$, and a lead time, and the four lead times are a random permutation of $\{4,8,12,16\}$ days. The hidden requirement is a maximum lead time drawn from $\{6,10,14\}$ days, independently of the supplier data. The correct choice is the supplier with the lowest landed cost, the quantity times the unit price plus the shipping charge, among the suppliers whose lead time does not exceed the requirement.

\paragraph{Independent work}
The first unit of work presents eight quotes in random order: the current approved quote of each supplier, two superseded quotes with higher prices and shorter lead times, one draft with a lower price, and one approved quote for another purchase with a higher shipping charge. The agent must return the current approved quote of each supplier. The second unit asks the agent to compute each supplier's landed cost from the quotes it extracted. Neither unit uses the requirement.

\paragraph{Search packets}
A search presents six documents in random order, and Table~\ref{tab:app-packet} shows an example from the run. Five documents are decoys: an approved requirement for another purchase, a superseded revision of this purchase's requirement, a draft proposal, a supplier recommendation, and an operations note without a number. The sixth document is a planning estimate for another purchase when the search fails and the binding requirement when it succeeds. The binding document is approved, issued by the buyer, and matches the purchase and its current revision. It also mentions a second figure that is not binding, such as an abandoned proposal or an internal planning estimate, so the agent must identify the binding number. Document identifiers and the length of the packet carry no information about success.

\begin{table}[!htbp]
\centering\small
\caption{A search packet from the run in which the search succeeds. The binding document is the last one.}
\label{tab:app-packet}
\begin{tabularx}{\textwidth}{@{}lllp{0.11\textwidth}X@{}}
\toprule
Purchase & Revision & Status & Issuer & Text\\
\midrule
 other & 3 & approved & buyer & Approved requirement: the maximum permitted lead time is 7 calendar days.\\
 this & 3 & informational & operations & The team would prefer delivery near the end of the planning week. No numeric calendar-day limit is authorized in this note.\\
 this & 2 & approved & buyer & Superseded requirement: offers must arrive within 5 calendar days. See revision 3 for the current decision.\\
 this & 3 & draft & buyer & Proposal under discussion: set a ceiling of 18 calendar days. This proposal is not approved.\\
 this & 3 & approved & supplier & Supplier recommendation: allow 13 calendar days. This is our delivery recommendation, not buyer authorization.\\
 this & 3 & approved & buyer & For this purchase, accept lead times of no more than 10 calendar days. An internal 11-day planning estimate is not the binding requirement.\\
\bottomrule
\end{tabularx}
\end{table}

\paragraph{Search outcomes and availability}
For each task and portfolio, recoverability is drawn once with probability $p$. If the task is recoverable, each search attempt succeeds independently with probability $\alpha$. Outcomes are indexed by attempt, so the $k$th search of a task has the same outcome under every policy. The availability of the human is drawn once for each period and is the same for all policies. The scripted human answers at most one question in each available period and returns the true requirement.

\paragraph{Prompt}
Every agent step uses the prompt below, followed by a JSON object that lists the records to process. A record contains the action, which is \texttt{search}, \texttt{extract}, \texttt{calculate}, or \texttt{final}, and the public data that the step needs. The final step receives the requirement learned from search or from the human, together with the agent's own earlier lead times and costs.
{\footnotesize
\begin{verbatim}
Execute the independent synthetic procurement records supplied below.
Use ONLY the content of each record; do not use tools, files, internet or outside
knowledge. Records are independent: never transfer a fact from one id/purchase to
another. Return JSON {"records":[...]}, exactly one response for each supplied id,
and no prose outside JSON. Every response must include its original id.

For action=search: inspect the document packet using its authority_rule. Only a
buyer-issued APPROVED document for the exact purchase_id and current
requirement_revision establishes the limit. Return found (boolean),
maximum_lead_days (integer or null), evidence_document_id (string or null), and
evidence_text (an exact supporting passage from that document, or null). A draft,
old revision, different purchase, supplier recommendation or vague deadline is
not sufficient. If there is no qualifying explicit numeric calendar-day limit,
return found=false and null for the other three fields. Documents can mention
both an abandoned proposal and the binding limit: apply the binding statement.

For action=extract: use quote_rule and current_quote_revisions to choose one
current approved offer for each supplier. Return suppliers (objects containing
only supplier, unit_price, fixed_shipping, lead_days, sorted by supplier ID) and
lead_days (map from supplier ID to its selected numeric lead time). Do not copy
draft, obsolete or other-purchase offers.

For action=calculate: use the supplied previously extracted suppliers and
quantity. Return costs, a map from supplier ID to quantity*unit_price+fixed_shipping.
Do not substitute numbers from another record or invent missing inputs.

For action=final: use the supplied observed maximum_lead_days and previous
lead_days/costs artifacts. Choose any supplier of minimum supplied landed cost
among those whose supplied lead time is at most maximum_lead_days. Return
supplier, landed_cost, lead_days. If required evidence is missing, return null
for all three. Do not infer a missing requirement from the available suppliers.
\end{verbatim}
}

\paragraph{Execution}
Calls use \texttt{gpt-6.1-sol} with low reasoning effort through \texttt{codex exec}, with an ephemeral session and a read-only sandbox. Responses that use tools are rejected. Each call starts a fresh context and processes up to 24 records from portfolios run by the same policy, so copies of a portfolio under different policies never share a call. Up to 4 calls run at the same time, each with a timeout of 120 seconds and no retries. A malformed or failed response would count as a failed step, and none occurred. The policies use the nominal parameters and the service-preferring tie rule, and PCP and the singleton policy minimize~\eqref{eq:greedy} exactly. The cost of a portfolio is the total cost defined in Section~\ref{sec:setup}, and the final choice of supplier does not add to it. The documents, the prompt, and all random draws were fixed before the first call. All portfolios finished within 27 periods, below the limit of 60.

\paragraph{Usage}
The four policies required 175 calls with 2,287 records, about 4.1 million input tokens, of which 2.6 million were cached, and 0.15 million output tokens. A call took 23.7\,s on average.

\paragraph{Execution diagnostics}
Section~\ref{subsec:num-llm} reports that every agent step was correct. The agents made 133, 74, 0, and 160 searches under PCP, the singleton policy, Always Ask, and Search Once Then Ask, respectively.

\paragraph{Evaluation under the model}
Because the agents made no errors, the run follows the model exactly, and we evaluate each policy and each decision with the exact dynamic program. Table~\ref{tab:app-llm-exact} reports the expected cost of each policy and the optimal cost. For each decision $d$ taken at a state $s$ in the run, we also compute its expected excess cost $g_d(s)+\sum_{s'}P_d(s,s')V^*(s')-V^*(s)$, the increase in expected cost caused by taking $d$ at $s$ and acting optimally afterward. Table~\ref{tab:app-llm-regret} reports its average over the decisions of each policy. The savings in Section~\ref{subsec:num-llm} average percentages over portfolios. For comparison with them, we estimate the model's expected saving per portfolio from 5,000 simulated portfolios per setting with common random numbers. Relative to the singleton policy, it is 0.0\%, 0.0\%, 0.6\%, and 9.4\% in the four settings, in the order of Table~\ref{tab:app-llm-exact}, with standard errors of at most 0.3\%.

\begin{table}[!htbp]
\centering\small
\caption{Expected cost under the model in the LLM case study, and PCP's saving (\%) relative to the singleton policy computed from these expected costs.}
\label{tab:app-llm-exact}
\begin{tabular}{@{}ccrrrrrr@{}}
\toprule
$a$ & $c$ & $V^*$ & PCP & Singleton & Always Ask & Search Once Then Ask & Saving\\
\midrule
 1 & 0.1 & 12.89 & 12.89 & 12.89 & 13.40 & 20.73 & 0.0\\
 1 & 1.5 & 13.40 & 13.40 & 13.40 & 13.40 & 26.33 & 0.0\\
 0.25 & 0.1 & 30.63 & 30.68 & 30.97 & 42.65 & 40.25 & 1.0\\
 0.25 & 1.5 & 37.12 & 37.25 & 42.65 & 42.65 & 45.85 & 12.7\\
\bottomrule
\end{tabular}
\end{table}

\begin{table}[!htbp]
\centering\small
\caption{Average expected excess cost per decision in the LLM case study.}
\label{tab:app-llm-regret}
\begin{tabular}{@{}ccrrrr@{}}
\toprule
$a$ & $c$ & PCP & Singleton & Always Ask & Search Once Then Ask\\
\midrule
 1 & 0.1 & 0.000 & 0.000 & 0.128 & 1.351\\
 1 & 1.5 & 0.000 & 0.000 & 0.000 & 2.268\\
 0.25 & 0.1 & 0.004 & 0.060 & 2.626 & 0.921\\
 0.25 & 1.5 & 0.011 & 1.496 & 1.693 & 0.598\\
\bottomrule
\end{tabular}
\end{table}

\subsection{Computing environment}
\label{app:num-computing}

The dynamic programs, the policies, and the simulations are implemented in Python. The timings reported in Section~\ref{sec:numerics} and Appendix~\ref{app:num-evaluation} are from a Macbook Pro with an Apple M5 processor. The LLM case study uses the same implementation of the policies and calls the Codex command-line interface from that computer. Every number in Section~\ref{sec:numerics} and in this appendix is computed from the stored results by a script, without manual transcription.

\subsection{The mechanisms beyond the settings of the theorems}
\label{app:num-mech}

Theorem~\ref{thm:overlap} and Proposition~\ref{prop:withdraw} identify two mechanisms that can increase demand for human attention: useful independent work and the option to withdraw a question after asking. In both, the mechanism can shift the optimal initial action from searching to asking. The two results assume a single search attempt ($\alpha=1$) and at most one unit of work. This subsection shows that both mechanisms also arise with repeated search, several units of work, and four tasks that share one human, and it measures how often the priority reversal of Proposition~\ref{prop:noindex} occurs. All comparisons are between optimal policies, evaluated exactly as in Appendix~\ref{app:num-evaluation}, and Table~\ref{tab:app-instances} lists the instances. The single-task experiments use the cost parameters of Figure~\ref{fig:regions}.

\begin{figure}[!htbp]
\centering
\includegraphics[width=\textwidth]{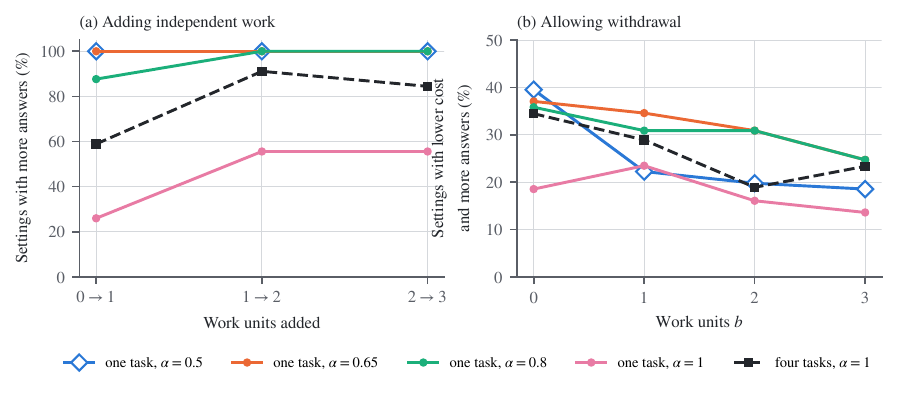}
\caption{The mechanisms of Section~\ref{sec:structure} beyond the settings of the theorems. Single tasks have $w=\kappa=\eta=1$, $c=0.1$, and 81 combinations of $a$ and $p$. (a) Share of settings in which one more unit of work raises the optimal expected number of human answers. The curve for $\alpha=0.5$ coincides with that for $\alpha=0.65$. (b) Share of settings in which withdrawal lowers the optimal cost and raises the expected number of answers. The dashed curves show the 90 portfolios of four tasks that share one human, with $\alpha=1$.}
\label{fig:mechanisms}
\end{figure}

\paragraph{Independent work}
Panel~(a) of Figure~\ref{fig:mechanisms} shows that useful work makes the optimal policy ask more. When search can be repeated ($\alpha<1$), adding one unit of work raises the optimal expected number of human answers in at least 88\% of settings at every work level, and in every setting when $\alpha\le0.65$. With work, an agent gives up searching sooner, because each failed search delays an answer that could overlap its work. With a single search attempt, extra work changes the number of answers only when it changes whether the agent searches before asking, which happens in 26\% to 56\% of settings. The effect persists when four tasks share the human: when the work of every task rises from zero to one, one to two, and two to three units, the optimal number of answers rises in 59\%, 91\%, and 84\% of portfolios, respectively.

\paragraph{Withdrawal}
Panel~(b) shows that withdrawal often raises human demand while lowering cost. When search can be repeated, withdrawal lowers the optimal cost in every setting, by 1.9\% to 6.6\% on average depending on $\alpha$ and $b$, and in 19\% to 40\% of settings it does so while raising the expected number of answers. In four-task portfolios, withdrawal lowers the optimal cost by 2.6\% to 4.8\% on average and raises answers while lowering cost in 19\% to 34\% of portfolios.

\paragraph{Priority reversals}
The reversal behind Proposition~\ref{prop:noindex} is rare. We draw 4,000 random portfolios of three tasks of the type in Example~\ref{ex:priority} (Table~\ref{tab:app-instances}) and compare the optimal costs of serving each task first. We keep the 3,942 portfolios in which one of the first two tasks should be strictly preferred for the first answer, both with and without the third task. Among them, adding the third task reverses which of the other two should be answered first in 80 (2.0\%). In 32 of these (0.8\% of all 3,942), the optimal policy serves the reversed choice first, so the reversal changes the optimal decision. In more than 99\% of these portfolios, therefore, the third task does not change which of the other two the optimal policy answers first.

\FloatBarrier
\subsection{Additional results}
\label{app:num-results}

Table~\ref{tab:app-factors} reports the benchmark results of Section~\ref{subsec:num-bench} by factor, and Table~\ref{tab:app-robust} reports every scenario of the experiment with misestimated inputs in Section~\ref{subsec:num-robust}.

\begin{table}[!htbp]
\centering\small
\caption{Benchmark results by factor: PCP's mean saving relative to the singleton policy (\%), mean gaps to the optimal cost (\%), and the mean difference between PCP and the singleton policy in the expected numbers of answers and searches.}
\label{tab:app-factors}
\begin{tabular}{@{}lrrrrr@{}}
\toprule
Level & Saving & PCP gap & Singleton gap & $\Delta$answers & $\Delta$searches\\
\midrule
 $n=3$ & 0.83 & 0.039 & 0.91 & -0.09 & 0.26\\
 $n=4$ & 1.38 & 0.091 & 1.57 & -0.19 & 0.58\\
\addlinespace
 $b=0$ & 1.23 & 0.002 & 1.35 & -0.15 & 0.45\\
 $b=1$ & 1.34 & 0.046 & 1.45 & -0.17 & 0.48\\
 $b=2$ & 1.49 & 0.100 & 1.69 & -0.13 & 0.43\\
 $b=4$ & 0.35 & 0.113 & 0.47 & -0.11 & 0.33\\
\addlinespace
 $a=0.25$ & 1.31 & 0.098 & 1.49 & -0.17 & 0.60\\
 $a=0.5$ & 1.23 & 0.066 & 1.37 & -0.19 & 0.51\\
 $a=1$ & 0.77 & 0.031 & 0.87 & -0.06 & 0.15\\
\addlinespace
 $p=0.3$ & 1.04 & 0.089 & 1.20 & -0.13 & 0.58\\
 $p=0.7$ & 1.17 & 0.041 & 1.28 & -0.15 & 0.27\\
\addlinespace
 $\alpha=0.8$ & 1.26 & 0.072 & 1.43 & -0.15 & 0.52\\
 $\alpha=1$ & 0.94 & 0.058 & 1.05 & -0.13 & 0.32\\
\addlinespace
 $c=0.1$ & 0.22 & 0.024 & 0.26 & -0.01 & 0.12\\
 $c=0.5$ & 1.08 & 0.069 & 1.20 & -0.14 & 0.46\\
 $c=1.5$ & 2.00 & 0.103 & 2.27 & -0.26 & 0.69\\
\bottomrule
\end{tabular}
\end{table}

\begin{table}[!htbp]
\centering\small
\caption{Policies built with misestimated inputs and evaluated under the true parameters: gap to the true optimal cost (\%), 144 portfolios.}
\label{tab:app-robust}
\begin{tabular}{@{}lrrrr@{}}
\toprule
& \multicolumn{2}{c}{PCP} & \multicolumn{2}{c}{Singleton}\\
\cmidrule(lr){2-3}\cmidrule(lr){4-5}
Inputs used by the policy & Mean & Worst & Mean & Worst\\
\midrule
 Correct inputs & 0.10 & 1.76 & 1.51 & 14.19\\
 $p$ halved & 0.64 & 5.79 & 4.02 & 33.04\\
 $p$ multiplied by 1.5 & 0.67 & 9.29 & 0.99 & 12.82\\
 $\alpha$ multiplied by 0.75 & 0.70 & 5.57 & 2.44 & 19.08\\
 $a$ decreased by 0.2 & 0.74 & 10.26 & 0.53 & 5.79\\
 $a$ increased by 0.2 & 0.17 & 1.76 & 2.27 & 18.33\\
\bottomrule
\end{tabular}
\end{table}

\end{document}